\documentclass[journal]{IEEEtran}

\usepackage{amssymb,amsmath,amsthm,amsfonts,algorithm,algorithmic,subcaption,graphicx,cite,color,flushend,mysymbol,hyperref,mathrsfs,lipsum,dsfont,multirow}
\AtEndDocument{\par\leavevmode}

\begin{document}

% Remove page numbering
%\pagenumbering{gobble}

\title{\huge{Physics-Based Learning for Robotic Environmental Sensing}}

% author names and IEEE memberships:
\author{Reza~Khodayi-mehr,~\IEEEmembership{Member,~IEEE},~and~Michael~M.~Zavlanos,~\IEEEmembership{Senior~Member,~IEEE}%
\thanks{Reza Khodayi-mehr and Michael M. Zavlanos are with the Department of Mechanical Engineering and Materials Science,  Duke University, Durham, NC 27708, USA, {\tt\footnotesize \{reza.khodayi.mehr, michael.zavlanos\}@duke.edu}.
This work is supported in part by the National Science Foundation under awards CNS-1837499 and CNS-1932011.
}}

% paper headers:
%\markboth{IEEE~Transactions~on~Robotics,~January~2021. Note:~First~Manuscript.}{Reza~Khodayi-mehr~and~Michael~M.~Zavlanos} 

% make the title area
\maketitle

\begin{abstract}
We propose a physics-based method to learn environmental fields (EFs) using a mobile robot. Common purely data-driven methods require prohibitively many measurements to accurately learn such complex EFs. Alternatively, physics-based models provide global knowledge of EFs but require experimental validation, depend on uncertain parameters, and are intractable for mobile robots. To address these challenges, we propose a Bayesian framework to select the most likely physics-based models of EFs in \textit{real-time}, from a pool of numerical solutions generated offline as a function of the uncertain parameters. Specifically, we focus on turbulent flow fields and utilize Gaussian processes (GPs) to construct statistical models for them, using the pool of numerical solutions to inform their prior mean. To incorporate flow measurements into these GPs, we control a custom-built mobile robot through a sequence of waypoints that maximize the information content of the measurements. We experimentally demonstrate that our proposed framework constructs a posterior distribution of the flow field that better approximates the real flow compared to the prior numerical solutions and purely data-driven methods.
\end{abstract}


\IEEEpeerreviewmaketitle

% ------------------------------------------------------------------------------------------------------------------------------ %
\section{Introduction} \label{sec:intro}
%
% Theme: motivation
\IEEEPARstart{E}{nvironmental} sensing using mobile robots has been widely explored in the literature to collect informative measurements while reducing cost \cite{REM2012DM,DCAUR2000Y}. The applications range from atmospheric pollution monitoring \cite{EMMR2008TRCL} and localization \cite{CPSL2006PF}, to oil and gas source localization \cite{UAUGGES2018RCDD}, chemical source identification \cite{meC1,meC2,meJ1}, underwater surveillance \cite{ISNAUHS2010JKEH}, oceanographic research \cite{UGOR2004RDEF}, estimation of temperature and salinity profiles \cite{NALGP2007KG,DEMMSN2008LSYF,CELSTDSF2011WZ}, and the investigation of the effects of ocean currents on aquatic life \cite{MFSTDDP2008CSMS,SSMAE2008CBKS}.
%
% Theme: necessity of flow field knowledge
Knowledge of the underlying flow field is often essential in these applications both for estimation and navigation.
For instance, in marine robotics, knowledge of ocean currents is important to plan efficient paths for autonomous underwater vehicles (AUVs) that perform long-term monitoring \cite{LUOCPSRN2016HPBSS,IDPMLOM2016MLS,PETSD2017JH,TEOPPGF2016KBH,DLMAOC2017ESGW}.
These currents can be modeled using, e.g., historical data from the Regional Ocean Modeling System \cite{ROMS2005SM} which, however, are known to be sparse and approximate. As a result, alternative approaches have been proposed in the literature that avoid models altogether and instead consider the flow field as a ``strong disturbance'' \cite{PETSD2017JH} or rely on periodic online drift measurements to account for the effect of ambient currents on the motion of an AUV \cite{OEOCSGDUV2019LYHA}.
Our goal in this paper is to develop an effective physics-based method to more accurately learn flow fields that are an essential component of many environmental sensing tasks such as those discussed above. Since learning flow fields is itself an environmental sensing problem, our proposed framework can also be applied to other environmental phenomena that can be described using principled physical models.
%
% Theme: major differences with marine robotics
Unlike commonly studied oceanic flow patterns that are often composed of simple circulations \cite{PETSD2017JH,TEOPPGF2016KBH,OEOCSGDUV2019LYHA}, here we focus on turbulent flows in confined environments, where complicated flow patterns emerge. Nonetheless, our methodology also applies to ocean currents \cite{IOT2007T}.
Moreover, unlike much of the robotics literature that assumes that instantaneous measurements of environmental fields can be acquired at no cost \cite{PETSD2017JH,TEOPPGF2016KBH,OEOCSGDUV2019LYHA,EBSPMSNCMRF2013XCDM,LEFECCUR2019DGKS,MRSEMGMRF2020NKRD}, here we treat flow measurements in a more realistic way by taking into consideration that turbulence is a stochastic process subject to considerable variation \cite{TMCFD1993W} and, as a result, collecting reliable measurements of turbulent flow properties is both time and energy consuming.

% Theme: numerical solution
A widely used method for estimating flow properties is numerical simulation based on the Reynolds-averaged Navier Stokes (RANS) equations. This approach is cost-effective and provides global estimation of the flow over a domain of interest.
Nevertheless, solutions provided by RANS models are generally incompatible with each other and with empirical data and require experimental validation \cite{EMLPRHRANSU2015LT}. Furthermore, precise knowledge of boundary conditions (BCs) and domain geometry is often unavailable, which results in even larger inaccuracies in the predicted flow properties. 
Finally, solving RANS models onboard mobile robots with limited computational resources is still intractable.
%
% Theme: data-driven approaches
Due to such challenges, the authors in \cite{OEOCSGDUV2019LYHA,EBSPMSNCMRF2013XCDM,LEFECCUR2019DGKS,MRSEMGMRF2020NKRD} forgo the use of physics-based models and instead advocate the use of purely data-driven statistical methods that rely on Gaussian processes (GPs)\footnote{As defined in \cite{GPML2006RW}, `a GP is a collection of random variables, any finite number of which have a joint Gaussian distribution'.} to estimate spatiotemporal fields starting from non-informative constant priors.
Such purely data-driven methods, however, can require a prohibitively large number of measurements to accurately estimate complex flow fields, making them intractable in practice. Moreover, even when empirical data are available, they are typically prone to noise and error.  As a result, only sparse unreliable data can usually be collected which are insufficient to adequately train statistical models.

% Theme: model selection
In this paper, we propose a Bayesian \textit{model selection} framework that combines empirical data with physics-based models to learn accurate representations of unknown flow fields in \textit{real-time}.
Specifically, given a distribution of the uncertain parameters like BCs, we generate offline a pool of numerical models of the flow field by solving the RANS equations for different choices of uncertain parameters selected according to this distribution. These numerical models may be inconsistent with each other or the true flow and thus, are only used to inform the prior mean of a corresponding pool of statistical GP models of the flow properties, specifically the mean velocity components and turbulent intensity field. Then, the proposed Bayesian framework allows to incorporate empirical data collected by a mobile robot sensor, to select the most likely flow models from the pool and to obtain the posterior distribution of the flow properties given each model.
As such, our approach takes advantage of the global information provided by physics-based models without needing to solve for them online, and produces high fidelity estimations of the flow properties using only sparse data, which is not possible with purely data-driven methods. In the simulation study presented in Section \ref{sec:sim}, our proposed framework achieves a three fold improvement in prediction error compared to a purely data-driven approach that initializes the GP models with uninformative priors.
%
% Theme: Gaussian processes
%A key idea that enables our \textit{real-time} model selection framework, is to construct a statistical model of the flow properties using GPs and employ the physics of the problem, captured by RANS models, to inform their prior mean. Particularly, we construct a set of GPs corresponding to the mean velocity components and turbulent intensity field for each numerical solution. Then, relying on Bayesian inference, we incorporate empirical data collected by a mobile robot sensor to obtain the posterior probabilities of the models in the pool and the flow properties given each model.
This is because unlike data-driven methods, here empirical data are only used to improve physics-based models and not to replace them. Moreover, unlike much of the robotics literature \cite{EBSPMSNCMRF2013XCDM,LEFECCUR2019DGKS,SLROAUNSGP2016BMSK,SRLCGP2015BS,LNMPCIVMRPT2014OSB,CCLNTDBNDPGP2014WLZFKO,LDSMSFMSR2017LS,PPPMTSRGRF2013LS,MRSEMGMRF2020NKRD} and marine robotics literature \cite{LUOCPSRN2016HPBSS,IDPMLOM2016MLS,PETSD2017JH,OEOCSGDUV2019LYHA} in particular, where GPs are mainly used for computational expediency, here the use of GPs is in fact theoretically justified since the mean velocity components are normally distributed.

% Theme: variance of GPs
An important advantage of our proposed physics-based method is the use of turbulent intensity estimates provided by RANS models to define informative prior uncertainty fields for the GPs. This is not possible in the marine robotics literature in \cite{IDPMLOM2016MLS,PETSD2017JH,OEOCSGDUV2019LYHA} or the purely data-driven methods in \cite{EBSPMSNCMRF2013XCDM,LEFECCUR2019DGKS,MRSEMGMRF2020NKRD} that, in the absence of any prior knowledge, define constant prior uncertainty fields for their GP models.
%\red{The work in \cite{LUOCPSRN2016HPBSS} is an exception that estimates the uncertainty in predictive ocean current models and relies on the resulting confidence measures to plan the path of an AUV.}
%
An additional advantage of our method is that it derives accurate expressions of the measurement noise and incorporates them in the statistical inference process. This is in contrast to much of the robotics literature discussed above that oversimplifies the measurement collection process.
For example, a common assumption in this literature is that the measurement noise has \textit{constant} variance which, in the case of mean turbulent flow properties is incorrect since turbulent variation is an important factor in determining measurement uncertainty; see Section \ref{sec:measM} for more details.

% Theme: planning
A major contribution of this paper is also the experimental validation of the proposed framework, demonstrating that it is robust to the significant uncertainties that are present in the real-world; see \cite{meJ3_video2,meJ3_video1} for a visualization of such uncertainties. This is in contrast to the relevant literature that typically relies on numerical simulations to showcase the proposed methods \cite{EBSPMSNCMRF2013XCDM,LEFECCUR2019DGKS,MRSEMGMRF2020NKRD}.
Specifically, to collect the empirical data needed to train the GPs, we have custom-built a mobile robot sensor that collects and analyzes instantaneous velocity readings and extracts the mean flow properties. We incorporate the specific noise characteristics of our sensor rig into the GP models.
It is known that in the presence of uncertain model parameters, more information can be gained through online planning compared to measurements collected at a pre-determined set of locations; the higher the uncertainty, the more significant the gain \cite{NALGP2007KG}.
Thus, to maximize the amount of information collected about the flow field, we develop an active learning method that guides the robot through a sequence of waypoints that minimize the posterior entropy of unobserved areas in the domain, given available empirical data. In a simulation study with a known ground-truth, we show that our planning metric correlates very well with the true (unknown) error field; see Section \ref{sec:sim}.
%Although our proposed planning algorithm is not optimal, we show that it outperforms an alternative approach based on mutual information, that possesses a suboptimality bound \cite{NSPGP2008KSG}, both in terms of predictive performance and computational efficiency.

% Theme: planning
Compared to path planning methods in marine robotics \cite{LUOCPSRN2016HPBSS,IDPMLOM2016MLS,PETSD2017JH,TEOPPGF2016KBH,DLMAOC2017ESGW,OEOCSGDUV2019LYHA}, where flow fields are used for \textit{persistent} monitoring of aquatic phenomena with an emphasis on designing optimal paths subject to time and energy budget constraints, here the purpose of planning is to maximize the information collected about the flow field itself.
While those planning methods might be optimal in theory,  in practice they often require approximations to mitigate their computational cost \cite{PETSD2017JH}.
Moreover, sophisticated algorithms like the stochastic optimal control approach proposed in \cite{LEFECCUR2019DGKS} for data-driven environmental sensing, have only been demonstrated on simple uncertainty fields and it is unclear if they can be applied to the highly nonlinear uncertainty fields considered here.
%
% Theme: general active learning in GPs
In general, active learning of GPs has been extensively investigated in the robotics literature with applications ranging from estimation of nonlinear dynamics to spatiotemporal fields \cite{SLROAUNSGP2016BMSK,SRLCGP2015BS,LNMPCIVMRPT2014OSB,CCLNTDBNDPGP2014WLZFKO,LDSMSFMSR2017LS,PPPMTSRGRF2013LS}.
Closely related are also methods for robotic state estimation and planning with Gaussian noise; see \cite{DASEUA2017FLZ,OPPRAATL2015FMZ,meJ2}.
This literature typically employs simple, explicit models of small dimensions and does not consider model ambiguity or parameter uncertainty. Instead, here we focus on complex models of continuous flow fields that are implicit solutions of RANS models, a system of PDEs.
Although our planning method is not theoretically optimal, we show that it is effective in solving complex real-world environmental sensing problems. %To the best of our knowledge, such a capability has not been yet demonstrated.

% Theme: contributions
The contributions of this work can be summarized as follows. To the best of our knowledge, this is the first physics-based framework for robotic environmental sensing that is demonstrated \textit{in practice}. Since the computationally expensive numerical simulations needed to inform the prior mean of GP models are performed offline, our method can be implemented in \textit{real-time} onboard mobile robots.
%\blue{To this date, learning methods have been utilized to improve the computational cost \cite{AMLAFMO1999MMK,AEFSCN2016TSSP,DFSRF2015JSPG,CSMLSCFC2013BLK} and accuracy \cite{MLMDDTM2015ZD,EMLPRHRANSU2015LT} of RANS models, but do not address the shortcomings that make physics-based methods infeasible for mobile robots.
%On the other hand, active learning methods for robotic environmental sensing \cite{EBSPMSNCMRF2013XCDM,LEFECCUR2019DGKS,MRSEMGMRF2020NKRD} are typically purely data-driven making them unsuitable for problems in which collecting measurements is costly, e.g., estimation of turbulent flow fields.}
Compared to purely data-driven methods that are common in the literature, our method can produce high-fidelity global estimations using only sparse measurements.
Incorporating empirical data allows our method to resolve inconsistencies between numerical solutions and the true flow. Thus, a main contribution of this work is that it provides a systematic and tractable framework to combine principled physical models and empirical data to learn accurate models of complex environmental phenomena using mobile robots.
While the proposed planning method is not theoretically optimal, our planning metric correlates very well with the true (unknown) error field and, therefore, outperforms a baseline lattice placement and performs at least as well as a planning method based on mutual information that possesses a suboptimality bound \cite{NSPGP2008KSG} but is intractable for large domains.
Finally, an additional significant contribution of this work lies in the experimental validation of our method, demonstrating that it is robust to the real-world uncertainties present in these complex environmental sensing problems. %We show that our proposed planning framework is able to select the most accurate models, obtain the posterior distribution of the flow properties and, predict these properties at new locations with reasonable uncertainty bounds and better accuracy compared to purely numerical and data-driven approaches.
We note that although here we focus on turbulent flow fields, our framework can also be applied to robotic learning of other environmental phenomena that can be described using physical models, e.g., chemical concentration \cite{meJ1} or temperature \cite{NSPGP2008KSG} fields.

% Theme: structure of the paper
The remainder of the paper is organized as follows. In Section \ref{sec:turbF}, we discuss our model selection framework and develop a statistical model to capture the properties of turbulent flows. Section \ref{sec:path} describes the design of the mobile sensor, the characterization of measurement noise, and the proposed planning method. In Section \ref{sec:exp}, we present our results and in Section \ref{sec:disc}, we discuss possible future directions. Finally, Section \ref{sec:concl} concludes the paper.

% ------------------------------------------------------------------------------------------------------------------------------ %
\section{Statistical Model of Flow Fields} \label{sec:turbF}
Consider a turbulent flow field over a domain of interest $\bbarOmega \subset \reals^3$ and let $\bbq(x,t) : \bbarOmega \times  [0, T] \to \reals^3$ denote the corresponding flow velocity vector, where $x \in \bbarOmega$ and $t \in [0, T]$.
Due to turbulence, $\bbq(x,t)$ is a random variable (RV) subject to high variation.
Therefore, in engineering applications, we rather focus on the mean velocity field $\hhatbbq(x,t)$ and the turbulent intensity field $i(x,t)$ which is a measure of turbulent variations; we accurately define these quantities in the next section.
Consider a mobile robot sensor that can obtain instantaneous measurements of the random velocity field $\bbq(x,t)$ for a period of time at a set of locations and let the vector $\hhatbby$ denote the collection of these measurements. In Section \ref{sec:measM}, we elaborate on obtaining empirical measurements of the mean velocity and turbulent intensity fields from these instantaneous velocity readings. Given this notation, the problem that we address in the paper can be defined as follows.
\begin{prob} [Environmental Sensing] \label{prob:main}
Given the vector of measurements $\hhatbby$ of the instantaneous velocity field $\bbq(x,t)$, collected by a mobile robot, obtain  the posterior distributions $\bbarpi(\hhatbbq(x,t) | \hhatbby) \and \bbarpi(i(x,t) | \hhatbby)$ of the mean velocity field and turbulent intensity field.
\end{prob}

To this date, the robotics literature has mainly focused on purely data-driven approaches to solve this environmental sensing problem. As discussed in Section \ref{sec:intro}, since these methods rely on uninformative priors, they require prohibitively many measurements to accurately estimate complex fields. Instead, here we are interested in a physics-based solution that allows us to obtain more accurate estimation of the desired posterior distributions using fewer measurements.

% ------------------------------------------------------------- %
\subsection{Preliminaries on Turbulent Flow Fields} \label{sec:RANS}
Let $\bbq^j(x,t)$ denote $j$-th realization of the random velocity vector $\bbq(x,t)$ at location $x$ and time $t$. Given a total of $\hhatn$ realizations, the ensemble average is defined as
\begin{equation} \label{eq:ensemble}
\hhatbbq(x,t) = \lim_{\hhatn \to \infty} \frac{1}{\hhatn} \sum\nolimits_{j=1}^{\hhatn} \bbq^j(x, t) .
\end{equation}
This is the desired mean velocity in Problem \ref{prob:main} but cannot be computed in practice since it is impossible to take more than one measurement of the flow field at a given time and location.
Consider instead a decomposition of the velocity vector into its \textit{time-averaged} value $\bbq(x)$ and a perturbation as 
\begin{equation} \label{eq:velDec}
\bbq(x,t) = \bbq(x) + \bbepsilon(x,t), 
\end{equation}
 where
\begin{equation} \label{eq:time-avg}
\bbq(x) = \lim_{T \to \infty} \frac{1}{T} \int_{t_1}^{t_1+T} \bbq(x,t) dt .
\end{equation}
Assuming that the flow is \textit{stationary} or \textit{statistically steady}, the integral will be independent of $t_1$; set $t_1=0$ without loss of generality.
The stationary random vector $\bbq(x,t)$ is \textit{ergodic} if in the limit, the time average \eqref{eq:time-avg} converges to the ensemble average \eqref{eq:ensemble}, i.e., $\bbq(x) = \hhatbbq(x,t)$. In this paper we make the following assumption.

\begin{assumption} \label{assumption:ergodic}
Turbulent flow is stationary and ergodic.
\end{assumption}
This is a common assumption that allows us to obtain time-series measurements of the instantaneous velocity field at a given location and use the time-averaged velocity $\bbq(x)$ as a surrogate for the desired ensemble average $\hhatbbq(x,t)$ \cite{TMCFD1993W,HMTHWA2001J}. Assumption \ref{assumption:ergodic} implies that $\hhatbbq(t,x)$ is independent of time. Learning transient turbulent flows can be challenging since meaningful measurements of the flow field cannot be easily made; see Section \ref{sec:disc} for more details.

Next, we define the turbulent intensity field whose posterior we seek to estimate in Problem \ref{prob:main}. For this, we first need to define the root mean square (RMS) value of the fluctuations of the $k$-th velocity component
\begin{equation}
\bbq_{k,\text{rms}}(x) = \left( \lim_{T \to \infty} \frac{1}{T} \int_{0}^{T} \abs{\bbepsilon_k(x,t)}^2 dt \right)^{0.5} .
\end{equation}
Noting that $\bbepsilon(x,t) = \bbq(x,t) - \bbq(x)$ and using the ergodicity assumption, we have
\begin{equation}
\var[ \bbq_k(x,t) ] \triangleq \mbE \set{ \left[ \bbq_k(x,t) - \hhatbbq_k(x,t) \right]^2 } = \bbq_{k,\text{rms}}^2(x) ,
\end{equation}
where $\var[ \cdot ]$ denotes the variance of a RV. The average normalized RMS value is called turbulent intensity and is defined as
\begin{equation} \label{eq:turbI}
i(x) = \frac{ \norm{ \bbq_{\text{rms}}(x) } }{ q_{\text{ref}} } ,
\end{equation}
where $\norm{\cdot}$ denotes the Euclidean norm and $q_{\text{ref}}$ is a constant characteristic velocity used to normalize $i(x)$. Turbulent intensity is an isotropic measure of the average fluctuations caused by turbulence at point $x$.
 
Finally, we define the integral time scale $\bbt_k^*(x)$ of the turbulent flow field along the $k$-th velocity direction as
\begin{equation} \label{eq:intTimeScale}
\bbt_k^*(x) = \frac{1}{\bbrho_k(0)} \int_0^{\infty} \bbrho_k(\tau) \, d\tau ,
\end{equation}
where $\bbrho_k(\tau)$ is the autocorrelation of the velocity given by
\begin{equation}
\bbrho_k(\tau) = \lim_{T \to \infty} \frac{1}{T} \int_0^T \bbepsilon_k(x,t) \, \bbepsilon_k(x, t + \tau)  dt ,
\end{equation}
and $\tau$ is the autocorrelation lag.
We use the integral time scale to determine the sampling frequency in Section \ref{sec:smpFreq}. For more details on turbulence theory, see \cite{TMCFD1993W}.

% ------------------------------------------------------------- %
\subsection{Statistical Flow Model} \label{sec:statModel}
Let $\xi \in \reals^{n_{\xi}}$ encode the parameters needed to specify the domain $\bbarOmega$ and flow conditions imposed on its boundaries. Given $\xi$, we employ Reynolds-Averaged Navier Stokes (RANS) models to predict the flow properties like the mean velocity components and turbulent intensity, globally over $\bbarOmega$.
RANS models substitute the decomposition \eqref{eq:velDec} into the Navier Stokes equations to obtain a model for the mean flow velocity components. This introduces additional unknowns, called Reynolds stresses, in the momentum equation that require proper treatment to complete the set of equations; this is known as the `closure problem'. Depending on how this problem is addressed, various RANS models exist that are broadly categorized into eddy viscosity models (EVMs) and Reynolds stress models (RSMs) \cite{TMCFD1993W}.
The solution returned by these models can be incompatible with each other. Moreover, there is often high uncertainty in the parameters $\xi$, i.e., the domain geometry and boundary conditions (BCs). As a result, numerical solutions generally require experimental validation. Finally, solving RANS models onboard mobile robots in real-time is still intractable. In what follows, we propose a statistical framework to address these challenges, enabling for the first time, a physics-based solution to Problem \ref{prob:main}.

%On the other hand, empirical data is prone to noise and this uncertainty needs to be considered before deciding on a model that best matches the real flow field. Bayesian inference provides a natural approach to combine theoretical and empirical knowledge in a systematic manner. In this paper, we utilize Gaussian processes (GPs)\footnote{As defined in \cite{GPML2006RW}, `a GP is a collection of random variables, any finite number of which have a joint Gaussian distribution'.}
%to model uncertain flow fields; see \cite{GPML2006RW}.
%Using Bayesian inference with GPs and for a Gaussian likelihood function, we can derive the posterior distributions of the flow properties conditioned on available data in closed-form which is computationally very advantageous.
%
To do so, we first define Gaussian Process (GP) models for the mean velocity and turbulent intensity fields. As we show in Proposition \ref{prop:CLT}, this choice is justified since the time-averaged velocity components are normally distributed regardless of the underlying distribution of the instantaneous velocity field. Specifically, let $u(x,t) = \bbq_1(x,t)$ denote the first instantaneous velocity component with ensemble mean $\hhatbbq_1(x,t)$ and variance $\var[\bbq_1(x,t)]$; note that these quantities are fixed but unknown. Then, we have the following result.
\begin{prop} [Gaussianity of Mean Velocity Components] \label{prop:CLT}
Let Assumption \ref{assumption:ergodic} hold. Assume also that $n$ \textit{independent} samples $u(x,t_l)$ of the first velocity component are given over time where $l \in \set{1, \dots, n}$, and define the sample time-average
$ u(x) = {1}/{n} \sum_{l=1}^n u(x,t_l) .$
Then, for large enough values of $n$, we approximately have
$$ u(x) \sim \ccalN \left( \hhatbbq_1(x,t), \frac{ \var[\bbq_1(x,t)] }{n} \right) .$$
A similar result holds for the other velocity components.
\end{prop}
\begin{proof}
This result directly follows from the central limit theorem (CLT) \cite{FCP2002R}. See Appendix \ref{app:CLT} for details.
%According to the ergodicity assumption, the time-averaged velocity component $u(x)$ is equal to the ensemble average $\hhatbbq_1(x,t)$. 
%%
%Also, since the flow is statistically steady, at a given spatial location $x$, the samples $u(x,t_k)$ are identically distributed.
%Moreover, since $u(x,t)$ is a physical quantity, its variance is bounded. Then, from the CLT it follows that as $n \to \infty$, the distribution of $u(x)$ approaches a normal distribution. More specifically,
%%
%$$ \sqrt{n} \, \left( u(x) - \hhatbbq_1(x,t) \right) \rightarrow \ccalN \left( 0, \var[\bbq_1(x,t)] \right) .$$
%%
%A similar result can be obtained for other velocity components.
\end{proof}
Note that by the ergodicity assumption, the continuous time-average \eqref{eq:time-avg} and the sample time-average in Proposition \ref{prop:CLT} are both equivalent to the ensemble average.\footnote{As a rule of thumb, $n>30$ samples are necessary for the approximation to hold \cite{FCP2002R}. Noting that the continuous time-average \eqref{eq:time-avg} is infeasible in practice, in the following we drop the term `sample' when there is no confusion. Also, we use the terms `time-averaged', `averaged', and `mean' interchangeably.}
However, for the distribution to approach the specified normal distribution, the samples must be \textit{independent}. In Section \ref{sec:smpFreq} we discuss how to obtain such independent samples.

Given Proposition \ref{prop:CLT}, we can now define GP models to capture the distribution of the mean velocity components.
Specifically, given a value for parameters $\xi$, we solve a RANS model, e.g., RSM, to obtain a prediction for the flow properties. Let $\mu_u(x)$ denote the prediction for the first mean velocity component $u(x)$. Then, we model the prior distribution of $u(x)$, before collecting any measurements, using the following GP
\begin{equation} \label{eq:GPu}
u(x) \sim \ccalG\ccalP(\mu_u(x), \bbarkappa_u(x,x')) ,
\end{equation}
where, the kernel function $\bbarkappa_u(x, x')$ is defined as
\begin{equation} \label{eq:kernel}
\bbarkappa_u(x,x') = \bbarsigma_u^2(x,x') \, \rho(x,x').
\end{equation}
In \eqref{eq:kernel}, the standard deviation $\bbarsigma_u(x,x') \in \reals_+$ encapsulates the prior uncertainty in $u(x)$ and $\rho(x,x')$ is the correlation function. Explicitly, we define the standard deviation as
\begin{equation} \label{eq:ustd}
\bbarsigma_u^2(x,x') = \bbarsigma_{u,0}^2 + \frac{ 1 }{n_0} \, q^2_{\text{ref}} \, i(x) \, i(x'),
\end{equation}
where the constant $\bbarsigma_{u,0}  \in \reals_+$ is a measure of confidence in the numerical solution $\mu_u(x)$ and is selected depending on the convergence metrics provided by the RANS solver. The second term in \eqref{eq:ustd}, which is unavailable in purely data-driven approaches, captures the local variability due to turbulence, by relying on the estimation of turbulent intensity $i(x)$ provided by the RANS model. $n_0 \in \naturals_+$ is a nominal number of samples that scales this variability for the \textit{averaged} velocity component, similar to Proposition \ref{prop:CLT}.
We define the correlation function $\rho(x,x')$ in \eqref{eq:kernel} as a compactly supported polynomial 
\begin{equation} \label{eq:correlation}
\rho(x,x') = \left[ \left(1 - \frac{ \norm{x-x'} }{\ell} \right)_+ \right]^2,
\end{equation}
where $\ell \in \reals_+$ is the correlation characteristic length and the operator $(\alpha)_+ = \max(0, \alpha)$.
The correlation function \eqref{eq:correlation} implies that two points with distance larger than $\ell$ are uncorrelated, which results in sparse covariance matrices \cite{GPML2006RW}. %In the following we use a constant standard deviation, i.e., we set $\bbarsigma(x,x') = \bbarsigma$.

\begin{prop} [Validity of Kernel Function] \label{prop:kernel}
The kernel function \eqref{eq:kernel} with standard deviation \eqref{eq:ustd} and correlation function \eqref{eq:correlation} constitutes a valid kernel function.
\end{prop}
\begin{proof}
See Appendix \ref{app:kernel}.
%This follows from the fact that the kernel function \eqref{eq:kernel} is the composition of multiple valid kernel function. First, note that the correlation function \eqref{eq:correlation} is a valid kernel itself. Regarding the standard deviation \eqref{eq:ustd}, note that multiplication of any function with itself, i.e., $\kappa_1(x,x') = i(x) i(x')$, constitutes a valid kernel. Moreover, scaling and addition with positive constants $q^2_{\text{ref}}/n_0$ and $\bbarsigma_{u,0}^2$, respectively, preserves the validity of the kernel. Thus the standard deviation \eqref{eq:ustd} is a valid kernel as well. Finally, the multiplication of two valid kernels results in a valid kernel meaning that the kernel function \eqref{eq:kernel} is valid as desired. See \cite[Sec 4.2]{GPML2006RW} for details.
\end{proof}

In practice, it is impossible to obtain noiseless samples as assumed in Proposition \ref{prop:CLT}. Thus, we consider a measurement model for $u(x)$ with additive Gaussian noise $\epsilon_u \sim \ccalN(0, \sigma_u^2(x))$; see Section \ref{sec:measM} for justification. Specifically, let $y_u(x) \in \reals$ denote a measurement of the first mean velocity component at a location $x$ given by
\begin{equation} \label{eq:measModel}
y_u(x) = u(x) + \epsilon_u(x) .
\end{equation}
Then $y_u(x)$ is also a GP
\begin{equation} \label{eq:GP}
y_u(x) \sim \ccalG\ccalP( \mu_u(x), \kappa_u(x,x') ) 
\end{equation}
with the following kernel function
\begin{equation} \label{eq:measKernel}
\kappa_u(x,x') = \bbarkappa_u(x,x') + \sigma_u^2(x) \, \delta(x-x'),
\end{equation}
where $\delta(x-x')$ is the Kronecker delta function.

Given a vector of measurements $\bby_{u,k}$ collected by the mobile robot at a set of $k$ locations $\ccalX_k$, the predictive distribution of $u(x)$ at a point $x$, conditioned on measurements $\bby_{u,k}$, is a Gaussian distribution whose mean and variance are given in closed-form by
\begin{subequations} \label{eq:inference}
\begin{eqnarray}
\mu_u(x| \ccalX_k) &=& \mu_u(x) + \barbSigma_{x\ccalX} \bbSigma_{\ccalX \ccalX}^{-1} \left(\bby_{u,k} - \bbmu_{\ccalX} \right) , \label{eq:predMean}	\\
\gamma_u^2(x | \ccalX_k) &=& \bbarkappa_u(x,x) - \barbSigma_{x\ccalX} \bbSigma_{\ccalX \ccalX}^{-1} \barbSigma_{\ccalX x} , \label{eq:predVar}
\end{eqnarray}
\end{subequations}
where $\bbmu_{\ccalX}$ denotes the mean function evaluated at measurement locations $\ccalX_k$ and the entries of the covariance matrices $\barbSigma_{x\ccalX}$ and $\bbSigma_{\ccalX \ccalX}$ are computed using \eqref{eq:kernel} and \eqref{eq:measKernel}, respectively. Note that for simplicity, we have dropped the subscripts $u \and k$ from the matrices in \eqref{eq:inference}. By Proposition \ref{prop:kernel}, the matrix $\bbSigma_{\ccalX \ccalX}$ is positive-definite and invertible.

From the Navier-Stokes PDEs, it follows that the mean velocity components are correlated \cite{TMCFD1993W}. However, since the prior fields, obtained from RANS models, already capture this correlation, to simplify the pursuant development, we make the following assumption.
\begin{assumption} \label{ass:velCompCorr}
The turbulent flow properties are independent and thus, uncorrelated.
\end{assumption}
Then, we can independently define GPs for the second mean velocity component $v$ and the turbulent intensity field with appropriate subscripts.\footnote{We only consider in-plane velocity components throughout the paper. The extension of the theoretical developments for the third component is trivial.}
%Specifically, for the first mean velocity component, we then have}
%%
%\begin{equation} \label{eq:GPu}
%u(x) \sim \ccalG\ccalP(\mu_u(x), \kappa_u(x,x'))
%\end{equation}
%%
%with mean $\mu_u(x)$ obtained from the numerical solution of a RANS model and the standard deviation in \eqref{eq:kernel} \blue{denote by $\bbarsigma_u$}.
%This value is a measure of confidence in the numerical solution $\mu_u(x)$ and can be adjusted to reflect this confidence depending on the model that is utilized and the convergence metrics provided by the RANS solver.
%GPs for the other velocity components $v$ and $w$ can be defined in a similar way.
%%
%On the other hand, 
Specifically, for the turbulent intensity $i(x)$, we define the prior GP as
\begin{equation} \label{eq:GPi}
i(x) \sim \ccalG \ccalP (\mu_i(x), \bbarkappa_i(x,x')) ,
\end{equation}
where we select the prior mean $\mu_i(x)$ from the numerical solution of the RANS model and use a constant standard deviation $\bbarsigma_i(x,x') = \bbarsigma_{i,0}$ in the kernel function \eqref{eq:kernel}. Note that unlike the mean velocity components that by Proposition \ref{prop:CLT} are normally distributed, modeling $i(x)$ as a GP is an assumption that allows us to take advantage of the closed-form posteriors \eqref{eq:inference}.

% ------------------------------------------------------------- %
\subsection{Hierarchical Bayesian Model Selection} \label{sec:hierarchical}
In Section \ref{sec:statModel}, we constructed GP models of the flow properties given known parameters $\xi$ using a specific RANS model. However, as discussed before, models constructed in this way, can be inaccurate due to the assumptions made to derive the RANS model and high uncertainty in the parameters $\xi$. Next, we outline a model selection method that takes into account this parameter uncertainty to generate \textit{offline}, a pool of likely models that later, the robot can select from and improve upon based on empirical data that it collects online.

Let $\ccalN$ denote a RANS model and consider a possibly uniform discrete prior distribution $\tdpi(\ccalN)$ over the available models. Furthermore, let $\tdpi(\xi)$ denote the discrete prior distribution on the domain geometry and BC parameters $\xi$. If the prior on the parameters is continuous, we can construct a discrete approximation $\tdpi(\xi)$ using stochastic reduced order models (SROMs); see \cite{meC3} for details.
Let $\ccalM = (\ccalN, \xi)$ denote the numerical solution obtained using the RANS model $\ccalN$ given the parameters $\xi$. Noting that $\ccalN \and \xi$ are independent, $\tdpi(\ccalM) = \tdpi(\ccalN) \tdpi(\xi)$. Let $\ccalM_j$ denote the $j$-th numerical model obtained for one combination of discrete $\ccalN$ and $\xi$ values and let the collection $\tdpi(\ccalM) = \set{p_{j,0}, \ccalM_j}_{j=1}^{\bbarn}$ denote the pool of numerical models, where $p_{j,0}$ denotes the prior probability of model $\ccalM_j$ and $\bbarn$ is the number of models.

Given measurements $\bby_{u,k}, \bby_{v,k}, \and \bby_{i,k}$ of the mean velocity components and turbulent intensity at a set of $k$ locations $\ccalX_k$, the posterior distribution over models can be obtained using Bayes' rule as
\begin{align} \label{eq:hierarchical}
\tdpi(\ccalM_j | \ccalX_k) &= \alpha \, \bbarpi(\bby_{u,k}, \bby_{v,k}, \bby_{i,k} | \ccalM_j) \, \tdpi(\ccalM_j) \\
& =\alpha \, \bbarpi(\bby_{u,k} | \ccalM_j) \, \bbarpi(\bby_{v,k} | \ccalM_j) \, \bbarpi(\bby_{i,k} | \ccalM_j) \, \tdpi(\ccalM_j) , \nonumber
\end{align}
where $\alpha$ is the normalizing constant in Bayes' rule and $ \bbarpi(\bby_{u,k} | \ccalM_j)$ is the likelihood of the measurements $\bby_{u,k}$ given model $\ccalM_j$ and similarly for $\bby_{v,k} \and \bby_{i,k}$. Note that the joint likelihood of the measurements in \eqref{eq:hierarchical} is equivalent to the product of the individual likelihoods since the flow properties are independent; see Assumption \ref{ass:velCompCorr}.
From the definition of the GPs constructed for model $\ccalM_j$ in Section \ref{sec:statModel}, we can obtain these likelihoods in closed-form. For instance
\begin{align}
\bbarpi(\bby_{u,k} | \ccalM_j) =& \det(2 \pi \, \bbSigma_{u,j} )^{-0.5} \\
& \exp \left( -\frac{1}{2} (\bby_{u,k} - \bbmu_{u,j})^T \bbSigma_{u,j}^{-1} (\bby_{u,k} - \bbmu_{u,j}) \right) , \nonumber
\end{align}
where $\bbmu_{u,j}$ and $\bbSigma_{u,j}$ are short-hand notation for the mean and covariance of the GP corresponding to $y_u(x) \and \ccalM_j$ at locations $\ccalX_k$; see the discussion after equation \eqref{eq:inference}. Since the sum of the discrete posterior model probabilities equals one, i.e., $\sum_{k=1}^{\bbarn} \tdpi(\ccalM_j | \ccalX_k) = 1$, we can compute the normalizing constant $\alpha$ from \eqref{eq:hierarchical} as
\begin{equation}
\alpha = \left( \sum\nolimits_{j=1}^{\bbarn} \bbarpi(\bby_{u,k}, \bby_{v,k}, \bby_{i,k} | \ccalM_j) \, \tdpi(\ccalM_j) \right)^{-1} .
\end{equation}
Given $\alpha$, we can finally compute the posterior distribution $\tdpi(\ccalM_j | \ccalX_k)$ over the pool of models using \eqref{eq:hierarchical}. This amounts to Bayesian \textit{model selection} and enables the mobile robot to assign probabilities to models constructed for likely parameter values, given the latest empirical data.
Note that although hierarchical Bayesian models are extensively used in the literature \cite{EBSPMSNCMRF2013XCDM,LEFECCUR2019DGKS,NALGP2007KG}, this work is the first to utilize them for physics-based robotic environmental sensing.

Given $\tdpi(\ccalM_j | \ccalX_k)$, the desired posterior distributions in Problem \ref{prob:main} are the marginal distributions $\pi(u | \ccalX_k), \pi(v | \ccalX_k)$, and $\pi(i | \ccalX_k)$ after integrating over the models.
These marginal distributions are GP mixtures (GMs) with their mean and variance given by
\begin{subequations} \label{eq:GM}
\begin{align}
\mu_u(x | \ccalX_k) &= \sum\nolimits_{j=1}^{\bbarn} p_{j,k} \, \mu_u(x | \ccalX_k, \ccalM_j), 	\label{eq:GMmean}	\\
\gamma_u^2(x | \ccalX_k) &= \sum\nolimits_{j=1}^{\bbarn} p_{j,k} \, \gamma_u^2(x | \ccalX_k, \ccalM_j)  \label{eq:GMvar} \\
& + \sum\nolimits_{j=1}^{\bbarn} p_{j,k} \, \left[ \mu_u(x | \ccalX_k, \ccalM_j) - \mu_u(x | \ccalX_k) \right]^2 ,	\nonumber
\end{align}
\end{subequations}
where $p_{j,k} = \tdpi(\ccalM_j | \ccalX_k)$ denotes the posterior model probabilities obtained from \eqref{eq:hierarchical}. Equation \eqref{eq:GMvar} follows from the fact that the variance of a RV is the mean of conditional variances plus the variance of the conditional means. The expressions for $v(x) \and i(x)$ are identical.

\subsection{Measurement Model} \label{sec:measM}
Given a set of instantaneous velocity readings over a period of time at a given point $x$, in this section we discuss the measurement model \eqref{eq:measModel} for the mean velocity components and turbulent intensity; see Section \ref{sec:mobSen} for details on collecting such instantaneous velocity readings using a mobile robot. First, we make the following assumption, which we further discuss in Section \ref{sec:measNoise}.
%Note that by the ergodicity assumption, these time samples are equivalent to multiple realizations of the RV $\bbq(x,t)$; see Section \ref{sec:RANS}.

\begin{assumption} \label{assumption:noise}
Instantaneous sensor noise on velocity readings is additive with a zero-mean normal distribution.
\end{assumption}

% ------------------------------------------------------------- %
\subsubsection{Mean Velocity Components} \label{sec:meanVelComp}
Consider an observation $y_u(x,t)$ of the instantaneous first velocity component $u(x,t)$ at a point $x$ and time $t$, given by
\begin{equation} \label{eq:meas}
y_u(x,t) = u(x,t) + \epsilon_{u,s}(x,t) , 
\end{equation}
where $\epsilon_{u,s}(x,t) \sim \ccalN(0, \gamma_{u,s}^2 )$ is the instantaneous sensor noise with standard deviation $\gamma_{u,s}(x,t) \in \reals_+$. 
Assume that we collect $n$ samples of $y_u(x,t)$ at time instances $t_l$ for $1 \leq l \leq n$. Then, the sample mean of these observations given by
\begin{equation} \label{eq:meanEstim}
y_u(x) = \frac{1}{n} \sum\nolimits_{l=1}^n y_u(x,t_l),
\end{equation}
is a measurement of the first mean velocity component $u(x)$.

As in Proposition \ref{prop:CLT}, assume that the samples $y_u(x,t_l)$ are \textit{independent}, and thus uncorrelated. Then, the variance of $y_u(x)$ is given by
\begin{align} \label{eq:measVar}
\var[y_u(x)] &= \frac{1}{n^2} \sum\nolimits_{l=1}^n \var[y_u(x,t)] = \frac{1}{n} \, \var[y_u(x,t)] .
\end{align}
An unbiased estimator of $\var[y_u(x,t)]$ is the sample variance
\begin{equation} \label{eq:smpVar}
\hat{\var}[y_u(x,t)] =  \frac{1}{n-1} \sum\nolimits_{l=1}^n \left[ y_u(x,t_l) - y_u(x) \right]^2 . 
\end{equation}
Substituting this estimator in \eqref{eq:measVar}, we get an estimator for the variance of the mean velocity measurement $y_u(x)$ as
\begin{align} \label{eq:smpMvar}
\hat{\var}[y_u(x)] &= \frac{1}{n \, (n-1)} \sum\nolimits_{l=1}^n \left[ y_u(x,t_l) - y_u(x) \right]^2 . 
\end{align}

The samples $y_u(x,t_l)$ are not identically distributed since the sensor noise $\epsilon_{u,s}(x,t)$ in \eqref{eq:meas} is time-dependent; see Section \ref{sec:measNoise}. Nevertheless, they are normally distributed by Assumption \ref{assumption:noise}. Consequently, being a linear combination of these samples, $y_u(x)$ in \eqref{eq:meanEstim} is also normally distributed, justifying our measurement model in \eqref{eq:measModel}. Furthermore, since $\epsilon_{u,s}(x,t)$ is zero-mean, by Proposition \ref{prop:CLT}, $y_u(x)$ is an unbiased estimator of the ensemble average $\hhatbbq_1(x)$ and by \eqref{eq:smpMvar}, it becomes more accurate as the number of samples $n$ increases.
Note that, in addition to the uncertainty caused by sensor noise, captured in the sample mean variance \eqref{eq:smpMvar}, other sources of uncertainty may also contribute to the measurement uncertainty $\sigma_u^2(x)$ in \eqref{eq:measModel}. In Section \ref{sec:measNoise}, we derive an estimator for $\sigma_u^2(x)$ taking into account these additional sources of uncertainty which are due to the mobile robot that collects the measurements.

Next, we take a closer look at $\var[y_u(x,t)]$ and the terms that contribute to it. Since the variance of the sum of uncorrelated RVs is the sum of their variances, from \eqref{eq:meas} we have $\var[y_u(x,t)] = \var[u(x,t)] + \gamma_{u,s}^2(x,t)$.
Furthermore, since the samples $y_u(x,t_l)$ are uncorrelated, we have
$ \sum\nolimits_{l=1}^n \var[y_u(x,t_l)] = \sum\nolimits_{l=1}^n \set{ \var[u(x,t_l)] + \gamma_{u,s}^2(x,t_l) } ,$
%
%\begin{equation*} \label{eq:varDecom}
%\sum_{k=1}^n \var[y_u(x,t_k)] = \sum_{k=1}^n \set{ \var[u(x,t_k)] + \gamma_{u,s}^2(x,t_k) } ,
%\end{equation*}
%
and thus,
\begin{equation} \label{eq:varDecom}
\var[y_u(x,t)] = \var[u(x,t)] + \bbargamma_{u,s}^2(x) ,
\end{equation}
where
\begin{equation} \label{eq:varMean}
\bbargamma_{u,s}^2(x) = \frac{1}{n} \sum\nolimits_{l=1}^n \gamma_{u,s}^2(x,t_l)
\end{equation}
is the mean variance of the sensor noise. We derive an expression for the instantaneous variance $\gamma_{u,s}^2(x,t)$ in Section \ref{sec:measNoise}.
Notice from equation \eqref{eq:varDecom} that the variance of $y_u(x,t)$ is due to turbulent fluctuations and sensor noise $\epsilon_{u,s}(x,t)$. The former is a variation caused by turbulence and inherent to the RV $u(x)$ whereas the latter is due to the inaccuracy in sensors and further contributes to the uncertainty in $y_u(x)$.

%Referring to Section \ref{sec:RANS} and definition of turbulent intensity \eqref{eq:turbI}, $\var[q(x,t)] \approx i^2(x) q^2(x) \approx i^2(x) \mu^2_q(x)$. Note that the variance of velocity magnitude is not exactly given by turbulence intensity \eqref{eq:turbI} since mean of the norm does not equal the norm of mean. Thus, the expression above is an approximation.

% ------------------------------------------------------------- %
\subsubsection{Turbulent Intensity} \label{sec:}
Recall from Section \ref{sec:RANS} that $\bbq_{1,\text{rms}}^2(x) = \var[u(x,t)]$ and similarly for the other velocity components. Then, from the definition of turbulent intensity in \eqref{eq:turbI} and using equations \eqref{eq:varDecom} and \eqref{eq:smpVar}, we can get a measurement of $i(x)$ as
\begin{align} \label{eq:turbIapprox}
y_i(x) = \frac{1}{ q_{\text{ref}} } & \left[ \hat{\var}[y_u(x,t)] - \bbargamma_{u,s}^2(x)  \right.  \\
& \left. + \hat{\var}[y_v(x,t)]  - \bbargamma_{v,s}^2(x) \right]^{0.5} , \nonumber
\end{align}
where the variance $\hat{\var}[y_v(x,t)]$ of the instantaneous second velocity measurement and the corresponding mean sensor noise variance $\bbargamma_{v,s}^2(x)$ are computed using equations similar to \eqref{eq:smpVar} and \eqref{eq:varMean}, respectively.

Let $\epsilon_i(x) \sim \ccalN(0, \sigma_i^2)$ denote the additive error in measurement model \eqref{eq:measModel} for turbulent intensity. Unlike the mean velocity measurements, due to the nonlinearity in definition \eqref{eq:turbI}, it is not straight-forward to theoretically obtain an expression for the variance $\sigma_i^2$ and instead, we utilize the Bootstrap resampling method to directly estimate it; see Appendix \ref{app:turbIvar} for details.

%In general, estimating this variance requires the knowledge of higher order moments of the random velocity components; cf. \cite{TBUETS1996BG}. Specifically under a Gaussianity assumption, they depend on the mean values $y_u(x), y_v(x)$ and $y_w(x)$ which are not necessarily negligible. Consequently, we need to incorporate this uncertainty into our statistical model.
%
%Here, we utilize the Bootstrap resampling method to directly estimate the variance $\sigma_i^2$. \red{\hhati(x)} Assume that the samples $y_u(x,t_k)$ are independent and consider the measurement set $\ccalY_u = \set{y_u(x,t_k) \, | \, 1 \leq k \leq n}$; define $\ccalY_v$ and $\ccalY_w$ similarly. Furthermore, consider $n_b$ batches $\ccalB_j$ of size $n$ obtained by randomly drawing the same samples from $\ccalY_u$, $\ccalY_v$, and $\ccalY_w$ with replacement. Using \eqref{eq:turbIapprox}, we obtain $n_b$ estimates $\hhati_j(x)$ corresponding to the batches $\ccalB_j$. Then, the desired variance $\sigma_i^2(x)$ can be estimated as
%%
%\begin{equation} \label{eq:turbIvar}
%\hat{\sigma}_i^2(x) = \frac{1}{n_b-1} \sum_{j=1}^{n_b} (\hhati_j - \bbari)^2 ,
%\end{equation}
%%
%where $\bbari(x)$ is the mean of the batch estimates $\hhati_j(x)$; see \cite{IVE2007W} for more details.

% ------------------------------------------------------------- %
\subsubsection{Sampling Frequency} \label{sec:smpFreq}
In the preceding analysis, a key assumption is that the samples are independent. Since fluid particles cannot have unbounded acceleration, their velocities vary continuously and thus consecutive samples are dependent. As a general rule of thumb, in order for the samples to be independent, the interval between consecutive sample times $t_l$  and $t_{l+1}$ must be larger than $2 \, t^*(x)$, where 
$ t^*(x) = \max_k \bbt_k^*(x) $
is the maximum of integral time scales \eqref{eq:intTimeScale} along different directions at point $x$. See Appendix \ref{app:smpFreq} for practical guidelines on approximating \eqref{eq:intTimeScale} given a finite set of samples.

Since a large number of samples are needed to ensure the reliability of mean velocity and turbulent intensity measurements and these samples need to be temporally separated, both the time and energy cost of collecting measurements can be significant. This cost, which is often ignored in the relevant literature, can make purely data-driven approaches prohibitively expensive and impractical, as discussed in Section \ref{sec:intro}.

%In practice, we only have a finite set of samples of the instantaneous velocity vector over time and can only approximate \eqref{eq:intTimeScale}. Let $l$ denote the discrete lag. Then, the sample autocorrelation of first velocity component is given by
%%
%$$ \hat{\rho}_u(l) = \frac{1}{n} \sum_{k=1}^{n-l} [y_u(x,t_k) - y_u(x)] [y_u(x,t_{k+l}) - y_u(x)] . $$
%%
%This approximation becomes less accurate as $l$ increases since the number of samples used to calculate the summation decreases. Furthermore, the integral of the sample autocorrelation $\hat{\rho}_u(l)$ over the range $1 \leq l \leq n-1$ is constant and equal to $0.5$; see \cite{SSAF2009H}. This means that we cannot directly use \eqref{eq:intTimeScale} which requires integration over the whole time range. The most common approach is to integrate $\hat{\rho}_u$ up to the first zero-crossing; see \cite{AFDIL2004NNDJ}.

% ------------------------------------------------------------------------------------------------------------------------------ %
\section{Learning Flow Fields using a Mobile Robot} \label{sec:path}
In this section we develop a mobile robot sensor to measure the instantaneous velocity vector and extract the necessary measurements discussed in Section \ref{sec:measM}. We also characterize the uncertainty associated with collecting these measurements with a mobile robot and formulate a path planning problem to maximize the information content of these measurements.

% ------------------------------------------------------------- %
\subsection{Mobile Robot Design} \label{sec:mobSen}
%
%Given sensors that measure the inflow velocity magnitude in their axial direction, we design a flow velocity sensor set for a small robot so that it can measure the velocity angle and magnitude in a plane parallel to the ground. Note that any obstruction of the flow will decrease the accuracy of readings. Thus, to minimize interference from the mobile sensor on the flow pattern, we select a small differential drive robot that carries the sensors and collects measurements. %We also need to make sure that the structure of the robot do not change the flow pattern before it reaches the sensors.
%
%Furthermore, since the sensors only measure the flow accurately in their axial direction, we need to take into account their directivity characteristics.

%For simplicity, in what follows we focus on the in-plane components of the velocity field and take measurements on a plane parallel to the ground. Extension to the 3D case is straightforward.
To measure the in-plane velocity components, we use a set of D6F-W01A1 flow sensors from \textsc{OMRON}; see Section \ref{sec:results} for more details. Figure \ref{fig:directivity} shows the directivity pattern for this sensor measured empirically.
\begin{figure}[t!]
	\centering
	\includegraphics[width=0.4\textwidth]{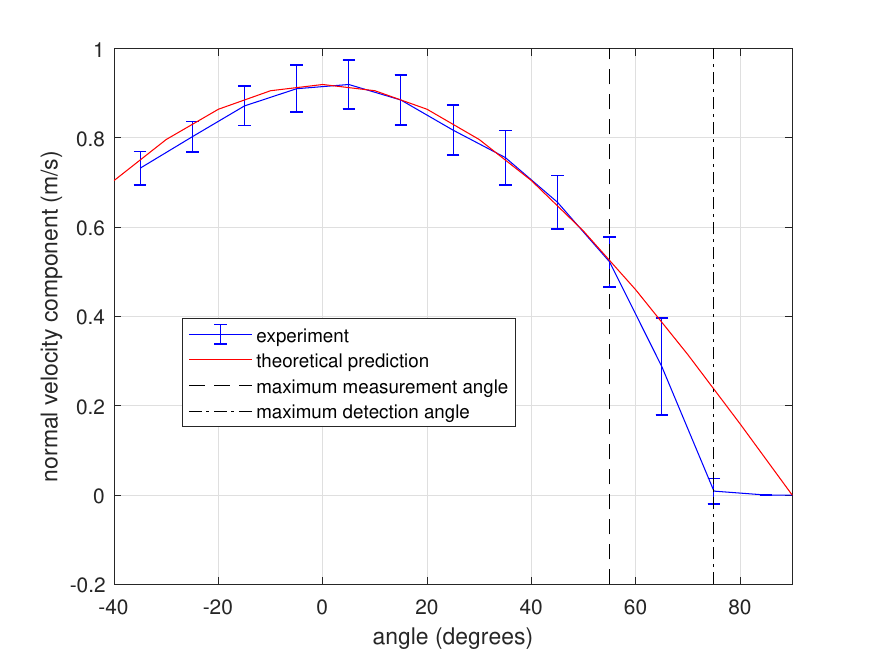}
	\caption{Directivity plot of the flow sensor used in the mobile robot where we measure the angle from the axis of the sensor. Above $55^o$ the measurements deviate from the theoretical prediction and above $75^o$ there is no detection. The bars on the experimental data show the standard deviation of the sensor samples caused by turbulent fluctuations and sensor noise. } \label{fig:directivity}
\end{figure}
In this figure, the normal velocity component $q \cos \hat{\theta}$ is given as a function of the angle $\hat{\theta}$ of the flow from the axis of the sensor. It can be seen that there exists an angle above which the measurements deviate from the mathematical formula and for an even larger angle, no flow is detected at all. The numerical values for these angles are roughly $55^o$ and $75^o$ for the sensor used here. This directivity pattern can be used to determine the maximum angle spacing between sensors mounted on a plane perpendicular to the robot's $z$-axis, so that any velocity vector on the $x-y$ plane can be measured with adequate accuracy. This spacing is $45^o$, which means that eight sensors are sufficient to sweep the plane.
%In Figure \ref{fig:directivity}, the bars show the oscillations in the flow. Note that below $35^o$, these oscillations increase meaning that the measurements of turbulent intensity \eqref{eq:turbI} also will be inaccurate below this angle.

\begin{figure}[t!]
  \begin{minipage}[c]{0.18\textwidth}
	\includegraphics[width=\textwidth]{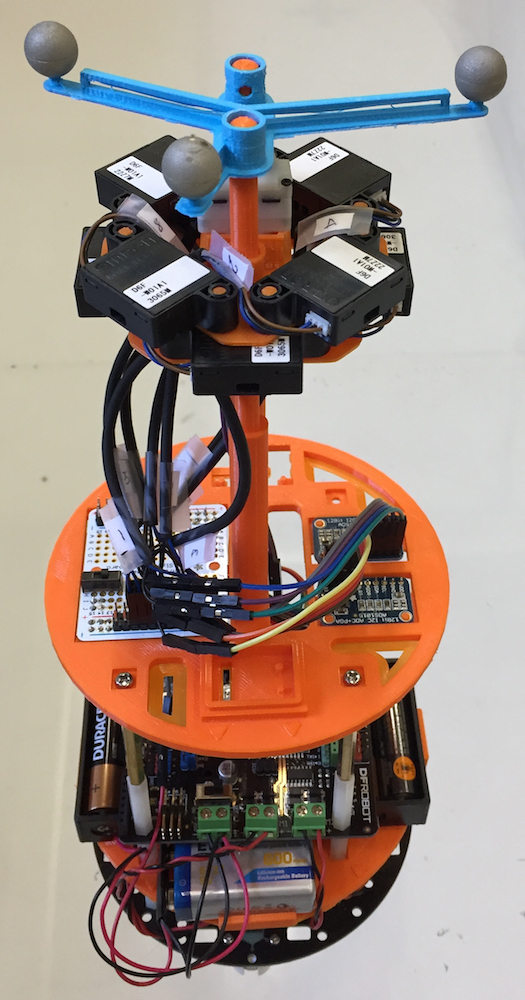}
  \end{minipage}
  \hspace{1mm}
  \begin{minipage}[t][][c]{0.29\textwidth}
  	\vspace{3mm}
	\caption{Structure of the mobile sensor and arrangement of the flow velocity sensors. Eight sensors are placed in two rows in the center of the robot with an angle spacing of $90^o$ in each row and combined spacing of $45^o$. We place the sensor rig at a higher plane away from the robot structure to minimize interference with the flow pattern.} \label{fig:robot}
  \end{minipage}
\end{figure}

Figure \ref{fig:robot} shows the mobile sensor comprised of these eight sensors separated by $45^o$ giving rise to a sensor rig that can completely cover the flow in the plane of measurements.
%
%\begin{figure}[t!]
%	\centering
%	\includegraphics[width=0.25\textwidth]{robot3.jpg}
%	\caption{Structure of the mobile sensor and arrangement of the flow velocity sensors. Eight sensors are placed in two rows in the center of the robot with an angle spacing of $90^o$ in each row and combined spacing of $45^o$. We place the sensor rig at a higher plane away from the robot structure to minimize interference with the flow pattern.} \label{fig:robot}
%\end{figure}
%
In this configuration, one sensor in the rig will have the highest reading at a given time, while one of its neighboring sensors having the next highest reading. Let $q$ and $\theta$ denote the magnitude and angle of the instantaneous flow vector and $s_j$ and $\beta_j$ denote the $j$-th sensor's reading and heading angle in the global coordinate system where $1 \leq j \leq 8$. Assume that sensors $j$ and $l$ have the two highest readings. Then, we have
\begin{align} \label{eq:sensVel}
& \ s_j = q \cos( \theta - \beta_j) \ \text{and} \ s_l = q \cos( \theta - \beta_l) \Rightarrow 	\\
& \left[
\begin{array}{c}
\sin \theta \\
\cos \theta
\end{array}
\right] = 
\frac{1}{q \sin(\beta_j - \beta_l)} 
\left[
\begin{array}{c}
s_j \cos \beta_l - s_l \cos \beta_j \\
- s_j \sin \beta_l + s_l \sin \beta_j
\end{array}
\right] .	\nonumber
\end{align}
These equations are used to get the four-quadrant angle $\theta$ in global coordinates. Given $\theta$, the magnitude $q$ is given by $q = s_j / \cos( \theta - \beta_j)$.
Assume that sensor $j$ has the highest reading at a given time, then $l=j+1$ and $\beta_j \leq \theta \leq \beta_{j+1}$ or $l=j-1$ and $\beta_{j-1} \leq \theta \leq \beta_j$. Any deviation from these conditions indicates inaccurate readings due to sensor noise or delays caused by the serial communication between the sensors and the microcontroller onboard; see Appendix \ref{app:hardware}.
Algorithm \ref{alg:flowmetry} in Appendix \ref{app:flowmetry} summarizes the procedure used by the mobile robot sensor to collect the desired measurements. Moreover, a video of the instantaneous velocity readings and the corresponding instantaneous velocity vector computed using Algorithm \ref{alg:flowmetry} can be found in \cite{meJ3_video1}.

\subsection{Measurement Noise} \label{sec:measNoise}
In Section \ref{sec:measM}, we defined the sensor noise for the instantaneous first velocity component by $\epsilon_{u,s}(x,t) \sim \ccalN(0, \gamma_{u,s}^2)$ and the corresponding measurement noise by $\epsilon_u(x) \sim \ccalN(0, \sigma_u^2)$; see equations \eqref{eq:meas} and \eqref{eq:measModel}, respectively. In the following, we derive an explicit expression for $\gamma_{u,s}$. We also discuss the  contribution of the robot heading $\gamma_{\beta}^2$ and location $\gamma_x^2$ errors to $ \sigma_u(x)$.
Since the relation between these noise terms and the measurements is nonlinear, we employ linearization so that the resulting distributions can be approximated by Gaussians. Specifically, let $y(\bbp)$ denote an observation which depends on a vector of uncertain parameters $\bbp \in \reals^{n_p}$. After linearization, $y \approx y_0 + \nabla y_0^T \, (\bbp - \bbp_0)$ where $y_0 = y(\bbp_0)$, $\nabla y_0 = \nabla y(\bbp_0)$, and $\bbp_0$ denotes the nominal value of uncertain parameters. For independent parameters $\bbp$, we have $(\bbp - \bbp_0) \sim \ccalN(\bb0, \bbGamma)$ where $\bbGamma$ is the constant diagonal covariance matrix. Then
$ y \sim \ccalN(y_0, \nabla y_0^T \, \bbGamma \, \nabla y_0 ), $
%
%\begin{equation*}  %\label{eq:lineGauss}
%y \sim \ccalN(y_0, \nabla y_0^T \, \bbGamma \, \nabla y_0 ),
%\end{equation*}
%
by the properties of the Gaussian distribution for linear mappings. Specifically, the linearized variance of measurement $y$ given the uncertainty in its parameters, can be calculated as
\begin{equation} \label{eq:lineVar}
\var[y] \approx \nabla y_0^T \, \bbGamma \, \nabla y_0 = \sum\nolimits_{j=1}^{n_p} \left( \frac{\partial \, y }{\partial \bbp_j} \Bigr|_{\bbp_{j,0} } \right)^2 \gamma_j^2 ,
\end{equation}
where $\gamma_j^2$ denotes the variance of the uncertain parameter $\bbp_j$. Consequently, since the sensor noise and heading error are independent, we can add their contributions to $\sigma_u^2(x)$.

First, we consider the sensor noise due to the flow sensors used onboard the robot; see Figure \ref{fig:robot}. We assume that the flow sensor noise is Gaussian.\footnote{Note the difference between the time-dependent `sensor noise' in \eqref{eq:meas} corresponding to the sensor rig in Figure \ref{fig:robot} and the identically distributed `flow sensor noise' corresponding to the individual flow sensors.} For flow sensors with a full-scale accuracy rating used here, this assumption is reasonable as the standard deviation is fixed and independent of the signal value. Let FS denote the full-scale error and $\gamma_s$ denote the standard deviation of the flow sensor noise. The cumulative probability of a Gaussian distribution within $\pm 3 \gamma_s$ is almost equal to one. Thus, we set
\begin{equation} \label{eq:sensorNosie}
\gamma_s = \frac{ \text{FS} }{3} .
\end{equation}

%In order to quantify the effect of sensor noise on the readings of velocity magnitude and components, we utilize the equation \eqref{eq:sensVel}.
%%
%First, consider the velocity magnitude $q(x)$. After some manipulation of \eqref{eq:sensVel}, we can rewrite it in terms of the two highest sensor readings $s_j$ and $s_l$ as
%%
%$$ y_q(x,t_k) = \abs{a} \left( s_j^2 + s_l^2 +b s_j s_l \right)^{0.5}, $$
%%
%where the costants
%%
%$$ a = \frac{1}{ \sin ( \beta_j - \beta_l) }, \ \ \ b = - 2 \cos( \beta_j - \beta_l), $$
%%
%depend on the angle spacing between sensors and $\beta_j$ and $\beta_l$ denote the sensor headings.
%%
%The gradient with respect to instantaneous sensor readings is
%%
%\begin{align*}
%& \nabla y_q =
%\left[
%\begin{array}{c}
%{\partial y_q}/{\partial s_j}   \\
%{\partial y_q}/{\partial s_l}   
%\end{array}
%\right]
%= \frac{ \abs{a} }{ 2 (s_j^2 + s_l^2 +b s_j s_l)^{0.5} }
%\left[
%\begin{array}{c}
%2 s_j + b s_l \\
%2 s_l + b s_j
%\end{array}
%\right] .
%\end{align*}
%%
%Noting that the sensor noise $\gamma_s$ is identical for $s_j$ and $s_l$, their contribution to $\gamma_q^2$ is given by $\gamma_s^2 \norm{\nabla y_{q,0} }^2$ from \eqref{eq:lineVar} and
%%
%\begin{align} \label{eq:sensNoiseQ}
%\norm{\nabla y_{q,0} }^2 &= \frac{a^2 \left[ (4+b^2) (s_{j,0}^2 + s_{l,0}^2) + 8 b s_{j,0} s_{l,0} \right] }{ 4 (s_{j,0}^2 + s_{l,0}^2 +b s_{j,0} s_{l,0}) } ,
%\end{align}
%%
%where $s_{j,0}$ and $s_{l,0}$ are the nominal values.

In order to quantify the effect of the flow sensor noise $\gamma_s$ on the instantaneous sensor noise $\gamma_{u,s}(x,t)$ in \eqref{eq:meas}, referring to equation \eqref{eq:sensVel}, we have
\begin{equation} \label{eq:instMeasU}
y_u(x,t_k) = q \cos \theta = a (- s_j \sin \beta_l + s_l \sin \beta_j ) ,
\end{equation}
where
$ a = {1}/{ \sin ( \beta_j - \beta_l) } $
depends on the angle spacing between sensors and $\beta_j$ and $\beta_l$ denote the sensor headings.
Then, $\nabla y_u = a [ - \sin \beta_l \ \ \sin \beta_j ]^T$ and
\begin{equation} \label{eq:sensNoiseU}
\gamma_{u,s}^2(x,t) = a^2 (\sin^2 \beta_{j,0} + \sin^2 \beta_{l,0}) \, \gamma_s^2
\end{equation}
from \eqref{eq:lineVar}, where $\beta_{j,0}$ and $\beta_{l,0}$ denote the nominal headings of the sensors in the global coordinate system. Similarly, 
\begin{equation} \label{eq:instMeasV}
y_v(x,t_k) = q \sin \theta = a ( s_j \cos \beta_l - s_l \cos \beta_j ) ,
\end{equation}
and $\nabla y_v = a [ \cos \beta_l  \  - \cos \beta_j ]^T$. Then,
\begin{equation} \label{eq:sensNoiseV}
\gamma_{v,s}^2(x,t) = a^2 (\cos^2 \beta_{j,0} + \cos^2 \beta_{l,0})  \, \gamma_s^2.
\end{equation}
Since equations \eqref{eq:instMeasU} and \eqref{eq:instMeasV} are linear in the instantaneous sensor readings, the contribution of flow sensor noise to the uncertainty in instantaneous velocity readings is independent of the sensor readings and only depends on the robot heading $\beta$. This also implies that Assumption \ref{assumption:noise} is valid as long as the flow sensor noise is normally distributed.

Next, we consider the uncertainty caused by localization error, i.e., error in the heading and position of the mobile robot. This error is due to structural and actuation imprecisions and can be partially compensated for by relying on motion capture systems to correct the nominal values; see Appendix \ref{app:hardware}.
Let $\beta \sim \ccalN(\beta_0, \gamma^2_{\beta})$ denote the distribution of the robot heading around the nominal value $\beta_0$ obtained from a motion capture system. Note that $\beta_j = \beta + \beta_{j,r}$ where $\beta_{j,r}$ is the relative angle of sensor $j$ in the local coordinate system of the robot; the same is true for $\beta_l$. Then, from \eqref{eq:instMeasU} we have
\begin{equation} \label{eq:headingU}
\frac{\partial y_u}{\partial \beta} = a ( - s_j \cos \beta_l + s_l \cos \beta_j ) = - y_v(x,t_k) ,
\end{equation}
where the last equality holds from \eqref{eq:instMeasV}.
Evaluating this expression at nominal values $s_{j,0}, s_{l,0}, \beta_{j,0}, \beta_{l,0}$ and using \eqref{eq:lineVar}, the contribution of the heading error to the uncertainty in the measurement of the instantaneous first velocity component is
\begin{align}  \label{eq:gammaU}
\gamma_{u, \beta}^2(x,t_k) = y_v^2(x,t_k) \, \gamma^2_{\beta} . 
\end{align}
Following an argument similar to the derivation of \eqref{eq:varDecom}, we can obtain an expression for $\bbargamma_{u, \beta}^2(x)$. Then, since from \eqref{eq:lineVar}, the contributions of independent sources of uncertainty, i.e., sensor and heading errors, are additive, we obtain an estimator for $\sigma_u^2(x)$ as
\begin{equation} \label{eq:uVar}
\hat{\sigma}_u^2(x) = \hat{\var}[y_u(x)] + \bbargamma_{u, \beta}^2(x),
\end{equation}
where, $\hat{\var}[y_u(x)]$ is given by \eqref{eq:smpMvar} and $\bbargamma_{u, \beta}^2(x)$ can be obtained using an equation similar to \eqref{eq:varMean}.
Note that unlike $\hat{\var}[y_u(x)]$ whose contribution to the uncertainty in $y_u(x)$ can be reduced by increasing the sample size $n$, the contribution of the heading error is independent of $n$ and can only be reduced by conducting multiple independent measurements.
Similarly, for the second component of flow velocity
\begin{equation} \label{eq:headingV}
\frac{\partial y_v}{\partial \beta} = a (- s_j \sin \beta_l + s_l \sin \beta_j ) = y_u(x,t_k) .
\end{equation}
Then using \eqref{eq:lineVar} and \eqref{eq:headingV}, we have
\begin{align}  \label{eq:gammaV}
\gamma_{v, \beta}^2(x,t_k) = y_u^2(x,t_k) \, \gamma^2_{\beta} 
\end{align}
and
\begin{equation} \label{eq:vVar}
\hat{\sigma}_v^2(x) = \hat{\var}[y_v(x)] + \bbargamma_{v, \beta}^2(x) .
\end{equation}
%
%Note that the velocity magnitude measurements is not a function of $\beta$ and not affected by the heading error. This implies that measurements of magnitude have relatively lower variance than the components and are more accurate; see the discussion at the end of Section \ref{sec:meanVelMagGP} as well. Particularly,
%%
%\begin{equation} \label{eq:gammaQ}
%\gamma_q^2(x,t) = \norm{\nabla y_{q,0} }^2 \gamma_s^2 ,
%\end{equation}
%%
%where $\norm{\nabla y_{q,0} }^2$ is given by equation \eqref{eq:sensNoiseQ} and we use instantaneous sensor readings as the nominal values. The mean measurement noise variance $\gamma_q^2(x)$, used in equations \eqref{eq:smpMvar} and \eqref{eq:turbImeas}, is defined by equation \eqref{eq:varMean}.

The procedure to incorporate the uncertainty due to measurement location error is outlined in Appendix \ref{app:measLocErr} where, we use a Gaussian distribution $x \sim \ccalN(x_0, \gamma^2_x \, \bbI_2)$ to encapsulate the uncertainty in measurement location; $x_0$ is the nominal location obtained from a motion capture system and $\bbI_2 \in \reals^{2\times2}$ denotes the identity matrix.

\subsection{Path Planning for Mobile Robot} \label{sec:OPP}
Next, we formulate a path planning problem for the mobile robot sensor to collect measurements with maximum information content. 
Specifically, let $\Omega$ denote a discretization of the 2D plane of the sensor rig excluding the points occupied by obstacles and let $\ccalR>0$ denote the constraint on maximum travel distance. Furthermore, given the current measurement location $x_k$ and set of currently collected measurements $\ccalX_k$, let $\ccalS_{k+1} = \set{x \in \Omega \setminus \ccalX_k \, | \, \norm{x - x_k} \leq \ccalR }$ denote the feasible subset of candidate measurement locations at step $k+1$.
Then, our goal is to select the next measurement location $x_{k+1}$ from $\ccalS_{k+1}$ so that the joint entropy of mean velocity components $u(x) \and v(x)$ at unobserved locations $\Omega \backslash \ccalX_{k+1}$, given the measurements in $\ccalX_{k+1}$, is minimized. With a slight abuse of notation, let $H(\Omega \backslash \ccalX_{k+1} \, | \, \ccalX_{k+1})$ denote this entropy. %where we replace the RVs at locations $\Omega \backslash \ccalX$ and $\ccalX$ with the locations themselves. 
Noting that $H(\Omega \backslash \ccalX_{k+1} \, | \, \ccalX_{k+1}) = H(\Omega) - H(\ccalX_{k+1})$, minimizing $H(\Omega \backslash \ccalX_{k+1} \, | \, \ccalX_{k+1})$ is equivalent to maximizing $H(\ccalX_{k+1})$.
Furthermore, by the chain rule of entropy
\begin{equation}
H(\ccalX_{k+1}) = H(x_{k+1} \, | \, \ccalX_k) + \dots + H(x_2 \, | \, \ccalX_1) + H(x_1) .
\end{equation}
Thus, we can find the next best measurement location $x_{k+1}$ by solving the following optimization problem
\begin{equation} \label{eq:greedy1}
x^*_{k+1} = \argmax_{x \in \ccalS_{k+1}} H(x \, | \, \ccalX_k) .
\end{equation}

%Our goal is to select $m$ measurement locations from the set $\ccalS$, which we collect in the set $\ccalX = \set{x_k \in \ccalS \, | \, 1 \leq k \leq m}$, so that the entropy of the velocity components at unobserved locations $\Omega \backslash \ccalX$, given the measurements $\ccalX$, is minimized. With a slight abuse of notation, let $H(\Omega \backslash \ccalX \, | \, \ccalX)$ denote this entropy. %where we replace the RVs at locations $\Omega \backslash \ccalX$ and $\ccalX$ with the locations themselves. 
%Then, we are interested in the following optimization problem
%%
%\begin{equation*}
%\ccalX^* = \argmin_{\ccalX \subset \ccalS, \abs{\ccalX} = m}  H(\Omega \backslash \ccalX \, | \, \ccalX) ,
%\end{equation*}
%%
%where $\abs{\ccalX}$ denotes the cardinality of the measurement set $\ccalX$. Noting that $H(\Omega \backslash \ccalX \, | \, \ccalX) = H(\Omega) - H(\ccalX)$, we can equivalently write
%%
%\begin{equation} \label{eq:optimPlace}
%\ccalX^* = \argmax_{\ccalX \subset \ccalS, \abs{\ccalX} = m}  H(\ccalX) .
%\end{equation}
%%
%The optimization problem \eqref{eq:optimPlace} is combinatorial and NP-complete; see \cite{NSPGP2008KSG}. This makes finding the optimal set $\ccalX^*$ computationally expensive as the size of the set $\ccalS$ grows.

Next we derive an expression for the objective $H(x \, | \, \ccalX_k) = H( u(x), v(x) \, | \, \ccalX_k )$ in \eqref{eq:greedy1}. Since from Assumption \ref{ass:velCompCorr}, $u(x)$ and $v(x)$ are assumed independent, we have 
\begin{equation} \label{eq:jointEnt}
H(u,v) = H(u \, | \, v) + H(v) = H(u) + H(v) ,
\end{equation}
where we have dropped dependence on $x \and \ccalX_k$ for simplicity.
Recall from Section \ref{sec:hierarchical} that the posterior distributions of $u(x) \and v(x)$ are GMs, for which closed-form expressions for $H(u) \and H(v)$ are unavailable \cite{EAGMRV2008HBD}. Instead, we optimize the expected entropy over the $\bbarn$ models. 
Particularly, given a model $\ccalM_j$ and measurements $\ccalX_k$, $u(x)$ is normally distributed according to \eqref{eq:inference}. Then, the value of (differential) entropy is independent of the mean and is given in closed-form as
\begin{equation} \label{eq:entropyGauss}
H(u(x \, | \, \ccalX_k, \ccalM_j) ) = \log(c \, \gamma_u(x  \, | \,  \ccalX_k, \ccalM_j) ) , 
\end{equation}
where $c = \sqrt{2 \pi e}$ and $\gamma_u^2(x  \, | \,  \ccalX_k, \ccalM_j)$ is defined in \eqref{eq:predVar}.
Given \eqref{eq:entropyGauss}, the expected entropy for $u(x)$ is given by
\begin{equation}
H(u(x \, | \, \ccalX_k) ) = \sum\nolimits_{j=1}^{\bbarn} p_{j,k} \, H(u(x \, | \, \ccalX_k, \ccalM_j) ) ,
\end{equation}
where $p_{j,k} = \tdpi(\ccalM_j \, | \, \ccalX_k)$ denotes the probability of model $\ccalM_j$, given the current measurements $\ccalX_k$, obtained from \eqref{eq:hierarchical}.
A similar expression holds true for $v(x)$. Then, from \eqref{eq:jointEnt}
\begin{equation}
H(x \, | \, \ccalX_k) = \sum_{j=1}^{\bbarn} p_{j,k} \, \log \left[ c^2 \, \gamma_u(x \, | \, \ccalX_k, \ccalM_j) \, \gamma_v(x \, | \, \ccalX_k, \ccalM_j) \right] 
\end{equation}
and we can rewrite the planning problem \eqref{eq:greedy1} explicitly as
\begin{equation} \label{eq:greedy}
x^*_{k+1} = \argmax_{x \in \ccalS_{k+1}} \sum\nolimits_{j=1}^{\bbarn} p_{j,k} \, \delta I_H(x  \, | \,  \ccalX_k, \ccalM_j) ,
\end{equation}
where
\begin{equation} \label{eq:addediH}
\delta I_H(x \, | \,  \ccalX_k, \ccalM_j) \propto \gamma_u^2(x  \, | \,  \ccalX_k, \ccalM_j) \, \gamma_v^2(x  \, | \,  \ccalX_k, \ccalM_j) 
\end{equation}
measures the information added by a potential measurement at $x \in \ccalS_{k+1}$, given current measurements $\ccalX_k$ and model $\ccalM_j$.

Using the closed-form expression \eqref{eq:addediH}, the planning problem \eqref{eq:greedy} can be solved very efficiently. Algorithm \ref{alg:NBP} summarizes our proposed solution to this problem at step $k+1$.
\begin{algorithm}[t]
\caption{Next Best Measurement Location}
\label{alg:NBP}
\begin{algorithmic}[1]
\small

\REQUIRE Covariance matrix $\barbSigma_{\Omega \Omega}$, the sets $\ccalX_k$ and $\ccalS_{k+1}$, and current model probabilities $p_{j, k} = \tdpi(\ccalM_j \, | \, \ccalX_k)$;

\FOR {$x \in \ccalS_{k+1}$} 	\label{line:loopCandMeas}
	
	\FOR{$j=1:\bbarn$}				\label{line:loopMode}
	
		\STATE Compute $\delta I_H(x \, | \, \ccalX_k, \ccalM_j)$ using \eqref{eq:addediH};
	
	\ENDFOR
	
\ENDFOR
	
\STATE Set $x^*_{k+1} = \argmax_{x \in \ccalS_{k+1}} \sum_{j=1}^{\bbarn} p_{j,k} \, \delta I_H(x  \, | \, \ccalX_k, \ccalM_j)$;	\label{line:NBML}

\end{algorithmic}
\end{algorithm}
To speed up the online computation, the algorithm requires the precomputed covariance $\barbSigma_{\Omega \Omega}$ between all discrete points in $\Omega$. For the correlation function \eqref{eq:correlation}, this matrix is sparse and can be computed efficiently. Given the current discrete probability distribution $p_{j,k}$ over models, the algorithm iterates through the candidate measurement locations $\ccalS_{k+1}$ and discrete models in lines \ref{line:loopCandMeas} and \ref{line:loopMode}, respectively, to compute the amount of added information $\delta I_H$. Finally in line \ref{line:NBML}, it selects the location $x^*_{k+1}$ with the highest expected added information as the next measurement location. Note that the evaluations of the entropy metric \eqref{eq:addediH} at different candidate locations are independent and can be done in parallel.

Algorithm \ref{alg:NBP} is suboptimal in that it only maximizes the information content of the next immediate measurement. It is possible to consider a longer horizon although at the expense of exponentially increased computational cost.
Note also that Algorithm \ref{alg:NBP} exhaustively evaluates the objective in \eqref{eq:addediH} for all candidate measurement locations in $\ccalS_{k+1}$. Given that the domain $\bbarOmega$ is generally non-convex, the prior uncertainty fields \eqref{eq:ustd} are highly nonlinear, and the planning objective \eqref{eq:addediH} can be computed efficiently, this exhaustive approach is effective in practice and there is no need to formulate and solve sophisticated optimization problems.
In the relevant literature, more sophisticated planning algorithms are often only studied for simple convex domains and constant prior uncertainty fields for which the planning problem reduces to a simple exploration; see e.g. \cite{EBSPMSNCMRF2013XCDM,LEFECCUR2019DGKS}.

\subsection{The Integrated System} \label{sec:inetgSys}
Algorithm \ref{alg:ARF} summarizes our proposed integrated system to actively learn turbulent flow fields.
\begin{algorithm}[t]
\caption{Physics-Based Learning of Flow Fields using a Mobile Robot}
\label{alg:ARF}
\begin{algorithmic}[1]
\small

\REQUIRE Prior numerical solutions $\ccalM_j$, prior distribution $\tdpi(\ccalM_j)$;

\REQUIRE Prior uncertainties $\bbarsigma_{u,0}, \bbarsigma_{v,0}, \and \bbarsigma_{i,0}$ for each model;

\REQUIRE Standard deviations $\gamma_s$, $\gamma_{\beta}$, and $\gamma_x$;

\REQUIRE Maximum number of measurements $m$;

\REQUIRE The set $\ccalX_{\bbarm}$ of $\bbarm$ exploration measurements;

%\STATE Collect exploration measurements $\ccalX_{\bbarm}$;

\STATE Set the measurement index $k = \bbarm$;

\WHILE { the algorithm has not converged and $k \leq m$}
	
	\STATE Collect measurements using Algorithm \ref{alg:flowmetry};	\label{line:meas}
	
	\STATE Compute mean variances from \eqref{eq:uVar} and \eqref{eq:vVar};	\label{line:samplMeanVar}

	\STATE Compute turbulent intensity variance from \eqref{eq:turbIvar};	\label{line:samplMeanVar2}
	
	\STATE Compute the mean values and variances from \eqref{eq:measLocE} considering the measurement location error;	\label{line:locErr}
	
	\STATE Compute measurement likelihood $\bbarpi(\bby_{u,k}, \bby_{v,k}, \bby_{i,k} \, | \, \ccalM_j)$;	\label{line:likelihood}
	
	\STATE Update model probabilities $p_{j, k}$ using \eqref{eq:hierarchical};
	
	\STATE Compute posterior GPs using \eqref{eq:inference};	\label{line:posterior}
	
	\STATE Check the convergence criterion \eqref{eq:stop};	\label{line:stop}
	
	\STATE Select $x_{k+1}$ using Algorithm \ref{alg:NBP};	\label{line:planning}
		
	\STATE $\ccalX_{k+1} = \ccalX_k \cup \set{x_{k+1} }$;
		
	\STATE $k \leftarrow k+1$;
	
\ENDWHILE

\STATE $m \leftarrow k$;

\STATE Return the discrete model probabilities $\tdpi(\ccalM_j \, | \, \ccalX_m)$;

\STATE Return the posteriors $\ccalG \ccalP(x | \ccalX_m, \ccalM_j)$ for $1 \leq j \leq \bbarn$;

\end{algorithmic}
\end{algorithm}
The algorithm begins by requiring the $\bbarn$ physics-based models of the flow properties and their discrete prior probabilities $p_{j,0} = \tdpi(\ccalM_j)$ for $1 \leq j \leq \bbarn$. Given prior uncertainty estimates $\bbarsigma_{u,0}, \bbarsigma_{v,0}, \and \bbarsigma_{i,0}$ for each model, it constructs the GP models \eqref{eq:GPu} and \eqref{eq:GPi} for the mean velocity components and turbulent intensity; see Section \ref{sec:statModel}. It also requires the standard deviations of the flow sensor noise $\gamma_s$, the heading error $\gamma_{\beta}$, and the location error $\gamma_x$; see Section \ref{sec:measNoise}. Finally, the total number of measurements $m$ and an initial set of $\bbarm$ exploration measurement must also be provided.

Given these inputs, in line \ref{line:meas}, at every iteration $k$, the mobile sensor collects measurements of the velocity field and computes $y_u(x_k), y_v(x_k)$, and $y_i(x_k)$ using Algorithm \ref{alg:flowmetry}.
Then, in line \ref{line:samplMeanVar}, it computes the sample mean variances $\hat{\sigma}^2_u(x_k)$ and $\hat{\sigma}^2_v(x_k)$ using equations \eqref{eq:uVar} and \eqref{eq:vVar}. To do so, the algorithm computes the instantaneous variances due to the heading error using equations \eqref{eq:gammaU} and \eqref{eq:gammaV} and the corresponding mean values using \eqref{eq:varMean}. Then, it adds these values to the sample variances of the mean velocity components given by \eqref{eq:smpMvar}. These values quantify the effect of sensor noise and heading error on the measurements; see Section \ref{sec:measNoise}.
Next, in line \ref{line:samplMeanVar2} it computes the variance of turbulent intensity measurement.
%Note that \eqref{eq:smpMvar} specifically defines the value for the mean velocity magnitude $q(x)$ but the expression for the components $u(x)$ and $v(x)$ are identical after appropriate substitution of fields; see Remark \ref{rem:velComp}.
%
Using this information, in line \ref{line:locErr}, the algorithm computes the mean values $\tdmu_u(x_k), \tdmu_v(x_k)$, and $\tdmu_i(x_k)$ and the variances $\tdsigma_u^2(x_k), \tdsigma_v^2(x_k)$, and $\tdsigma_i^2(x_k)$ from equations \eqref{eq:measLocE} taking into account the measurement location error; see Appendix \ref{app:measLocErr}.
Given these measurements and their corresponding variances, in line \ref{line:likelihood} the likelihood of the measurements given each RANS model $\ccalM_j$ is assessed; see Section \ref{sec:hierarchical}. %Noting that the mean velocity components $u(x)$, $v(x)$, and the turbulent intensity $i(x)$ are assumed independent, their joint likelihood is the multiplication of their individual likelihoods. 
Then, given this likelihood, the posterior discrete probability of the models given current measurements $\ccalX_k$ is obtained from \eqref{eq:hierarchical}.
%Then, the algorithm obtains the posterior GP fields in line \ref{line:posterior} using the closed-form equations \eqref{eq:inference}. These posterior fields are required to check the stopping criterion \eqref{eq:stop} and more importantly to make predictions about unobserved regions of the domain. 
%
Finally, given the model probabilities and conditional GPs of velocity components, in line \ref{line:planning}, the planning problem \eqref{eq:greedy} is solved using Algorithm \ref{alg:NBP} to determine the next measurement location.

As the number of collected measurements increases, the amount of information added by a new measurement will decrease since entropy is a submodular function \cite{NSPGP2008KSG}. This means that the posterior estimates of the mean velocity field and turbulent intensity field will reach a steady state. Thus to terminate Algorithm \ref{alg:ARF}, in line \ref{line:stop} we check the change in these posterior fields averaged over the model probabilities. Particularly, we check if
\begin{equation} \label{eq:stop}
d_k = \sum\nolimits_{j=1}^{\bbarn} p_{j,k} \, (d_{u,k,j} + d_{v,k,j} + q_{\text{ref}} \, d_{i,k,j}) < tol ,
\end{equation}
where $tol$ is a user-defined tolerance and the difference in conditional fields of $u(x)$ at iteration $k$, is given by
\begin{equation} \label{eq:errorU}
d_{u,k,j} = \frac{1}{| \Omega | } \int \abs{\mu_u(x \, | \, \ccalX_k, \ccalM_j) - \mu_u(x \, | \, \ccalX_{k-1}, \ccalM_j) } d \, \Omega ,
\end{equation}
where $| \Omega |$ is the area of the 2D projection of domain $\bbarOmega$. The other two values $d_{v,k,j}$ and $d_{i,k,j}$ are defined in the same way.

% ------------------------------------------------------------------------------------------------------------------------------ %
\section{Results} \label{sec:results}
%
% This plot belongs to the simulation case study and is put here for better arrangement.
%
\begin{figure*}
	\centering
	\begin{subfigure}[b]{0.23\textwidth}
		\includegraphics[width=\textwidth]{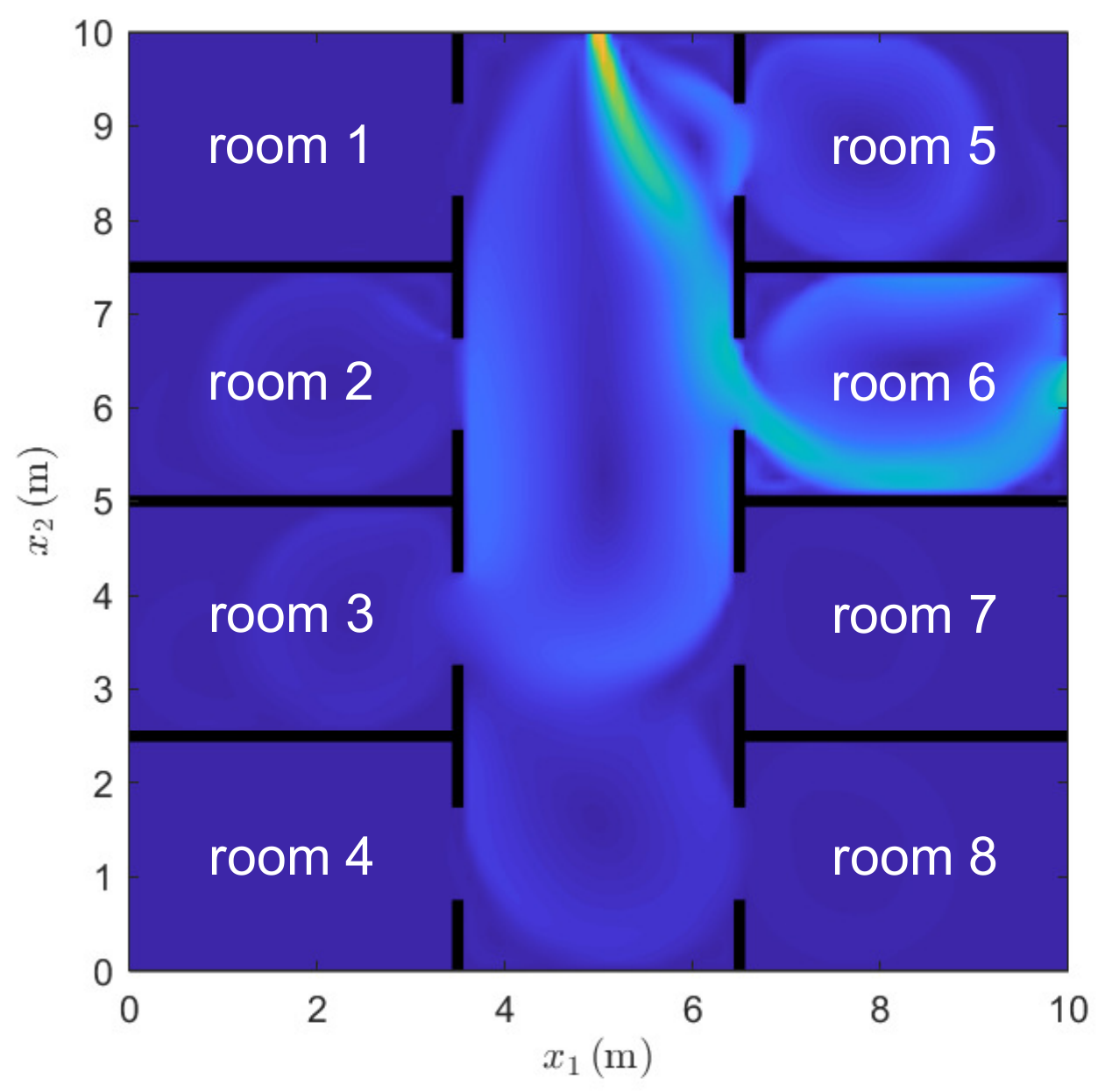}
	\captionsetup{justification=centering}
	\caption{ground truth} \label{fig:simTrue}
	\end{subfigure}
%	\quad
	\begin{subfigure}[b]{0.235\textwidth}
		\includegraphics[width=\textwidth]{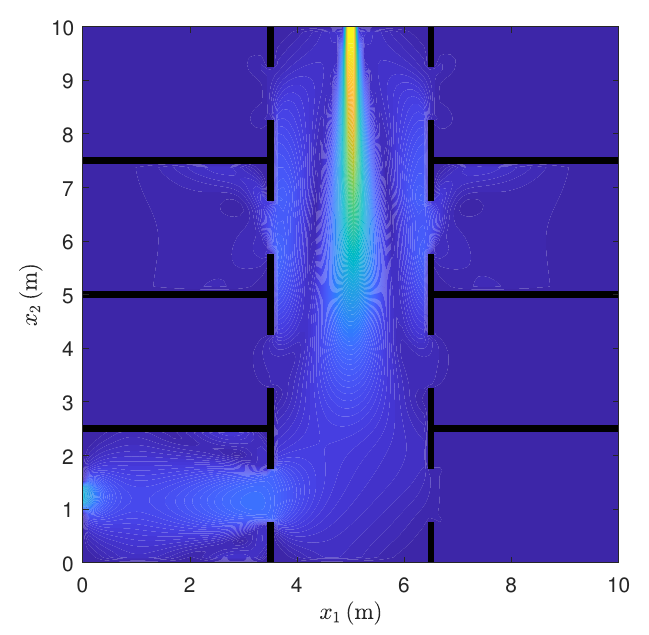}
	\captionsetup{justification=centering}
	\caption{outlet $4$}
	\end{subfigure}
	\begin{subfigure}[b]{0.23\textwidth}
		\includegraphics[width=\textwidth]{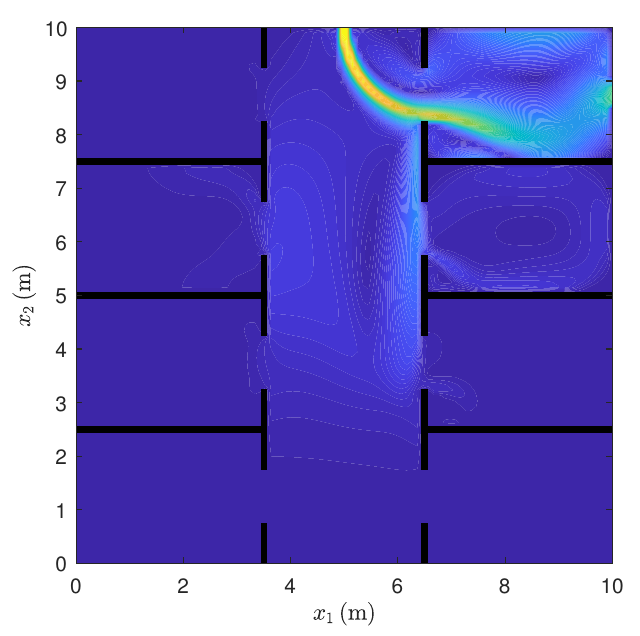}
	\captionsetup{justification=centering}
	\caption{outlet $5$}
	\end{subfigure}
	\begin{subfigure}[b]{0.27\textwidth}
		\includegraphics[width=\textwidth]{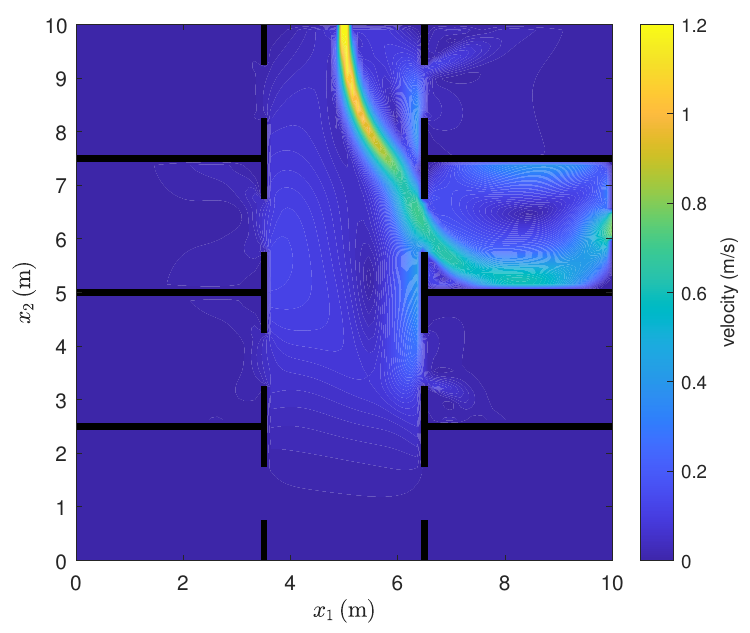}
	\captionsetup{justification=centering}
	\caption{outlet $6$}
	\end{subfigure}
	\caption{The ground truth and predicted velocity magnitude fields for outlets $4$-$6$. The ground truth field is obtained using $k-\epsilon$ RANS model whereas the rest of the fields are obtained using $k-\omega$ model.}\label{fig:simFlds}
\end{figure*}
%
% Theme: roadmap
In this section, we present simulation and experimental results to demonstrate the performance of our proposed environmental sensing Algorithm \ref{alg:ARF}.
Specifically, we first consider a large-scale simulated environment with a known ground truth. In this case-study, we (i) show that our proposed framework scales to large environments, (ii) demonstrate the performance of our proposed planning Algorithm \ref{alg:NBP} compared to a baseline lattice placement,
%\red{(iii) show the importance of active sensing in the presence of uncertain parameters,}
(iii) demonstrate the superior performance of our physics-based method compared to a purely data-driven approach and finally, (iv) provide a qualitative justification for the use of entropy metric \eqref{eq:addediH}.
Then, we present an experimental study to demonstrate the robustness of Algorithm \ref{alg:ARF} to real-world uncertainties and show that it (i) correctly selects models that are in agreement with empirical data, (ii) reaches a steady state as the number of measurements increases and, most importantly, (iii) predicts the flow properties more accurately than the prior numerical models while providing reasonable uncertainty bounds for the predictions. 

% Theme: software and parameters
We utilize \textsc{ANSYS Fluent} to solve the RANS models for the desired flow properties using the $k-\epsilon$, $k-\omega$, and RSM solvers \cite{TMCFD1993W} and \textsc{MATLAB} to implement Algorithm \ref{alg:ARF}. Throughout the results, the dimension of velocity values is m/s and turbulent intensity is dimensionless. We set the reference velocity in \eqref{eq:turbI} to $q_{\text{ref}} =1$\,m/s and use $n_0 = 200$ nominal samples in equation \eqref{eq:ustd}.
%
% Theme: characteristic length
We select the correlation characteristic length in equation \eqref{eq:correlation} by inspecting the integral length scale of the turbulent flow provided by the RANS models; see \cite{TMCFD1993W}.
For the BC values used in this paper, a correlation characteristic length of $\ell = 0.35$\,m works generally well; it provides a reasonable correlation between each measurement and its surroundings and guarantees an acceptable spacing between robot waypoints.
Alternatively, we could adopt a data-driven approach and learn the characteristic length that maximizes the likelihood of empirical data or infer its distribution in a stochastic framework; see \cite{meC3,EBSPMSNCMRF2013XCDM}.
In the absence of any prior knowledge, we assign uniform probabilities to all models in Algorithm \ref{alg:ARF}, i.e., $p_{j,0} = 1/\bbarn$ for $1 \leq j \leq \bbarn$.
As discussed in Section \ref{sec:mobSen}, the mobile sensor is equipped with eight D6F-W01A1 \textsc{OMRON} sensors with a range of $0-1$\,m/s and full-scale error rate of FS $ = 0.05$\,m/s.
For consistency, in both simulated and experimental cases, we set the standard deviation of the flow sensor noise to $\gamma_s = 0.017$\,m/s according to \eqref{eq:sensorNosie}, the heading error to $\gamma_{\beta} = 5^o$, and the measurement location error to $\gamma_x = 0.025$\,m. We construct a SROM model for the location error with $\tdn = 5$ samples; see Appendix \ref{app:measLocErr}. Note that the location error also takes into account the discretization of the numerical solutions and the error caused by the sensor rig structure.

% Theme: prediction error
To compute the prediction error of the conditional models, we compare their predictions at a set of $\hhatm$ locations to the ground truth or empirical measurements, whichever is available. Let $\hhaty_u(x_l)$ denote this value for the first velocity component at a location $x_l$ for $1 \leq l \leq \hhatm$. 
Let also $\mu_u(x_l | \ccalX_k)$ denote the predicted mean velocity obtained from \eqref{eq:GMmean} after conditioning on the previously collected measurements $\ccalX_k$ up to iteration $k$.
Then, we define the mean prediction error for the first mean velocity component over these $\hhatm$ locations as
$e_{u,k} = {1}/{\hhatm} \sum\nolimits_{l=1}^{\hhatm} \abs{ \mu_u(x_l | \ccalX_k) - \hhaty_u(x_l) } ,$
and the total mean prediction error as
\begin{equation} \label{eq:predErr}
e_k = \frac{1}{3} ( e_{u,k} + e_{v,k} + q_{\text{ref}} \, e_{i,k} ) ,
\end{equation}
where the expressions for $e_{v,k}$ and $e_{i,k}$ are identical. We are particularly interested in $e_0$ as the prediction error of the prior model and $e_m$ as the prediction error of the posterior model. 
We also define $\bbare_{k,j}$ for the individual models $\ccalM_j$ by using $\mu_u(x_l | \ccalX_k, \ccalM_j)$ in definition \eqref{eq:predErr} instead of the mean value from \eqref{eq:GMmean}; similarly for $v(x)$ and $i(x)$.

% ------------------------------------------------------------- %
\subsection{Environmental Sensing Simulation} \label{sec:sim}
In this section, we consider a $10\times10$\,m$^2$ simulated office environment with $q_{\text{in}} = 1$\,m/s velocity inlet at the top with an angle of $-80^o$, inlet turbulent intensity of $i_{\text{in}} = 0.02$, and a velocity outlet in room $6$; see Figure \ref{fig:simTrue}. We utilize \textsc{ANSYS Fluent} and the $k-\omega$ RANS model to obtain the ground truth mean velocity field shown in Figure \ref{fig:simTrue}. Then, we simulate the turbulent flow at a point $x$, using t-distributions centered at the predicted mean velocity components at that point and consider sensor noise and heading and location error in simulating velocity readings.
Assuming that we do not know which outlet is open \textit{a priori}, we define a discrete distribution over the status of the outlets $1$-$8$; see Section \ref{sec:hierarchical}. We generate the pool of prior numerical solutions for these eight outlets by solving the $k-\epsilon$ RANS model for a $q_{\text{in}} = 1.2$\,m/s velocity inlet with a perpendicular angle $-90^o$ and inlet turbulent intensity of $i_{\text{in}} = 0.05$, instead of the ground truth above; such errors are inevitable in practice. For simplicity however, we do not include additional uncertain parameters for RANS models or the erroneous inlet parameters. Figure \ref{fig:simFlds} includes the prior mean velocity magnitude fields for numerical models corresponding to outlets $4$-$6$.
%
%\begin{figure*}
%	\centering
%	\begin{subfigure}[b]{0.23\textwidth}
%		\includegraphics[width=\textwidth]{simulation/velMag_true.pdf}
%	\captionsetup{justification=centering}
%	\caption{ground truth} \label{fig:simTrue}
%	\end{subfigure}
%%	\quad
%	\begin{subfigure}[b]{0.235\textwidth}
%		\includegraphics[width=\textwidth]{simulation/velMag4.eps}
%	\captionsetup{justification=centering}
%	\caption{outlet $4$}
%	\end{subfigure}
%	\begin{subfigure}[b]{0.23\textwidth}
%		\includegraphics[width=\textwidth]{simulation/velMag5.eps}
%	\captionsetup{justification=centering}
%	\caption{outlet $5$}
%	\end{subfigure}
%	\begin{subfigure}[b]{0.27\textwidth}
%		\includegraphics[width=\textwidth]{simulation/velMag6.eps}
%	\captionsetup{justification=centering}
%	\caption{outlet $6$}
%	\end{subfigure}
%	\caption{Ground truth and predicted velocity magnitude fields for outlets $4$-$6$. The ground truth field is obtained using the $k-\epsilon$ RANS model whereas the rest of the fields are obtained using the $k-\omega$ model.}\label{fig:simFlds}
%\end{figure*}
%
We set $\bbarsigma_{u,0} = \bbarsigma_{v,0} = 0.05$\,m/s and $\bbarsigma_{i,0} = 0.05$ in equation \eqref{eq:ustd} for prior uncertainties of all models. To examine the performance of a purely data-driven model, we add to the pool of models, a constant prior flow field with $u(x) = 0$\,m/s and $v(x)=-0.2$\,m/s and $i(x) = 0.05$ with $\bbarsigma_{u,0} = \bbarsigma_{v,0} = 0.3$\,m/s and $\bbarsigma_{i,0} = 0.3$. Thus, we have $\bbarn=9$ prior models for this case study.
In the following, we use $\bbarm = 16$ exploration measurements and set the maximum number of measurements to $m = 200$.

% Theme: planning comparison
%
Figure \ref{fig:simplanOptim} depicts the prediction error \eqref{eq:predErr} as a function of the number of measurements, computed over the whole domain, i.e., $\hhatm = 245,537$, when compared to the ground truth. Specifically, we consider three different planning scenarios: (i) using Algorithm \ref{alg:NBP} without a travel distance constraint, (ii) with a travel distance constraint of $\ccalR = 1$\,m and,
%(iii) using an offline version of Algorithm \ref{alg:NBP} where the prior turbulent intensity field and model probabilities are used and,
(iii) a set of baseline lattice placements from $4\times4$ to $15\times15$ measurements.
\begin{figure}[t!]
	\centering
	\begin{subfigure}[b]{0.35\textwidth}
		\includegraphics[width=\textwidth]{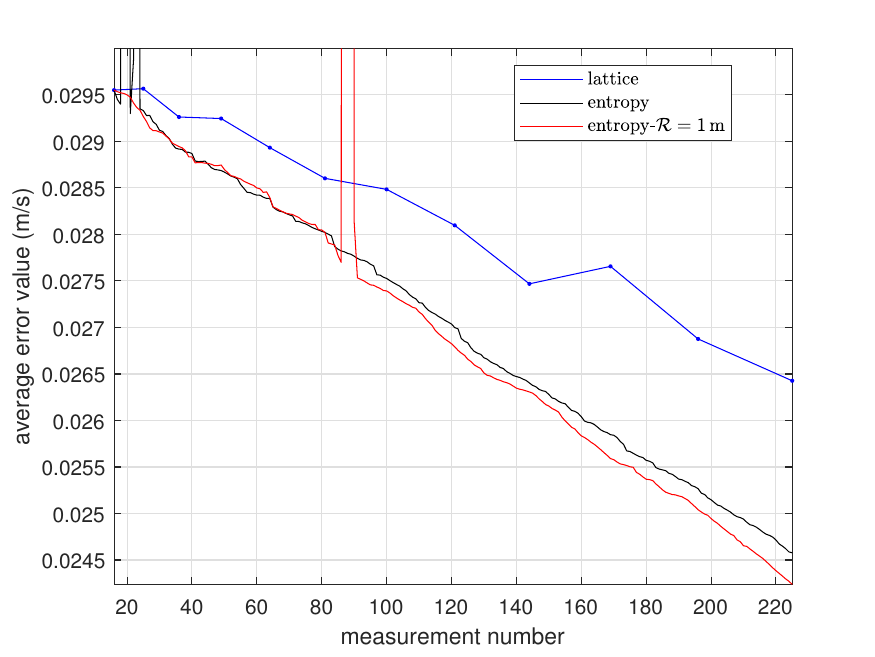}
	\captionsetup{justification=centering}
	\caption{prediction error \eqref{eq:predErr}} \label{fig:simplanOptim}
	\end{subfigure}
%	\quad
	\begin{subfigure}[b]{0.35\textwidth}
		\includegraphics[width=\textwidth]{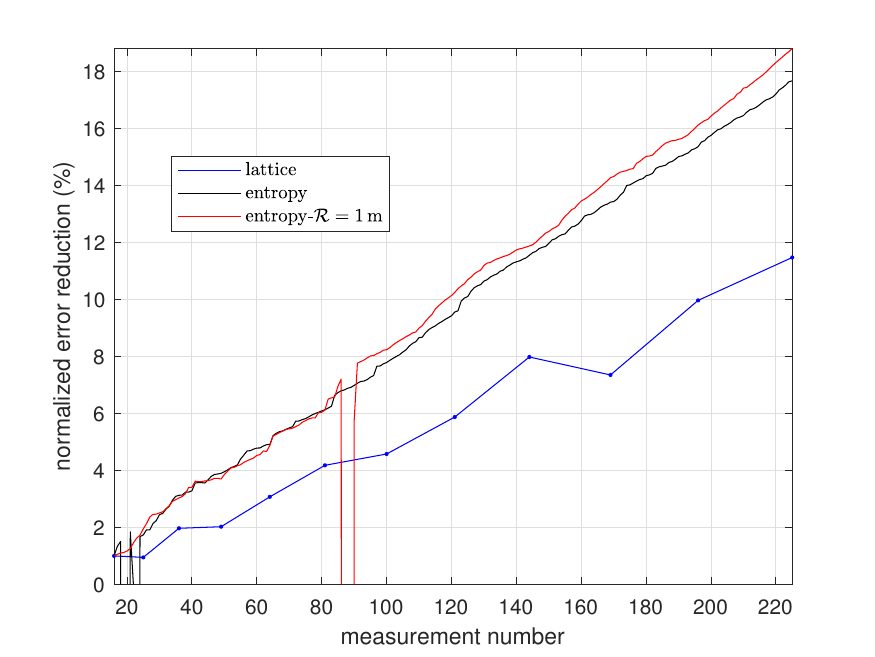}
	\captionsetup{justification=centering}
	\caption{normalized error reduction} \label{fig:simplanOptim0}
	\end{subfigure}
	\caption{Prediction error \eqref{eq:predErr} and the normalized reduction in prediction error compared to the most likely model $6$ as a function of number of measurements for the three planning scenarios. The discontinuities are due to oscillations in the distribution of models and stabilize once enough measurements are collected to determine the most likely model; see Section \ref{sec:hierarchical}.} \label{fig:}
\end{figure}
Observe that both versions of our proposed planning scheme considerably and consistently outperform the baseline lattice placement.
Note that a travel constraint of $\ccalR=1$\,m does not affect the prediction performance; obviously very small values of $\ccalR$ would impede the performance.
%Note also that the prediction error of the offline planning oscillates considerably until enough information is collected and Model $6$ is identified as the most likely model. This does not happen for the online version which uses the latest available information for planning \cite{NALGP2007KG}.
Figure \ref{fig:simWaypoints} shows the sequence of waypoints selected using Algorithm \ref{alg:NBP} under scenario (ii).
\begin{figure}[t!]
  \centering
    \includegraphics[width=0.35\textwidth]{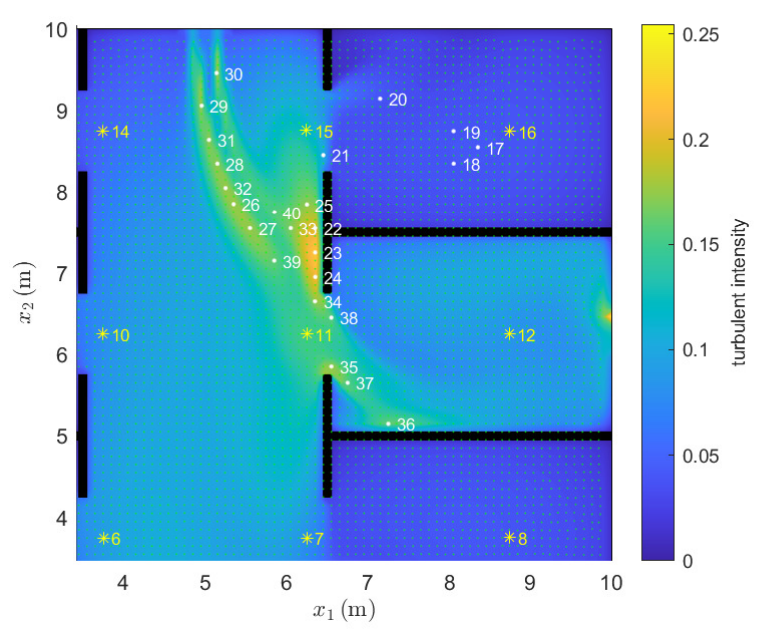}
  \caption{Sequence of the first $40$ measurements selected using Algorithm \ref{alg:NBP} under scenario (ii) with a travel constraint of $\ccalR=1$\,m, overlaid on zoomed-in turbulent intensity field from model $6$. The green dots show the candidate measurement locations and the yellow stars show the exploration measurement locations. Finally, the white dots show the sequence of waypoints.}\label{fig:simWaypoints}
\end{figure}

Model $6$ is the most likely among the prior numerical solutions since it is solved for the correct outlet albeit with a different RANS model and inaccurate inlet BCs. Its prior prediction error is $\bbare_{0,6} = 0.030$\,m/s.
Figure \ref{fig:simplanOptim0} shows the percentage of reduction in prediction error compared to the prior error of Model $6$ for each scenario.
The posterior prediction error of the purely data-driven approach is $\bbare_{225,9} = 0.082$\,m/s and $\bbare_{225,9} = 0.084$\,m/s using the measurements collected by our proposed planning scheme (scenario i) and lattice placement (scenario iii), respectively. Note that the posterior error of the purely data-driven approach is $273$\% higher than the prior error of the most likely model. This highlights the superiority of a physics-based solution to this environmental sensing problem.
As a reference, the posterior prediction error of the model obtained using Algorithm \ref{alg:ARF} under planning scenario (i) is $e_{225} = 0.025$\,m/s which is $328$\% lower than the purely data-driven approach.

Figure \ref{fig:simIdealMetric} shows the true error field for the most likely model $6$ compared to the ground truth. This field is the ideal planning metric that ranks each measurement location according to the amount of prediction error at that point. Figure \ref{fig:simEntropy} shows the entropy metric \eqref{eq:addediH} for model $6$ before collecting any measurements.
\begin{figure}[t!]
	\centering
	\begin{subfigure}[b]{0.241\textwidth}
		\includegraphics[width=\textwidth]{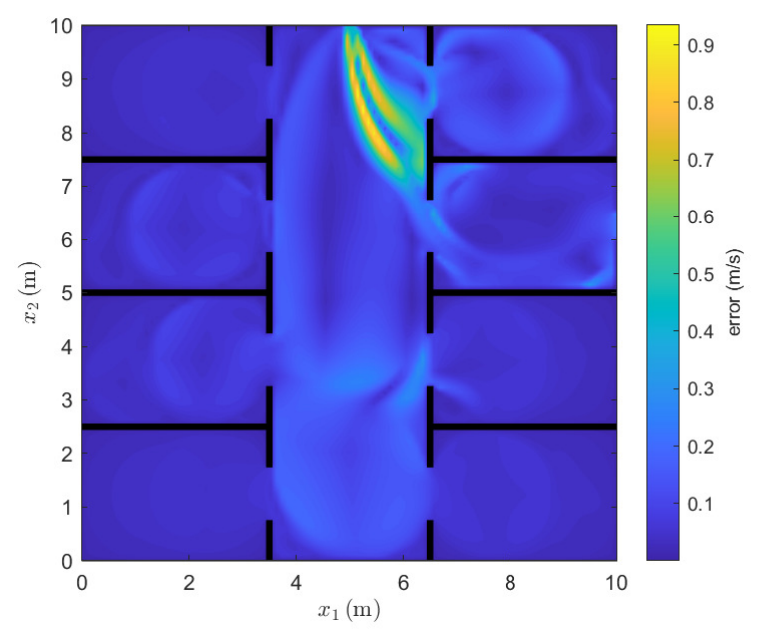}
	\captionsetup{justification=centering}
	\caption{ideal planning metric} \label{fig:simIdealMetric}
	\end{subfigure}
%	\quad
	\begin{subfigure}[b]{0.239\textwidth}
		\includegraphics[width=\textwidth]{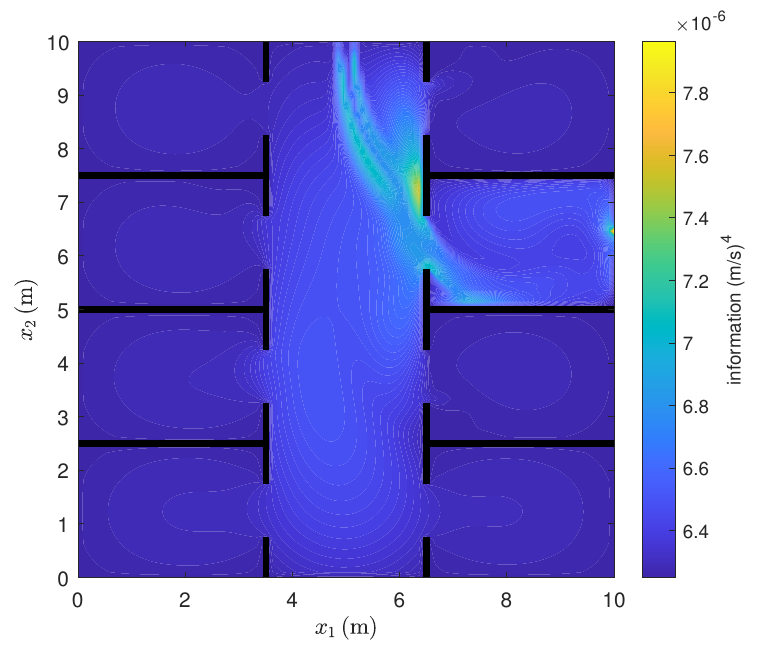}
	\captionsetup{justification=centering}
	\caption{entropy planning metric} \label{fig:simEntropy}
	\end{subfigure}
	\caption{The qualitative resemblance of the true error field (ideal planning metric) and the entropy metric \eqref{eq:addediH} for the most likely model $6$.} \label{fig:simMetrics}
\end{figure}
It can be seen that the entropy metric qualitatively resembles this field, confirming the observation that highly turbulent regions of the environment often correlate with larger error and should be prioritized in planning. Our planning Algorithm \ref{alg:NBP} is designed based on this observation.

% ------------------------------------------------------------- %
\subsection{Environmental Sensing Experiment} \label{sec:exp}
%
% Theme: RANS solver, domain, BCs
Next, we demonstrate the robustness of our proposed framework to significant uncertainties present in the real-world by considering an experiment in a $2.2 \times 2.2 \times 0.4$\,m$^3$ domain with an inlet, an outlet, and an obstacle inside as shown in Figure \ref{fig:domain}; the origin of the coordinate system is located at the bottom left corner of the domain.
\begin{figure}[t!]
  \centering
    \includegraphics[width=0.35\textwidth]{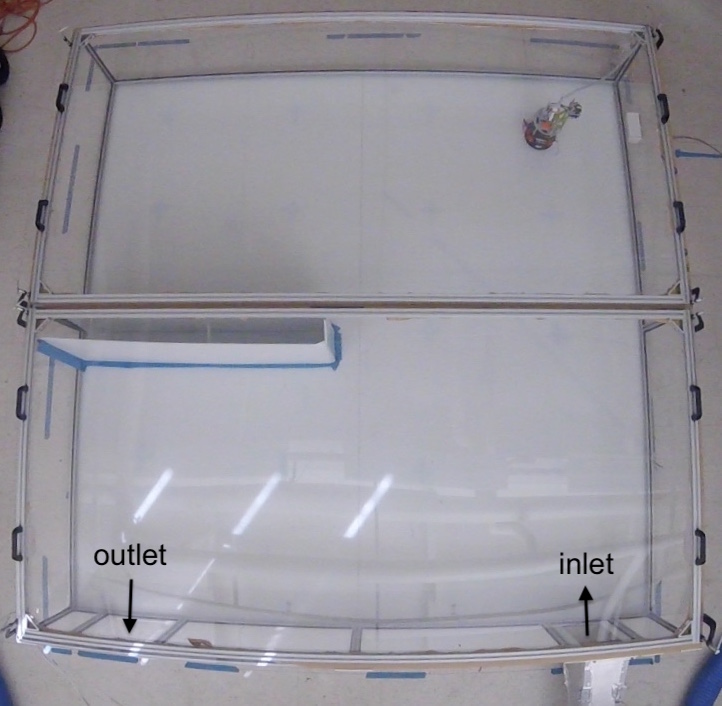}
  \caption{Domain of the experiment. A $2.2 \times 2.2 \times 0.4$m$^3$ box with velocity inlet at bottom right and outlet at bottom left. The origin of the coordinate system is located at the bottom left corner.} \label{fig:domain}
\end{figure}
We use a fan to generate a flow at the inlet with average velocity $q_{\text{in}} = 0.78$\,m/s and utilize the mobile robot designed in Section \ref{sec:mobSen} to conduct the experiment; see Appendix \ref{app:hardware} for details on hardware and control of the robot. See Appendix \ref{app:sigProcess} for a discussion on how to process the instantaneous velocity readings and the accuracy of the probabilistic measurement models developed in Section \ref{sec:statModel}.

We assume uncertainty in the inlet velocity and turbulent intensity values. This results in $\bbarn = 12$ different combinations of BCs and RANS models; see the first four columns in Table \ref{table:param} for details.
\begin{table}[t!]
\centering
\renewcommand{\arraystretch}{1.2}
\caption{BCs and prior uncertainty values for the pool of numerical solutions obtained for combinations of solvers and uncertain parameters.}
\footnotesize
\begin{tabular}{|c||c|c|c||c|c|} 
 \hline
 No. 			& model 		& $q_{\text{in}}$ (m/s)& $i_{\text{in}} $	& $\bbarsigma_{u,0}$ (m/s)	& $\bbarsigma_{i,0}$ 	\\ [0.5ex] 
 \hline\hline
 $1$			& $k-\epsilon$	& $0.78$		& $0.02$ 		&	$0.10$		& $0.05$			\\ 	 \hline
 $2$			& RSM		& $0.78$		& $0.02$ 		&	$0.20$		& $0.10$			\\ 	 \hline
 $3$			& $k-\omega$	& $0.78$		& $0.02$ 		&	$0.20$		& $0.10$			\\ 	 \hline
 $4$			& $k-\epsilon$	& profile		& $0.02$ 		&	$0.14$		& $0.07$			\\ 	 \hline
 $5$			& RSM		& profile		& $0.02$ 		&	$0.20$		& $0.10$			\\ 	 \hline
 $6$			& $k-\omega$	& profile		& $0.02$ 		&	$0.20$		& $0.10$			\\ 	 \hline
 $7$			& RSM		& profile		& $0.05$ 		&	$0.14$		& $0.07$			\\ 	 \hline
 $8$			& RSM		& profile		& $0.03$ 		&	$0.20$		& $0.10$			\\ 	 \hline
 $9$			& RSM		& profile		& $0.01$ 		&	$0.20$		& $0.10$			\\ 	 \hline
 $10$		& RSM		& profile		& $0.04$ 		&	$0.20$		& $0.10$			\\ 	 \hline
 $11$		& RSM		& $0.76$		& $0.03$ 		&	$0.20$		& $0.10$			\\ 	 \hline 
 $12$		& RSM 		& $0.80$	 	& $0.03$		&	$0.14$		& $0.07$			\\
 \hline
\end{tabular}
\label{table:param}
\end{table}
In the third column, `profile' refers to cases where the inlet velocity is modeled by an interpolated function instead of the constant value  $q_{\text{in}} = 0.78$\,m/s. Columns 5 and 6 show the prior uncertainty in the solutions of the first velocity component and turbulent intensity, where we set $\bbarsigma_{v,0} = \bbarsigma_{u,0}$; see equation \eqref{eq:ustd}. As discussed in Section \ref{sec:statModel}, these values should be selected to reflect the uncertainty in the numerical solutions. Here, we use the residual values provided by \textsc{ANSYS Fluent} as an indicator of the confidence in each numerical solution.
Figure \ref{fig:priorFlds} shows the velocity magnitude fields obtained using models 1 and 2.
\begin{figure}[t!]
	\centering
	\begin{subfigure}[b]{0.22\textwidth}
		\includegraphics[width=\textwidth]{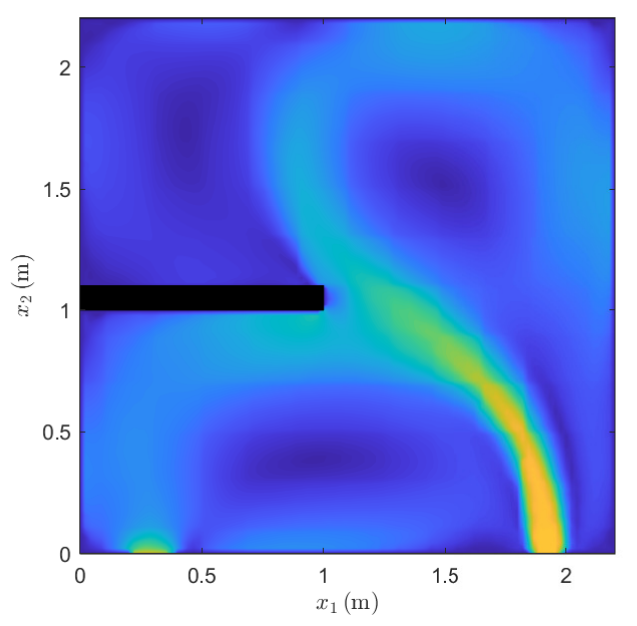}
	\captionsetup{justification=centering}
	\caption{$k-\epsilon$ model}
	\end{subfigure}
%	\quad
	\begin{subfigure}[b]{0.26\textwidth}
		\includegraphics[width=\textwidth]{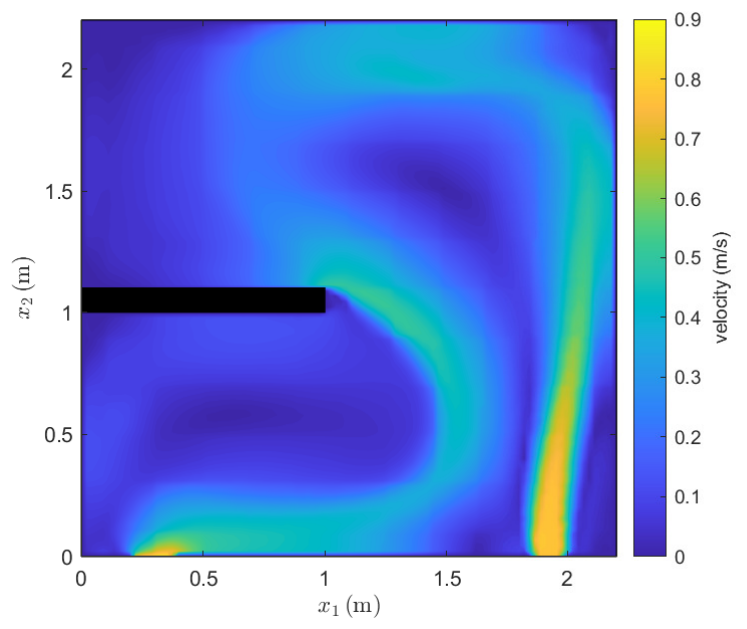}
	\captionsetup{justification=centering}
	\caption{RSM}
	\end{subfigure}
	\caption{Predictions of the velocity magnitude field according to models $1$ and $2$ in the plane of the mobile sensor located at the height of $0.27$\,cm.}\label{fig:priorFlds}
\end{figure}
The former is obtained using the $k-\epsilon$ model whereas the latter is obtained using the RSM, as reported in Table \ref{table:param}. Note that these two solutions are inconsistent and require experimental validation to determine the correct flow pattern.
% Theme: initialization
%We initialize Algorithm \ref{alg:ARF} by assigning uniform probabilities to all models in Table \ref{table:param}, i.e., $p_{j,0} = 1/\bbarn$ for $1 \leq j \leq \bbarn$.
We use $\bbarm = 9$ exploration measurements and set the maximum number of measurements to $m = 200$ and the maximum travel distance to $\ccalR = 1$\,m. The convergence tolerance in \eqref{eq:stop} is set to $tol = 2\times10^{-4}$\,m/s.

% Theme: placement
Figure \ref{fig:planning} shows the sequence of first 30 waypoints selected by Algorithm \ref{alg:NBP}. %For these results, we use a correlation characteristic length of $\ell = 0.35$m. This value provides a reasonable correlation between each measurement and its surroundings and guarantees an acceptable spacing between waypoints.
\begin{figure}[t!]
  \centering
    \includegraphics[width=0.35\textwidth]{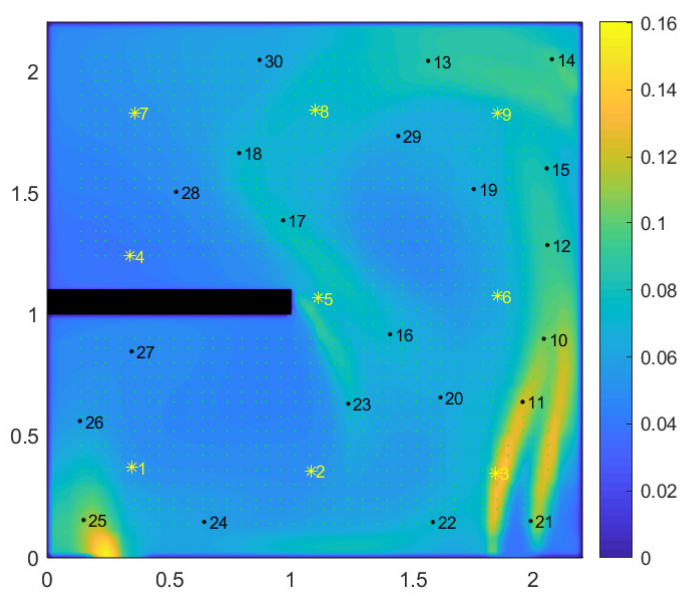}
  \caption{Path of the mobile sensor according to the entropy metric \eqref{eq:addediH} overlaid on the turbulent intensity field from model $7$. The green dots show the candidate measurement locations and the yellow stars show the exploration measurement locations. Finally, the black dots show the sequence of waypoints. For better clarity, only the first 30 measurements are included.}\label{fig:planning}
\end{figure}
The green dots in Figure \ref{fig:planning} show the $1206$ candidate measurement locations, collected in the set $\Omega$, and the yellow stars show the $\bbarm = 9$ exploration measurement locations selected over a lattice.
The black dots show the sequence of waypoints returned by Algorithm \ref{alg:NBP}.
%Note that as the size of the set $\Omega$ grows, it will take longer to evaluate the planning metric at all candidate locations in Algorithm \ref{alg:NBP}. The resolution of these candidate locations should be selected in connection with the characteristic length $\ell$. For the larger values of $\ell$, the resolution could be lower. Note also that the computations of the entropy metric \eqref{eq:addediH} at different locations are independent and can be done in parallel.% Finally, for very large domains we can consider a subset of candidate measurement locations from the set $\ccalS$ that lie in the neighborhood of the current location of the robot at every step.
%Referring to Figure \ref{fig:planning}, note that further reducing the travel distance constraint will certainly reduce energy consumption but will be detrimental to the primary goal of collecting informative measurements.
%
According to the convergence criterion \eqref{eq:stop}, Algorithm \ref{alg:ARF} converges after $k=155$ measurements.
%
% Theme: planning objective
Figure \ref{fig:deltaI} shows the added information using the entropy metric \eqref{eq:addediH} after the addition of each of these measurements. %where the correlation function \eqref{eq:correlation} is computed according to $L_g(x)$.
\begin{figure}[t!]
  \centering
    \includegraphics[width=0.35\textwidth]{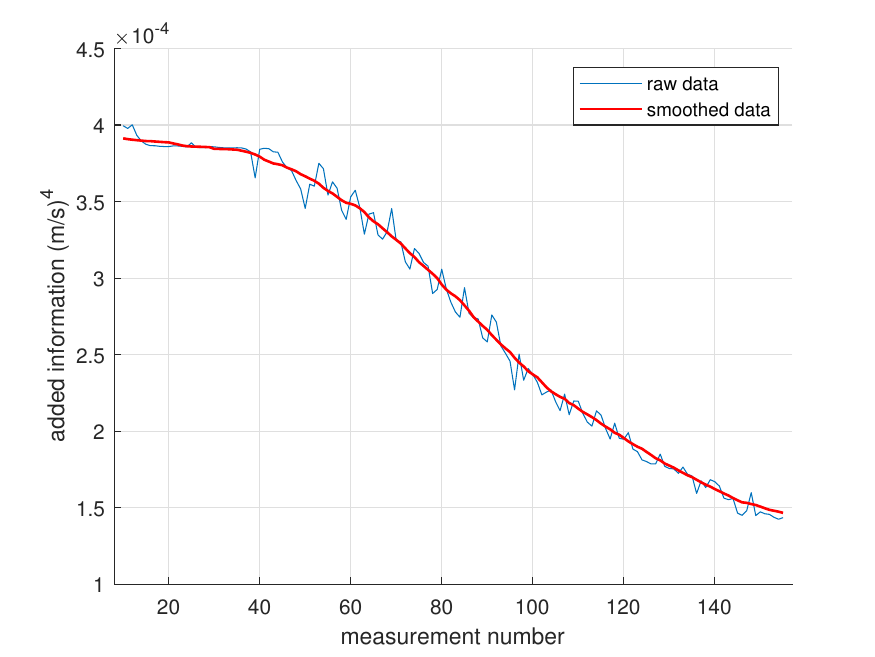}
  \caption{Added information vs measurement number for entropy metric \eqref{eq:addediH}.} \label{fig:deltaI}
\end{figure}
It can be observed that the amount of added information generally decreases as the mobile robot keeps adding more measurements. This is expected by the submodularity of the entropy information metric \eqref{eq:addediH}.
The oscillations in this figure are due to the travel distance constraint that might prevent the selection of the most informative measurement location at every step $k$.
%
% Theme: vector field
Figure \ref{fig:vecFld} shows the collected velocity vector measurements.
\begin{figure}[t!]
  \centering
    \includegraphics[width=0.33\textwidth]{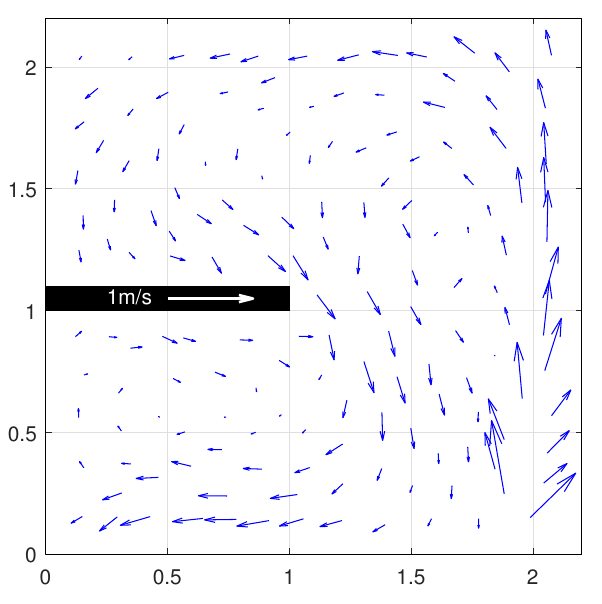}
  \caption{Velocity vector measurements for the experiment.} \label{fig:vecFld}
\end{figure}
Referring to Figure \ref{fig:priorFlds}, observe that this vector field qualitatively agrees with model 2 that was obtained using the RSM. The smooth streamlines that are clearly visible in this figure indicate that the interference from the robot is negligible.
% Theme: learning metric
Also, Figure \ref{fig:convergence} shows the evolution of the convergence criterion \eqref{eq:stop} of Algorithm \ref{alg:ARF}. The high value of this criterion for measurement $\bbarm=9$ corresponds to conditioning on all exploration measurements at once; the next measurements do not alter the posterior fields as much.
\begin{figure}[t!]
  \centering
    \includegraphics[width=0.35\textwidth]{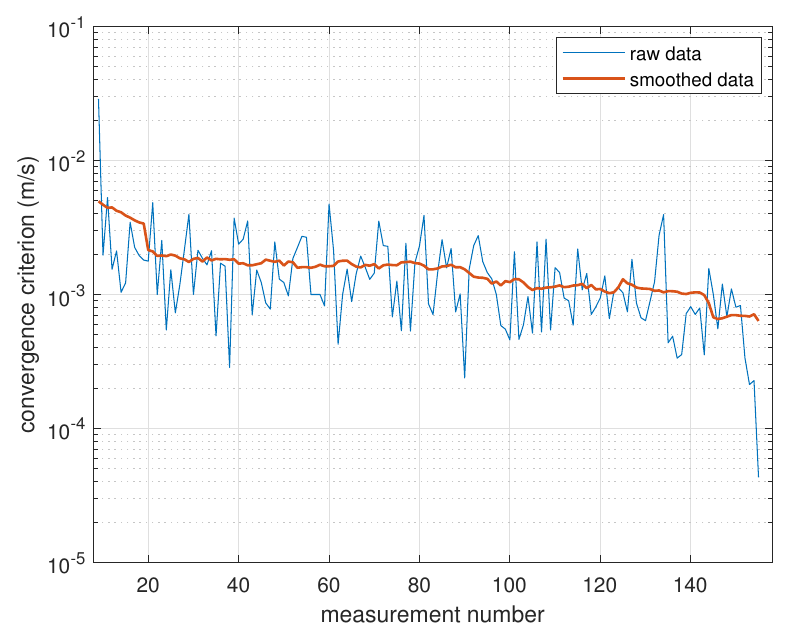}
  \caption{Evolution of convergence criterion \eqref{eq:stop} vs measurement number $k$.} \label{fig:convergence}
\end{figure}
%
%Note that the model probabilities $p_{j,k}$ used in these calculations are the posteriors given all measurements up to measurement $k$.

% Theme: model probabilities
%Figure \ref{fig:modelProb} shows the posterior probability of the models after conditioning on the measurements. %collected according to entropy metric and using $L_g(x)$ to compute the correlations.
%
%\begin{figure}[t!]
%  \centering
%    \includegraphics[width=0.48\textwidth]{modelP.eps}
%  \caption{Posterior model probabilities given all measurements.} \label{fig:modelProb}
%\end{figure}
%
The posterior probabilities $p_{j,k}$ converge shortly after the exploration measurements are collected and do not change afterwards. Particularly, the numerical solution from the RSM model $7$ is the only solution to have nonzero probability, i.e., $p_{7,155} = 1$. This means that the most accurate model can be selected given a handful of measurements that determine the general flow pattern. It is important to note that all solutions provided by RSM share a similar pattern and the empirical data help to select the most accurate model.
Note also that these posterior probabilities are computed given `only' the available models listed in Table \ref{table:param} and they should be interpreted with respect to these models and not as absolute probability values.

% Theme: posterior velocity field and turbulent intensity and their variances
In Figure \ref{fig:postFld}, the prior velocity magnitude and turbulent intensity fields corresponding to the most likely model $7$ and the posterior fields, computed using equations \eqref{eq:GM}, are given.
\begin{figure*}
	\centering
	\begin{subfigure}[b]{0.221\textwidth}
		\includegraphics[width=\textwidth]{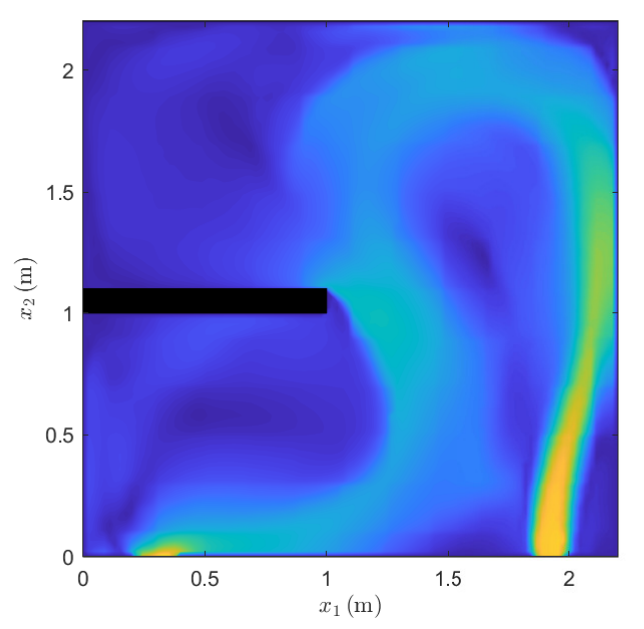}
	\captionsetup{justification=centering}
	\caption{prior velocity magnitude} \label{fig:}
	\end{subfigure}
%	\quad
	\begin{subfigure}[b]{0.256\textwidth}
		\includegraphics[width=\textwidth]{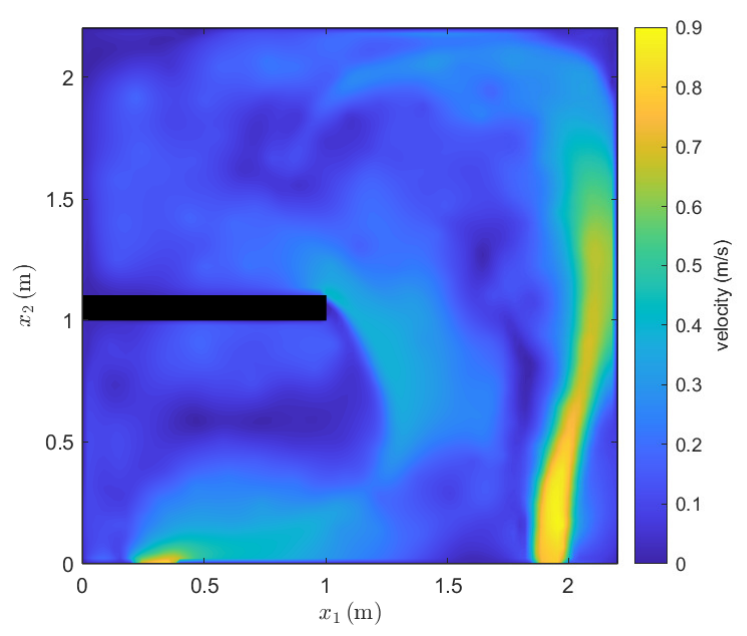}
	\captionsetup{justification=centering}
	\caption{posterior velocity magnitude} \label{fig:velMagPost}
	\end{subfigure}
%	\quad
	\begin{subfigure}[b]{0.221\textwidth}
		\includegraphics[width=\textwidth]{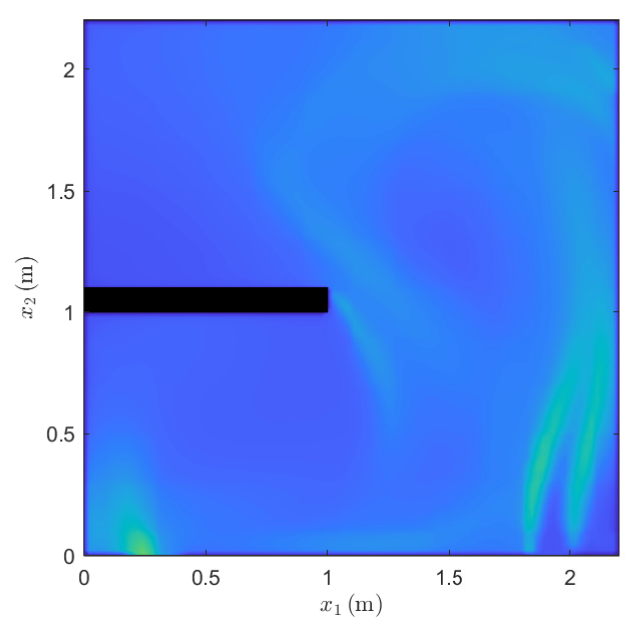}
	\captionsetup{justification=centering}
	\caption{prior turbulent intensity} \label{fig:}
	\end{subfigure}
%	\quad
	\begin{subfigure}[b]{0.261\textwidth}
		\includegraphics[width=\textwidth]{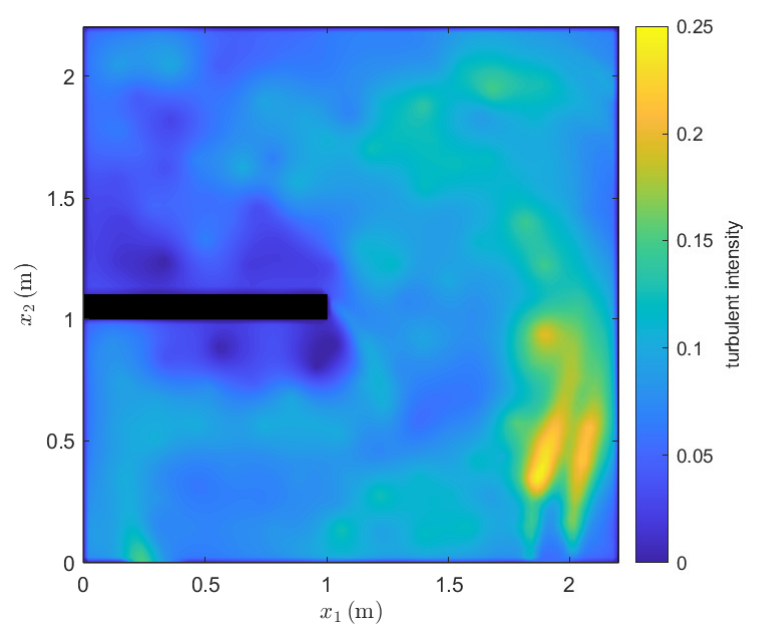}
	\captionsetup{justification=centering}
	\caption{posterior turbulent intensity} \label{fig:turbIpost}
	\end{subfigure}
	\caption{The prior fields for the most likely model $7$ and the posterior fields obtained using equations \eqref{eq:GM} after conditioning on empirical data.} \label{fig:postFld}
\end{figure*}
Comparing the prior and posterior velocity fields, we observe a general increase in velocity magnitude at the top-left part of the domain indicating that the flow sweeps the whole domain unlike the prior prediction from model $7$; see also Figure \ref{fig:vecFld}.
Furthermore, comparing the prior and posterior turbulent intensity fields, we observe a considerable increase in turbulent intensity throughout the domain. Referring to Table \ref{table:param}, note that among all RSM models, model $7$ has the highest turbulent intensity BC.
%a reduction in velocity and increase in turbulent intensity below the domain center that is not predicted by the numerical solution. This could be due to the presence of metal beams used for structural reinforcement in the domain; cf. Figure \ref{fig:domain}. These details are not modeled in the geometry of numerical solution but can have considerable effect on the real flow field.
%Furthermore, from Figure \ref{fig:velMagPostSd} we observe that the uncertainty in front of the inlet drops less than the rest of the domain despite collecting measurements there. This implies that there is higher variation in this region and larger number of instantaneous samples are required to decrease the uncertainty.
In \cite{meJ3_video2}, a visualization of the flow field is shown that validates the flow pattern depicted in Figures \ref{fig:vecFld} and \ref{fig:postFld}.

% ------------------------------------------------------------- %
%\subsection{Prediction} \label{sec:pred}
%
% Theme: error metrics and prediction error
To approximate the prediction performance of the posterior model, we collect $\hhatm=100$ new measurements at randomly selected locations; note that unlike Section \ref{sec:sim}, here we do not have access to the ground truth.
%Let $\hhaty_u(x_l)$ denote the measurement of the first velocity component at a location $x_l$ for $1 \leq l \leq \hhatm$. 
%Let also $\mu_u(x_l | \ccalX_k)$ denote the predicted mean velocity obtained from \eqref{eq:GMmean} after conditioning on the previously collected measurements $\ccalX_k$ up to iteration $k$ of Algorithm \ref{alg:ARF}.
%%
%Then, we define the mean error over the random locations $x_l$ as
%%
%\begin{equation} \label{eq:predErr}
%e_{u,k} = \frac{1}{\hhatm} \sum_{l=1}^{\hhatm} \abs{ \mu_u(x_l | \ccalX_k) - \hhaty_u(x_l) } ,
%\end{equation}
%%
%and finally
%%
%$ e_k = {1}/{3} ( e_{u,k} + e_{v,k} + e_{i,k} ) , $
%%
%where the expressions for $e_{v,k}$ and $e_{i,k}$ are identical.
%We are particularly interested in $e_0$ as the prediction error of the numerical solutions and $e_m$ as the prediction error of the posterior field.
Using equation \eqref{eq:predErr}, the prior prediction error is $e_0 = 0.092$\,m/s while the posterior error is $e_{155} = 0.037$\,m/s, a $60\%$ improvement compared to $e_0$.
%
%In Table \ref{table:error}, we report these posterior error values $e_{40}$ for both correlation field $L(x)$ options and information metrics.
%%
%\begin{table}[t!]
%\centering
%\renewcommand{\arraystretch}{1.1}
%\captionsetup{justification=centering}
%\caption{Prediction error values.}
%\begin{tabular}{|c||c|c|} 
% \hline
% $L(x)$ 		& entropy 		& mutual information		\\ [0.5ex] 
% \hline\hline
% $L_t(x)$		& $0.15$		& $0.70$				\\ 	 \hline
% $L_g(x)$		& $0.10$ 		& $0.80$	 			\\
% \hline
%\end{tabular}
%\label{table:error}
%\end{table}
%%
%It is obvious that the posterior error value $e_{40}$ is smaller than the numerical error value $e_0$ irrespective of the combination of options used in the ARF Algorithm \ref{alg:ARF}.
%
%Given the posterior knowledge that model $7$ has the highest probability, we also compute $\bbare_{0,7}$ where we use $\mu_u(x_l | \ccalX_k, \ccalM_7)$ in definition \eqref{eq:predErr} instead of the mean value from \eqref{eq:GMmean}. We do the same for $v(x)$ and $i(x)$. In this case, the error using the prior numerical model is 
Given the posterior knowledge that model $7$ is the most likely model, we have $\bbare_{0,7} = 0.052$\,m/s which is still $29\%$ higher than the posterior error value $e_{155}$. This demonstrates that the real flow field can be best predicted by systematically combining physical models and empirical data. Note that this hypothetical scenario requires the knowledge of the best model which is not available \textit{a priori}.
In Figure \ref{fig:errVal}, we plot separately for $u, v, \and i$, the prior errors of individual models as well as the prior and posterior models.
\begin{figure}[t!]
  \centering
    \includegraphics[width=0.45\textwidth]{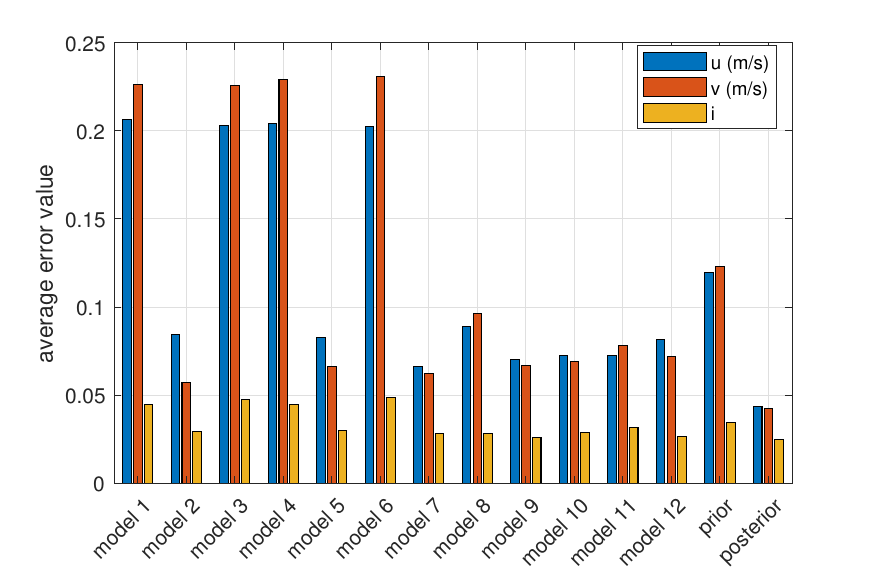}
  \caption{Prior error values for individual models along with their averaged prior as well as the posterior errors $e_{u,155}, e_{v, 155}$, and $e_{i, 155}$.} \label{fig:errVal}
\end{figure}
It can be seen that the solutions using the RSM models, including model $7$, generally have smaller errors; see also Table \ref{table:param}.
%Finally, referring to error values in Table \ref{table:error}, we observe that the entropy metric outperforms MI for our application. Given that the MI has suboptimality lower bound from Section \ref{sec:OPP}, it is a good benchmark to gauge the entropy metric.

As discussed at the beginning of this section, an important advantage of the proposed approach is that it provides uncertainty bounds on the predictions. Figure \ref{fig:uncertainty} in Appendix \ref{app:individualPredErrBars} shows the values of the individual measurements at each one of the $\hhatm$ measurement locations along with the posterior predictions and their uncertainty bounds.
%
%\begin{figure*}
%	\centering
%	\begin{subfigure}[b]{0.95\textwidth}
%		\includegraphics[width=\textwidth]{experiment/uErr.eps}
%	\captionsetup{justification=centering}
%	\caption{first velocity component} \label{fig:uErr}
%	\end{subfigure}
%	\quad
%	\centering
%	\begin{subfigure}[b]{0.95\textwidth}
%		\includegraphics[width=\textwidth]{experiment/vErr.eps}
%	\captionsetup{justification=centering}
%	\caption{second velocity component} \label{fig:vErr}
%	\end{subfigure}
%	\quad
%	\begin{subfigure}[b]{0.95\textwidth}
%		\includegraphics[width=\textwidth]{experiment/iErr.eps}
%	\captionsetup{justification=centering}
%	\caption{turbulent intensity} \label{fig:iErr}
%	\end{subfigure}
%	\caption{Posterior predictions and their uncertainty bounds at $\hhatm$ test locations along with the measured values shown by red stars. We also plot the error bounds on the measurements that fall outside the uncertainty bounds.} \label{fig:uncertainty}
%\end{figure*}
%
We also include error bars for the measurements that fall outside these uncertainty bounds. Out of $\hhatm=100$ measurements, $90\%$ fall inside the one standard deviation prediction bound for the first velocity component. Similarly, $90\%$ and $89\%$ of the measurements of the second velocity component and turbulent intensity are inside this bound. This shows that the proposed method provides reasonable uncertainty bounds.

%\textcolor{red}{There is spatial and temporal resolution lower bound on the sensors due to mechanical constraints. The spatial resolution is higher when the flow has less oscillation in angle since nearby sensors are closer to each other but in low velocity regions this error might be higher.}
%
%\textcolor{red}{There is difference between variation and uncertainty; the former is a property of the process while the latter can be decreased by better measurements. For turbulent intensity measurements, the localization and heading errors do not contribute to variation and should be excluded. In contrast, lack of spatial and temporal resolution contributes to variation but is not easy to model!}

In Appendix \ref{app:planComp}, we perform a study similar to Section \ref{sec:sim} on the performance of our proposed planning metric with an emphasis on suboptimality. It is known that mutual information (MI) as a planning metric possesses a suboptimality bound under certain set of conditions. This planning metric, however, is computationally intractable for large domains and was not included in the study of Section \ref{sec:sim}. 
For the smaller domain of this experiment, the prediction error of the entropy metic is better than the MI metric since unlike the entropy metric which signifies regions with higher variations in the velocity field (see Figure \ref{fig:simMetrics}), the MI metric is more sensitive to correlation; see Appendix \ref{app:planComp} for details.

\section{Discussion and Future Research Directions} \label{sec:disc}
%
% Theme: limitations of physics-based solutions
%
Our proposed physics-based approach, although more effective in learning environmental fields than purely data-driven methods, requires more information. Specifically, it assumes that a reliable physical model as well as the geometry of the domain and BCs are available. Although the probabilistic framework of Algorithm \ref{alg:ARF} allows for uncertainty in this information, we cannot expect a better performance if the available models are non-informative. This paper however, removes computational cost as an impediment to application of physics-based solutions to environmental sensing problems.

%Regarding scalability, note that the importance of our proposed algorithm will in fact become more evident as the domain size increases. This is because for larger domains, numerical simulations become more expensive. Consequently, it might be infeasible to obtain many numerical simulations. On the other hand, pure empirical estimation of the flow field becomes equally challenging since it requires dense measurements throughout the domain. A combined approach that allows for selection of the most accurate models given a reasonable number of measurements, is the ideal setting benefiting from both numerical simulations and experimental data.
%\footnote{Note that beyond a given domain size, it might be unnecessary to consider the entirety of the domain at once. Instead, one might be able to decompose the domain into smaller subdomains for which the flow conditions are in fact influenced by the immediate surroundings.}
%Note also that the computational cost of evaluating our entropy planning metric \eqref{eq:addediH} is independent of the domain size and only depends on the number of candidate measurement locations. For larger domains, as discussed in Remark \ref{rem:travelDist}, we enforce proximity constraints to prevent long travel distances and to decrease the number of candidate measurement locations.
%
In this paper, we have assumed that the turbulent intensity can be modeled using a GP. While this assumption allowed us to obtain a closed-form expression for the posterior distribution, more accurate probabilistic models of the turbulent intensity may improve the performance of the proposed method. %and will be explored in our future research. 
Also, as discussed prior to Assumption \ref{ass:velCompCorr}, the components of the velocity field are in fact correlated by the Navier-Stokes PDEs. While assuming independence simplified the development of the proposed framework, these correlations can be statistically captured using multi-output GPs \cite{CECMOGP2011AL}.
%In this paper we have designed a cost-effective, custom-built sensor set that measures the flow field in a plane parallel to the ground. Extending this design to measure all three components of the flow, and also using unmanned aerial vehicles (UAVs) that can measure at different heights are some possible extensions. Note that when using UAVs, one needs to ensure that the flow generated by the propellers does not interfere with the sensor readings.
%In the planning Algorithm \ref{alg:NBP}, we enumerate over the set of candidate measurement locations to select the next best measurement locations. Although we show that this approach when entropy metric is used, outperforms the MI which possesses a suboptimality bound, one might want to consider planning over a horizon. In such case, the brute force approach will become considerably more expensive.

% Theme: time-dependent problems
%
Additionally, while we have used a \textit{constant} correlation length scale in \eqref{eq:correlation}, it is possibly beneficial to define this length scale as a function of the characteristics of the turbulent flow; e.g., the integral length scale \cite{TMCFD1993W}. For instance, referring to Figure \ref{fig:postFld}, observe that the flow varies very little in the top-left area whereas it varies a lot in the bottom-right corner. Tuning the length scale to account for such variations, can result in better inference, prediction, and planning. %For example, we can consider the integral length scale of turbulent flows \cite{TMCFD1993W}.
The challenge with using such arbitrary length scales is ensuring that the resulting covariance matrices remain positive-definite; see \cite[Sec. 4.2.3]{GPML2006RW} for more details. %Designing such kernel functions specifically tailored for turbulent flows is an important future research direction.}

In Assumption \ref{assumption:ergodic}, we assumed that the turbulent flow is statistically steady and ergodic. For time-dependent problems, the flow is unsteady and Proposition \ref{prop:CLT} does not hold, meaning that we cannot model the flow properties using GPs.
It is also impossible to empirically measure the mean flow properties in this case. Extension of the proposed framework for time-dependent turbulent flows is a future research direction but as the authors in \cite{TEOPPGF2016KBH} point out, "utility of such an extension is limited by the accuracy of flow model used to predict the future time variations of the flow".
%
%\blue{ Finally, extending the results to a network of robots will only change the inference and planning and not the statistical model of the flow or development of noise components. Assuming that a central processing unit performs the inference, one can simply divide the domain of interest into subdomains that are covered with the network of robots. A distributed setting where the robots maintain their local estimates of the flow field is also possible but beyond the scope of the current paper; see for instance \cite{meJ2}.}

% Theme: scalibility
%
Finally, note that the main focus of this paper was to demonstrate the effectiveness of physics-based solutions to environmental sensing problems. In Section \ref{sec:sim}, we showed that Algorithm \ref{alg:ARF}, in its current form, can handle large domains since it requires considerably fewer measurements than purely data-driven methods. We plan to further investigate the scalability of our framework as well as distributed implementations involving teams of mobile robots.
It is well-known that the computational cost of inference using GPs increases when the number of observations increases \cite{EBSPMSNCMRF2013XCDM,MRSEMGMRF2020NKRD,ELGFGMRF2011LH}. A simple way to further improve scalability, is to utilize Gauss-Markov random fields (GMRFs) with the Matren kernel function instead of GPs; see \cite{GMRF2005RH}. For instance, the authors in \cite{EBSPMSNCMRF2013XCDM} derive a sequential formula for prediction using GMRFs whose cost is independent of the number of measurements.

% ------------------------------------------------------------------------------------------------------------------------------ %
\section{Conclusion} \label{sec:concl}
We proposed a physics-based method to learn environmental fields using a mobile robot. Specifically, we constructed GP models of the flow properties and used numerical simulations to inform their prior mean. Then, utilizing Bayesian inference, we incorporated measurements of flow properties into these GPs. To collect the measurements, we controlled a custom-built mobile robot sensor through a sequence of waypoints that maximize the information content of the measurements. We showed that, compared to purely data-driven methods that are common in the literature, our method can produce high-fidelity global estimations using only sparse measurements. Moreover, since the computationally expensive numerical simulations can be performed offline, our method can be implemented in real-time onboard mobile robots. To the best of our knowledge, this is the first physics-based framework for environmental sensing that has also been effectively demonstrated in practice.

% ------------------------------------------------------------------------------------------------------------------------------ %
\section*{Acknowledgements}
We would like to thank Dr. Wilkins Aquino for providing access to \textsc{ANSYS Fluent}, Dr. Scovazzi for his valuable input on turbulence, and Eric Stach and Yihui Feng for their help in designing the experimental setup and the mobile sensor.

% ------------------------------------------------------------------------------------------------------------------------------ %
\bibliographystyle{ieeetr}
\bibliography{MyBibliography}

% ------------------------------------------------------------------------------------------------------------------------------ %
\renewcommand{\thefigure}{A\arabic{figure}}
\setcounter{figure}{0}

\renewcommand{\thetable}{A.\Roman{table}}
\setcounter{table}{0}

\renewcommand{\thealgorithm}{A.\Roman{algorithm}}
\setcounter{algorithm}{0}

\appendices
% ------------------------------------------------------------------------------------------------------------------------------ %
\section{Details of Statistical Flow Model}\label{sec:turbF_app}
%
% ------------------------------------------------------------- %
\subsection{Proof of Gaussianity of Mean Flow Components} \label{app:CLT}
This directly follows from Assumption \ref{assumption:ergodic} and the central limit theorem (CLT) \cite{FCP2002R}. According to the ergodicity assumption, the time-averaged velocity component $u(x)$ is equal to the ensemble average $\hhatbbq_1(x,t)$. 
Also, since the flow is statistically steady, at a given spatial location $x$, the samples $u(x,t_l)$ are identically distributed.
Moreover, since $u(x,t)$ is a physical quantity, its variance is bounded. Then, from the CLT it follows that as $n \to \infty$, the distribution of $u(x)$ approaches a normal distribution. More specifically,
$$ \sqrt{n} \, \left( u(x) - \hhatbbq_1(x,t) \right) \rightarrow \ccalN \left( 0, \var[\bbq_1(x,t)] \right) .$$

% ------------------------------------------------------------- %
\subsection{Proof of Validity of Kernel Function} \label{app:kernel}
This follows from the fact that the kernel function \eqref{eq:kernel} is the composition of multiple valid kernel functions. First, note that the correlation function \eqref{eq:correlation} is a valid kernel itself. Regarding the standard deviation \eqref{eq:ustd}, note that multiplication of any function with itself, i.e., $\kappa_1(x,x') = i(x) i(x')$, constitutes a valid kernel. Moreover, scaling and addition with positive constants $q^2_{\text{ref}}/n_0$ and $\bbarsigma_{u,0}^2$, respectively, preserves the validity of the kernel. Thus the standard deviation \eqref{eq:ustd} is a valid kernel as well. Finally, the multiplication of two valid kernels results in a valid kernel meaning that the kernel function \eqref{eq:kernel} is valid as desired. See \cite[Sec 4.2]{GPML2006RW} for details.

% ------------------------------------------------------------- %
\subsection{Variance of Turbulent Intensity Measurements} \label{app:turbIvar}
In general, estimating the variance of turbulent intensity measurement $\sigma_i^2$ requires the knowledge of higher order moments of the random velocity components; cf. \cite{TBUETS1996BG}. Specifically under a Gaussianity assumption, it depends on the mean values $y_u(x) \and y_v(x)$ which are not necessarily negligible. Consequently, we need to incorporate this uncertainty into our statistical model.
To this end, we utilize the Bootstrap resampling method. Assume that the samples of instantaneous velocity $y_u(x,t_l)$ are independent and consider the measurement set $\ccalY_u = \set{y_u(x,t_l) \, | \, 1 \leq l \leq n}$; define $\ccalY_v$ and $\ccalY_w$ similarly. Furthermore, consider $n_b$ batches $\ccalB_j$ of size $n$ obtained by randomly drawing the same samples from $\ccalY_u$, $\ccalY_v$, and $\ccalY_w$ with replacement. Using \eqref{eq:turbIapprox}, we obtain $n_b$ estimates $\hhati_j(x)$ corresponding to the batches $\ccalB_j$. Then, the desired variance $\sigma_i^2(x)$ can be estimated as
\begin{equation} \label{eq:turbIvar}
\hat{\sigma}_i^2(x) = \frac{1}{n_b-1} \sum\nolimits_{j=1}^{n_b} (\hhati_j - \bbari)^2 ,
\end{equation}
where $\bbari(x)$ is the mean of the batch estimates $\hhati_j(x)$; see \cite{IVE2007W} for more details.

% ------------------------------------------------------------- %
\subsection{Estimation of Integral Time Scale} \label{app:smpFreq}
In practice, we have access to a finite set of samples of the instantaneous velocity vector over time and can only approximate \eqref{eq:intTimeScale}. Let $l$ denote the discrete lag. Then, the sample autocorrelation of the first velocity component is given by
$$ \hat{\rho}_u(l) = \frac{1}{n} \sum\nolimits_{k=1}^{n-l} [y_u(x,t_k) - y_u(x)] [y_u(x,t_{k+l}) - y_u(x)] . $$
This approximation becomes less accurate as $l$ increases since the number of samples used to calculate the summation decreases. Furthermore, the integral of the sample autocorrelation $\hat{\rho}_u(l)$ over the range $1 \leq l \leq n-1$ is constant and equal to $0.5$ \cite{SSAF2009H}. This means that we cannot directly use \eqref{eq:intTimeScale} which requires integration over the whole time range. The most common approach is to integrate $\hat{\rho}_u$ up to the first zero-crossing \cite{AFDIL2004NNDJ}.

% ------------------------------------------------------------------------------------------------------------------------------ %
\section{Details of Mobile Robot Design}\label{sec:path_app}
%
% ------------------------------------------------------------- %
\subsection{Flow Measurement Algorithm} \label{app:flowmetry}
Algorithm \ref{alg:flowmetry} is used by the mobile robot sensor to collect the desired flow measurements.
\begin{algorithm}[t]
\caption{Flow Measurements using Mobile Robot}
\label{alg:flowmetry}
\begin{algorithmic}[1]
\small

\REQUIRE Measurement location $x$, sample number $n$, and the orientation $\beta$ of the mobile sensor;

\STATE Given $\beta$, compute sensor headings $\beta_j$ for $1 \leq j \leq 8$;

\FOR{$k=1:n$}
	
	\STATE Receive the readings $\bbs \in \reals^8$ from all sensors;
	
	\STATE Let $j$ and $l$ denote sensors with the highest readings;	\label{line:sortReading}
	
	\IF{$ \abs{j-l} > 1 $}	\label{line:check1}
	
		\STATE \texttt{warning: inaccurate sample!}
	
	\ENDIF
	
	\STATE Let $\xi = \sign[ \sin(\beta_j - \beta_l) ]$ and compute:	\label{line:angle}
	\begin{align*}
	\theta = \arctan [ \xi (s_j \cos \beta_l - s_l \cos \beta_j),	
	\xi(- s_j \sin \beta_l + s_l \sin \beta_j) ] .
	\end{align*}
	
	\IF{$ \theta < \min \set{\beta_j,\beta_l} $ or $ \max \set{\beta_j,\beta_l} < \theta $}	\label{line:check2}
	
		\STATE \texttt{warning: inaccurate sample!}
	
	\ENDIF
	
	\STATE Velocity magnitude: $q = \bbs_j / \cos( \theta - \beta_j) ; $ \label{line:velMag}
	
	\STATE First velocity component: $\hhatbby_{u,k} = q \cos \theta$;	\label{line:vel1}
	
	\STATE Second velocity component: $\hhatbby_{v,k} = q \sin \theta$; \label{line:vel2}
		
\ENDFOR

%\STATE Average sensor readings: $\bbr = \texttt{mean}(\bbs)$;	\label{line:avgRead}
%
%\STATE Mean velocity magnitude: $y_q = \texttt{mean}(\bby_q)$;

\STATE First mean velocity component: $y_u = \texttt{mean}(\hhatbby_u)$;	\label{line:avgRead1}

\STATE Second mean velocity component: $y_v = \texttt{mean}(\hhatbby_v)$;		\label{line:avgRead2}

\STATE Compute sample variances $\hat{\sigma}^2_u$ and $\hat{\sigma}^2_v$ via \eqref{eq:uVar} and \eqref{eq:vVar};

\STATE Compute turbulent intensity measurement $y_i$ via \eqref{eq:turbIapprox}; \label{line:turbI}

\STATE Compute variance $\hat{\sigma}^2_i$ of turbulent intensity using \eqref{eq:turbIvar};	\label{line:turbIvar}

\STATE Return $y_u$, $y_v$, $y_i$, $\hat{\sigma}^2_u$, $\hat{\sigma}^2_v$, and $\hat{\sigma}^2_i$;

\end{algorithmic}
\end{algorithm}
Given the measurement location $x$ and the heading of the mobile sensor $\beta$, the algorithm collects instantaneous samples of the velocity vector field and computes the mean velocity component measurements $y_u(x)$ and $y_v(x)$ and the turbulent intensity measurement $y_i(x)$.
Particularly, in line \ref{line:sortReading} it finds the sensor indices $j$ and $l$ with the two highest readings, respectively. As discussed in Section \ref{sec:mobSen}, these two sensors should be next to each other. If not, the algorithm issues a warning indicating an inaccurate measurement in line \ref{line:check1}.
In line \ref{line:angle}, it computes the flow angle $\theta$ using the four-quadrant inverse Tangent function and equation \eqref{eq:sensVel}. This flow angle is validated in line \ref{line:check2}.
In lines \ref{line:velMag}-\ref{line:vel2}, the algorithm computes the magnitude of the velocity and its components.
The vectors $\hhatbby_u$ and $\hhatbby_v$ store samples at different times, i.e., $\hhatbby_{u,k} = y_u(x,t_k)$ is the $k$-th sample of the instantaneous velocity component $u(x,t)$; see also equation \eqref{eq:meas}. In lines \ref{line:avgRead1} and \ref{line:avgRead2}, \texttt{mean$(\cdot)$} denotes the mean function.
Finally, in lines \ref{line:turbI} and \ref{line:turbIvar}, the algorithm computes the turbulent intensity measurement and its variance.

% ------------------------------------------------------------- %
\subsection{Uncertainty due to Measurement Location Error} \label{app:measLocErr}
Let $x \sim \ccalN(x_0, \gamma^2_x \, \bbI_2)$ denote a $2$D distribution for the measurement location where $x_0$ is the nominal location and $\tdpi(x) = \set{\tdp_k, \tdx_k}_{k=1}^{\tilde{n}}$ denote its corresponding SROM discretization \cite{meC3}. Without loss of generality, we assume $\tdx_1 = x_0$, i.e., $x_0$ belongs to the set of SROM samples. From \eqref{eq:GP}, the measurement $y(x)$ at a point $x$ is normally distributed as $y | x \sim \ccalN(\mu(x), \kappa(x,x))$. Then, we can marginalize the location to obtain the expected measurement distribution as
$$ \bbarpi(y) = \sum\nolimits_{k=1}^{\tilde{n}} \tdp_k \, \bbarpi(y| \tdx_k) . $$
This distribution is a Gaussian mixture (GM) and cannot be properly modeled as a normal distribution. Nevertheless, the probability of the measurement location being far from the mean (nominal) value $x_0$ drops exponentially. This means that $\tdp_1$ corresponding to $\tdx_1 = x_0$ is larger than the rest of the weights and the distribution is close to a unimodal distribution.
Noting that we can obtain the mean and variance of $\bbarpi(y)$ in closed-form, a Gaussian distribution can match up to two moments of the underlying GM. Particularly, 
\begin{subequations} \label{eq:measLocE}
\begin{align}
\tilde{\mu}(x) &= \sum\nolimits_{k=1}^{\tilde{n}} \tdp_k \, \mu(\tdx_k), 	\label{eq:measLocMean} \\
\tilde{\sigma}^2(x) &= \sum\nolimits_{k=1}^{\tilde{n}} \tdp_k \left[ \kappa(\tdx_k, \tdx_k) + (\mu(\tdx_k) - \tilde{\mu}(x))^2 \right] . 	\label{eq:measLocVar}
\end{align}
\end{subequations}

Considering equations \eqref{eq:measLocE}, the following points are relevant:
(i) For simplicity, we do not consider the covariance between different measurements in this computation.
(ii) Equations \eqref{eq:measLocE} are only relevant at measurement locations. Thus, in equations \eqref{eq:inference}, we replace the entry of $\bbmu_{\ccalX}$ at location $x$ with the corresponding mean value $\tilde{\mu}(x)$ from \eqref{eq:measLocMean} and similarly, the diagonal entry of $\bbSigma_{\ccalX \ccalX}$ corresponding to $x$ with $\tilde{\sigma}^2(x)$ from \eqref{eq:measLocVar}.
(iii) The kernel function, defined in \eqref{eq:measKernel}, depends on the heading error variances $\bbargamma_{u, \beta}^2(x)$ and $\bbargamma_{v, \beta}^2(x)$ of the velocity components through \eqref{eq:uVar} and \eqref{eq:vVar}. These values depend on empirical data that are only available at the true measurement location and not at SROM samples $\tdx_k$. Thus, we use the same value for $\bbargamma_{u, \beta}$ and $\bbargamma_{v, \beta}$ at all SROM samples.

% ------------------------------------------------------------------------------------------------------------------------------ %
\section{Details of Experimental Results}\label{sec:res_app}
%
% ------------------------------------------------------------- %
\subsection{Mobile Robot Hardware and Control} \label{app:hardware}
%
% Theme: mobile robot, sensors, localization
As discussed in Section \ref{sec:mobSen}, we construct a custom mobile robot to carry out the experiment. To minimize interference with the flow pattern and since the experimental domain is small, we use a small differential-drive robot constructed with thin components; see Figure \ref{fig:robot}. This robot carries the flow sensors and collects the measurements.
We equip the mobile robot with an \textsc{Arduino Leonardo} microprocessor that implements simple motion primitives and collects instantaneous velocity readings and communicates them back to an off-board PC for processing. The radio communication between the mobile robot and the PC happens via a network of \textsc{Xbee-S2}s. The PC  processes the instantaneous readings, the localization information, and runs Algorithm \ref{alg:ARF} using \textsc{MATLAB}.
%As discussed in Section \ref{sec:mobSen}, the mobile sensor is also equipped with eight D6F-W01A1 \textsc{OMRON} sensors with range of $0-1$m/s and full-scale error rate of FS $ = 0.05$m/s.
%We set the standard deviation of the sensor noise to $\gamma_s = 0.017$m/s according to \eqref{eq:sensorNosie}, the heading error to $\gamma_{\beta} = 5^o$, and the measurement location error to $\gamma_x = 0.025$m. We construct a SROM model for the location error with $\tdn = 5$ samples; see Section \ref{sec:measNoise}. Note that the location error also takes into account the discretization of the numerical solutions and the error caused by the sensor rig structure.

In order to control the mobile robot between consecutive measurement locations, we decouple its motion into heading control and straight-line tracking. The low level planning to generate these straight-line trajectories while avoiding the obstacles can be done in a variety of ways; here we use geodesic paths generated by \textsc{VisiLibity} toolbox \cite{VisiLibityL2008O}. The feedback required for motion control is provided by an \textsc{OptiTrack} motion capture system. The robot stops when the distance between desired and final locations is below a threshold. A similar condition is used for the heading angle.

% ------------------------------------------------------------- %
\subsection{Signal Processing} \label{app:sigProcess}
%
% Theme: preface
In order to examine the accuracy of the probabilistic measurement models developed in Section \ref{sec:statModel}, we conduct a series of experiments where we collect $25$ measurements of the flow properties at two locations $x^1 = (1.9, 0.4)$ and $x^2 = (0.35, 0.4)$ in the domain of experiment; see Figure \ref{fig:domain}. $x^1$ is located at a high velocity region whereas $x^2$ is located at a low velocity region. %The robot travels back and forth between these two points ensuring that the measurements are independent. Each time we rotate the robot $15^o$ to study the effect of heading on the accuracy of measurements.
%
% Theme: instantaneous samples
The instantaneous samples of $y_u(x^1,t)$ and $y_v(x^1,t)$ are given in Figure \ref{fig:velSample} for the first measurement. 
\begin{figure}[t!]
  \centering
    \includegraphics[width=0.4\textwidth]{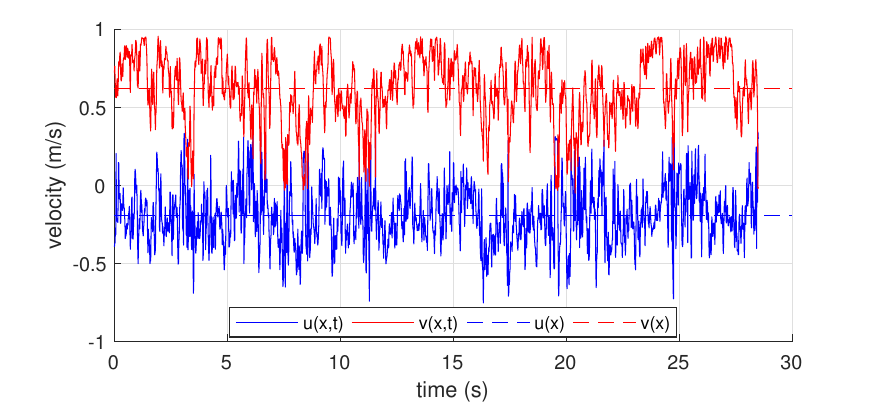}
  \caption{The instantaneous velocity components at $x^1$ for the first measurement where $y_u(x^1) = - 0.07$\,m/s and $y_v(x^1) = 0.63$\,m/s.} \label{fig:velSample}
\end{figure}
A video of the instantaneous velocity readings can be found in \cite{meJ3_video1}; note the amount of fluctuations in the velocity vector. The robot samples the velocity vector with a fixed frequency of $67$\,Hz. Then, to ensure that the samples are uncorrelated, it down-samples them using the integral time scale; see Section \ref{sec:smpFreq}.
The average integral time scales over $25$ measurements are $\bbart^*(x^1) = 0.26$\,s and $\bbart^*(x^2) = 0.29$\,s, respectively. 

% Theme: 25 measurements analysis
Figure \ref{fig:uSample} shows a histogram of the $25$ independent $u(x^1)$ measurements. This histogram approaches a Gaussian distribution as the number of measurements increases; see Section \ref{sec:meanVelComp}. Histograms of the second component and turbulent intensity are similar for both points.
\begin{figure}[t!]
  \centering
    \includegraphics[width=0.3\textwidth]{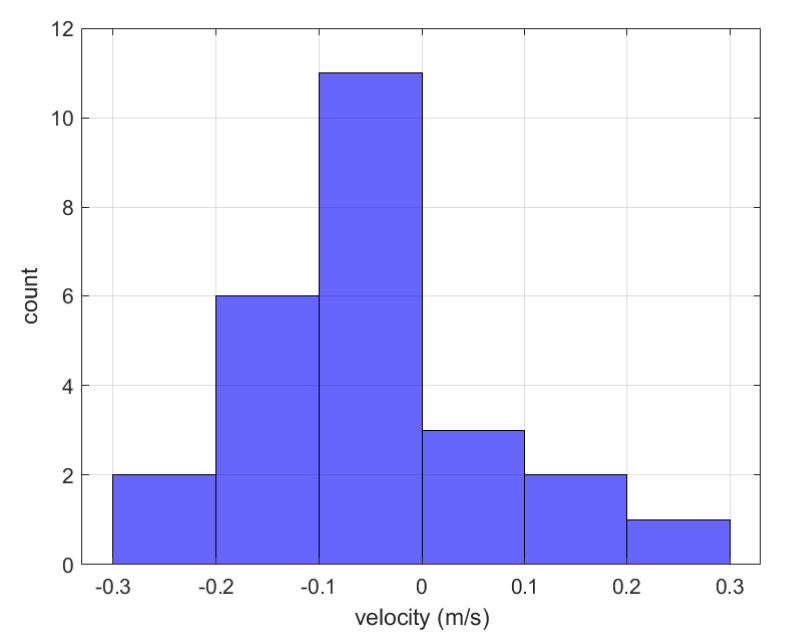}
  \caption{Histogram of the $25$ independent measurements of $u(x^1)$.} \label{fig:uSample}
\end{figure}
Figure \ref{fig:x1Vector} shows the measurements of the velocity vector at $x^1$.
\begin{figure}[t!]
  \centering
    \includegraphics[width=0.3\textwidth]{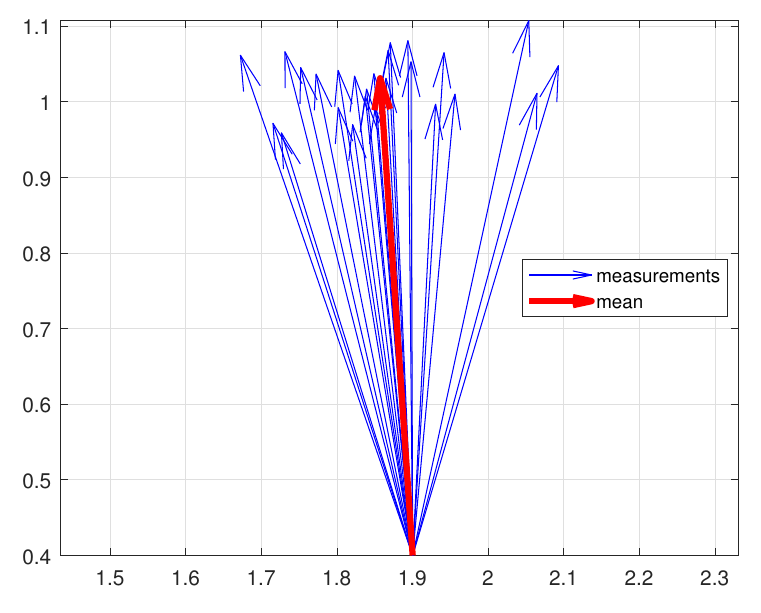}
  \caption{$25$ independent measurements of the velocity vector at $x^1 = (1.9, 0.4)$ along with their averaged velocity vector.} \label{fig:x1Vector}
\end{figure}
In Figure \ref{fig:sdSamples}, we plot the sample standard deviations of the velocity components and turbulent intensity, obtained from equations \eqref{eq:smpMvar} and \eqref{eq:turbIvar}, for both locations.
\begin{figure*}
	\centering
	\begin{subfigure}[b]{0.32\textwidth}
		\includegraphics[width=\textwidth]{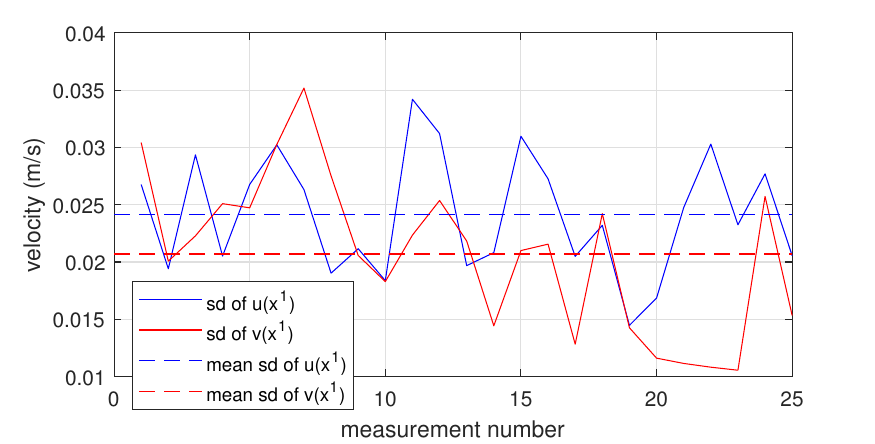}
	\captionsetup{justification=centering}
	\caption{velocity components at point $x^1$}
	\end{subfigure}
%	\quad
	\centering
	\begin{subfigure}[b]{0.32\textwidth}
		\includegraphics[width=\textwidth]{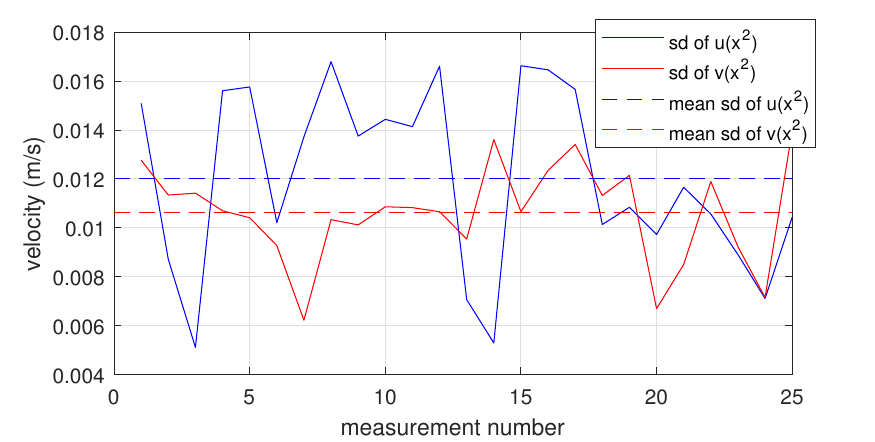}
	\captionsetup{justification=centering}
	\caption{velocity components at point $x^2$}
	\end{subfigure}
%	\quad
	\begin{subfigure}[b]{0.32\textwidth}
		\includegraphics[width=\textwidth]{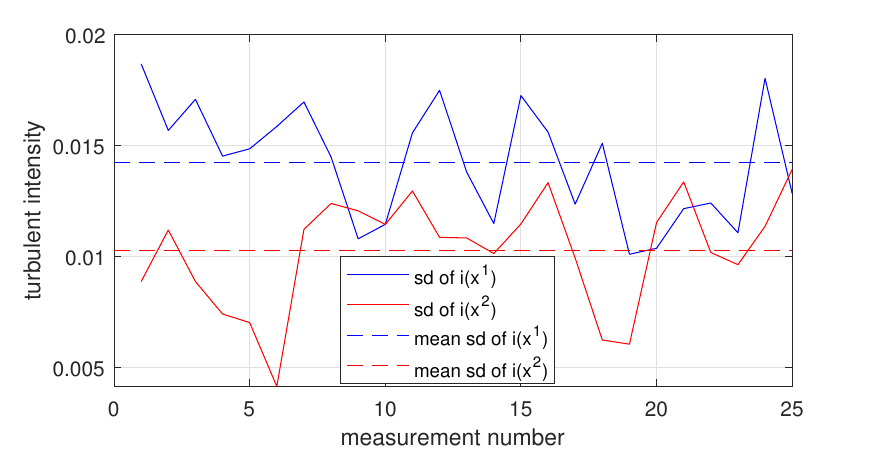}
	\captionsetup{justification=centering}
	\caption{turbulent intensity for points $x^1 \and x^2$}
	\end{subfigure}
	\caption{Standard deviations of velocity components and turbulent intensity for each of $25$ measurements in locations $x^1 = (1.9, 0.4)$ and $x^2 = (0.35, 0.4)$.}\label{fig:sdSamples}
\end{figure*}
%
%Note that for the velocity components, the average standard deviation values are close indicating that the isotropy assumption is appropriate.
In Table \ref{table:sd}, we compare the standard deviations estimated from the time-series measurements of each experiment, averaged over all experiments, to the standard deviation of time-averaged measurements from all $25$ experiments. We refer to the former as within-experiment and the latter as cross-experiment.
\begin{table}[t!]
\centering
\renewcommand{\arraystretch}{1.2}
\caption{Within-experiment and cross-experiment standard deviations.}
\footnotesize
\begin{tabular}{|c||c|c|} 
 \hline
 method 		& within-exp 	& cross-exp			\\ [0.5ex] 
 \hline\hline
 sd of $y_u(x^1)$	& $0.024$		& $0.120$				\\ 	 \hline
 sd of $y_v(x^1)$	& $0.021$		& $0.042$				\\ 	 \hline \hline
 sd of $y_u(x^2)$	& $0.012$		& $0.038$				\\ 	 \hline
 sd of $y_v(x^2)$	& $0.011$		& $0.023$				\\ 	 \hline \hline
 sd of $y_i(x^1)$	& $0.014$ 	& $0.026$	 			\\	 \hline
 sd of $y_i(x^2)$	& $0.010$ 	& $0.015$	 			\\	 \hline
\end{tabular}
\label{table:sd}
\end{table}
Note that the cross-experiment values, reported in the second column, are larger. This is due to actuation errors that are not considered within experiments but appear across different measurements, i.e., the experiments corresponding to location $x^1$ for instance, are not all conducted at the same exact location and heading of the robot. The small number of experiments may also contribute to inaccurate estimation of the standard deviations in the second column.

% ------------------------------------------------------------- %
\subsection{Uncertainty Quantification at New Measurement Locations} \label{app:individualPredErrBars}
Figure \ref{fig:uncertainty} depicts the individual measurements at each one of the $\hhatm=100$ test locations in Section \ref{sec:exp} along with the posterior predictions and their uncertainty bounds.
\begin{figure*}
	\centering
	\begin{subfigure}[b]{0.75\textwidth}
		\includegraphics[width=\textwidth]{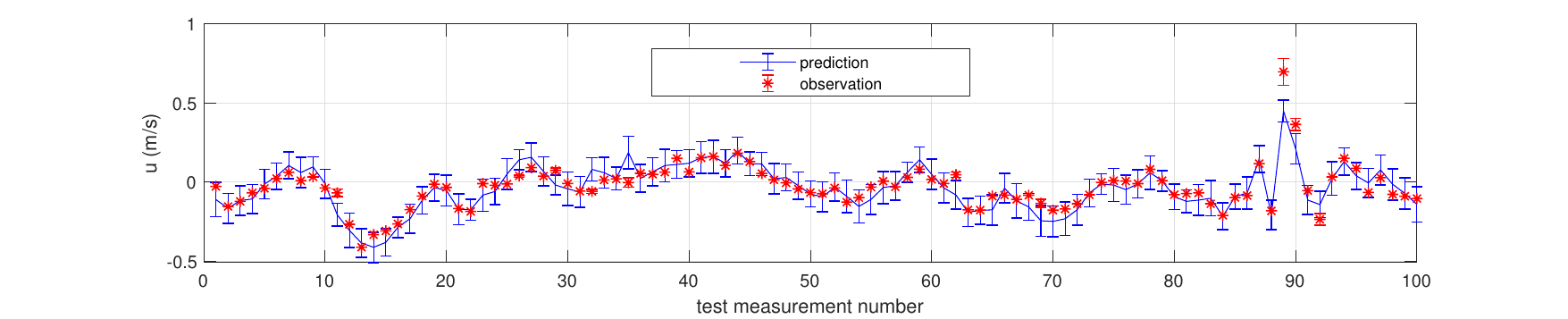}
	\captionsetup{justification=centering}
	\caption{first velocity component} \label{fig:uErr}
	\end{subfigure}
	\quad
	\centering
	\begin{subfigure}[b]{0.75\textwidth}
		\includegraphics[width=\textwidth]{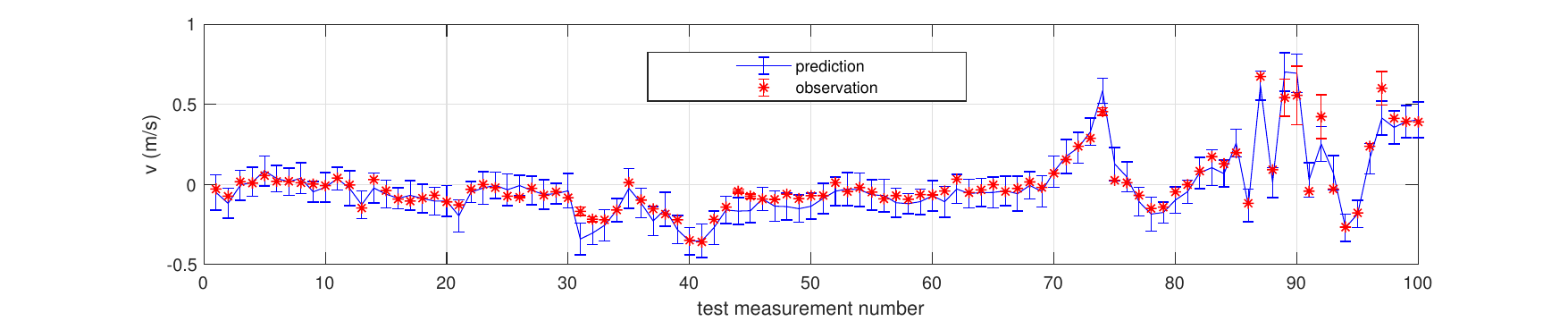}
	\captionsetup{justification=centering}
	\caption{second velocity component} \label{fig:vErr}
	\end{subfigure}
	\quad
	\begin{subfigure}[b]{0.75\textwidth}
		\includegraphics[width=\textwidth]{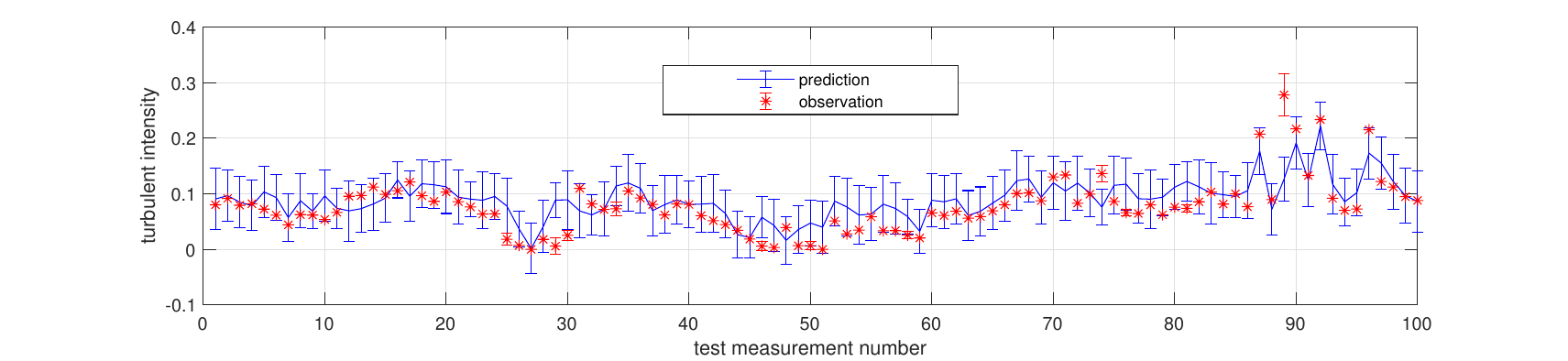}
	\captionsetup{justification=centering}
	\caption{turbulent intensity} \label{fig:iErr}
	\end{subfigure}
	\caption{Posterior predictions and their one standard-deviation uncertainty bounds at $\hhatm=100$ test locations of Section \ref{sec:exp} along with the measured values shown by red stars. We also plot the error bounds on the measurements that fall outside the uncertainty bounds.} \label{fig:uncertainty}
\end{figure*}
%

% ------------------------------------------------------------- %
\subsection{Suboptimality of the Planning Algorithm} \label{app:planComp}
%
% Theme: placement
%
In this section, we compare the performance of our proposed planning Algorithm \ref{alg:NBP} for the experimental setup of Section \ref{sec:exp}, to a baseline placement over a lattice and an information-theoretic planning using mutual information (MI) metric which has a suboptimality lower-bound.
Using the same notation introduced in Section \ref{sec:OPP}, let $I(\ccalX_k, \Omega \backslash \ccalX_k)$ denote the MI between measurements $\ccalX_k$ and unobserved regions of the domain $\Omega \backslash \ccalX_k$, given by
$ I(\ccalX_k, \Omega \backslash \ccalX_k) = H(\Omega \backslash \ccalX_k) - H(\Omega \backslash \ccalX_k \, | \, \ccalX_k) . $
Denote this value by $I(\ccalX_k)$ for short. Then, the amount of added information by a new measurement $x \in \ccalS_{k+1}$ is
$ \delta I_{\text{MI}}(x \, | \,  \ccalX_k) = I(\ccalX_k \cup x) - I(\ccalX_k) = H(x \, | \, \ccalX_k) - H(x \, | \, \bar{\ccalX_k} ) , $
where $\bar{\ccalX}_k = \Omega \backslash (\ccalX_k \cup x)$; see \cite{NSPGP2008KSG} for details of derivation.
%Particularly, for a GP using the definition of the joint entropy in \eqref{eq:entropyJoint} we get
%
%\begin{equation} \label{eq:addediMI}
%\delta I_{\text{MI}}(x \, | \,  \ccalX) \propto \frac{ \gamma_u^2(x  \, | \,  \ccalX)  }{ \gamma_u^2(x  \, | \,  \bar{\ccalX})  } .
%\delta I_{\text{MI}}(x \, | \,  \ccalX) \propto \frac{ \bbarkappa(x,x) - \bbSigma_{x\ccalX} \bbSigma_{\ccalX \ccalX}^{-1} \bbSigma_{\ccalX x} }
%		{ \bbarkappa(x,x) - \bbSigma_{x\bar{\ccalX}} \bbSigma_{\bar{\ccalX} \bar{\ccalX}}^{-1} \bbSigma_{\bar{\ccalX} x}} .
%\end{equation}
%
For a GP, MI as a set function is submodular and under certain assumptions on the discretization $\Omega$ of the domain, it is `approximately' monotone \cite{NSPGP2008KSG}. Then, according to the maximization theory of submodular functions, it can be optimized using a sequential greedy approach, as in Algorithm \ref{alg:NBP}, while ensuring sub-optimality lower-bounds; see \cite{AAMSSF1978NWF}.
In order for this suboptimality bound to hold, we perform the planning for MI offline. %meaning that we do not update the model probabilities $p_{j,k}$ in \eqref{eq:greedy}.
%On the other hand, differential entropy although submodular, is not monotone and thus the suboptimality result does not hold for it. We compare the performance of these two indices in Section \ref{sec:exp}.
%
%In order to predict the measurement variances $\sigma_u^2(x)$ and $\sigma_v^2(x)$ at a candidate location $x$, which are unknown \textit{a priori}, we utilize the prior turbulent intensity field $i(x)$. To do so, we assume that the turbulent flow is isotropic, meaning that the variations are direction-independent. %\footnote{Note that this assumption is not true in general but in the absence of any prior knowledge and for planning purposes it suffices.}
%Then, the expression for the entropy defined in \eqref{eq:entropyJoint} can be approximated by
%%
%$$ H(x \, | \, \ccalX_k) \approx 2 H(u(x \, | \, \ccalX_k) ) \propto \gamma_u^2(x  \, | \,  \ccalX_k) , $$
%%
%where in \eqref{eq:predVar} we approximate
%%
%$$ \sigma_u^2(x  \, | \,  \ccalX_k) \approx \frac{ 1 }{ n_0 } \, q^2_{\text{ref}} \, i^2(x  \, | \,  \ccalX_k); $$
%%
%see Section \ref{sec:RANS} for details. Similar to \eqref{eq:ustd}, $n_0$ is the number of independent instantaneous velocity samples to be collected in the candidate measurement location $x$. Given these approximations, the added information according to the MI metric is
From the above definition, the amount of added information according to the MI metric for model $\ccalM_j$ is
\begin{equation} \label{eq:addediMI}
\delta I_{\text{MI}}(x \, | \,  \ccalX_k, \ccalM_j) \propto \frac{ \gamma_u^2(x  \, | \,  \ccalX_k, \ccalM_j) \, \gamma_v^2(x  \, | \,  \ccalX_k, \ccalM_j)  }{ \gamma_u^2(x  \, | \,  \bar{\ccalX}_k, \ccalM_j) \, \gamma_v^2(x  \, | \,  \bar{\ccalX}_k, \ccalM_j)  } .
\end{equation}
Note that depending on the size of the discrete domain $\Omega$, the covariance matrix $\bbSigma_{\bar{\ccalX} \bar{\ccalX}}$ in \eqref{eq:addediMI} can be very large. This makes the evaluation of the MI metric \eqref{eq:addediMI} considerably more expensive than the entropy metric \eqref{eq:addediH}. In fact, planning using MI is intractable for the larger domain of Section \ref{sec:sim}. Unlike MI, the computation cost of the entropy metric \eqref{eq:addediH} is independent of the size of the domain and only depends on the number of collected measurements $\ccalX_k$.

In Figure \ref{fig:lattice}, we compare our proposed planning approach to the MI metric \eqref{eq:addediMI} and a sequence of independent experiments run for lattices ranging from $4\times4$ to $9\times9$ measurements. We use a maximum of $m=81$ measurements and do not impose any travel distance constraints on the planning methods for this comparison.\footnote{Each lattice placement explores the whole domain. Placing travel distance constraints will prevent the other planning methods from a similar exploration and favor the lattice placement.}
Particularly, Figure \ref{fig:lattice} shows the prediction error value $e_k$ as a function of the planning step $k$ for the same $\hhatm = 100$ test locations used in Section \ref{sec:exp}.
\begin{figure}[t!]
  \centering
    \includegraphics[width=0.35\textwidth]{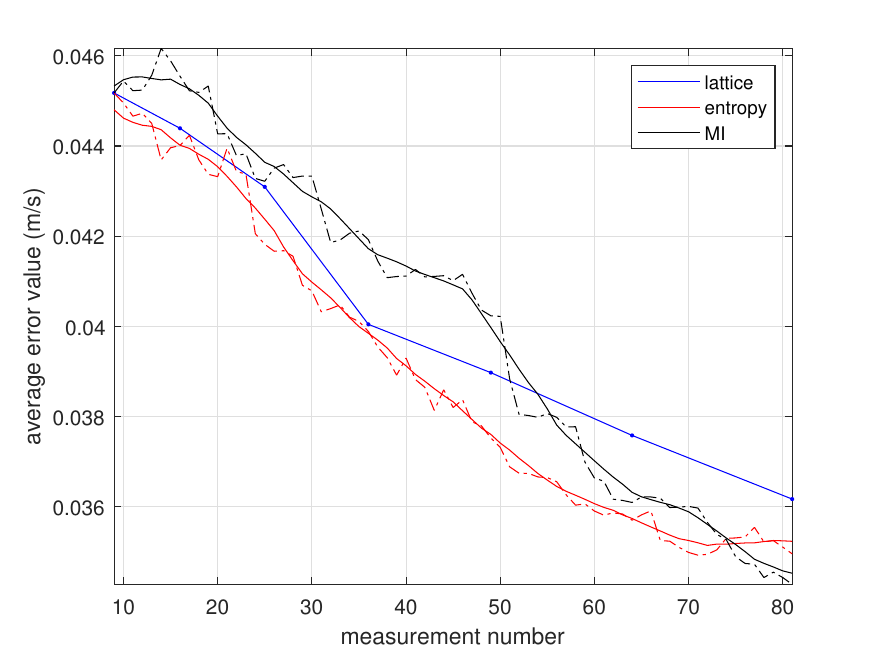}
  \caption{Prediction error values for the proposed planning Algorithm \ref{alg:NBP} using the entropy \eqref{eq:addediH} and MI \eqref{eq:addediMI} metrics, and a sequence of lattices ranging from $4\times4$ to $9\times9$ measurements.} \label{fig:lattice}
\end{figure}
From Figure \ref{fig:lattice}, it can be seen that initially our method and the baseline lattice placement perform roughly similarly while the MI metric has consistently higher errors. Once the proposed planning method has collected enough measurements to properly select the most likely model, it outperforms the baseline method as well. Clearly, once a large number of measurements have been collected so that the whole domain is sufficiently covered, all approaches would perform similarly.
It can be seen that the performance improvement using the proposed planning algorithm is less significant compared to Section \ref{sec:sim}. This is because for larger domains, information is more localized whereas here most measurements are relatively informative regardless of their location.
Note also that unlike the simulation study in Section \ref{sec:sim}, the ground truth is unknown here. This makes the conclusions less reliable and leads to oscillations in prediction error values in Figure \ref{fig:lattice} since the test measurements themselves are empirically collected and subject to uncertainty.
%In the absence of ground truth knowledge of flow properties which does not exist for such turbulent flows, it is impossible to have a more reliable comparison.
%Finally, note that the nature of this comparison is coverage-like which means that the improvements obtained by optimal planning would not be considerable. However, although there might exist lattice placements that by accident measure in more informative regions of the domain but it is not possible to know the size of this lattice \textit{a priori}.

Figure \ref{fig:planning2} shows the sequence of waypoints generated by Algorithm \ref{alg:NBP} for the entropy \eqref{eq:addediH} and MI \eqref{eq:addediMI} metrics, respectively.
\begin{figure}[t!]
	\centering
	\begin{subfigure}[b]{0.32\textwidth}
		\includegraphics[width=\textwidth]{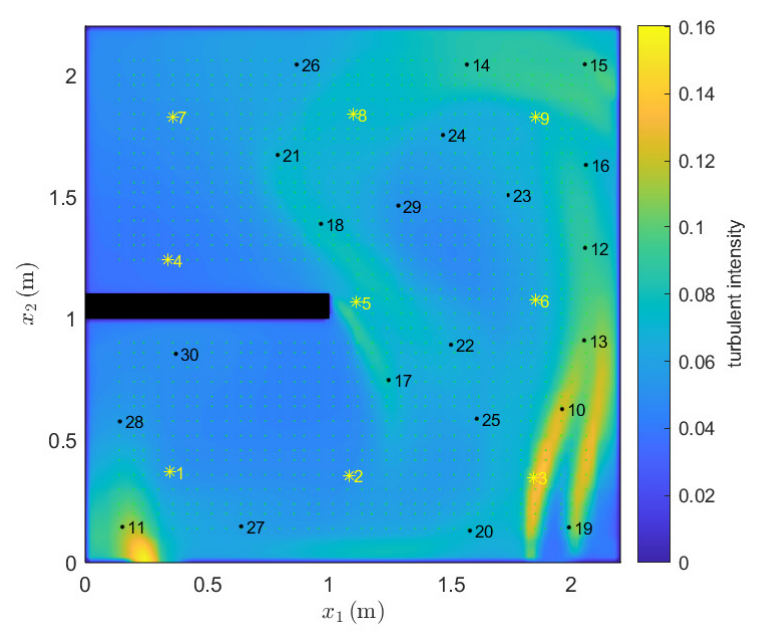}
	\captionsetup{justification=centering}
	\caption{entropy metric \eqref{eq:addediH}}
	\end{subfigure}
	\quad
	\begin{subfigure}[b]{0.32\textwidth}
		\includegraphics[width=\textwidth]{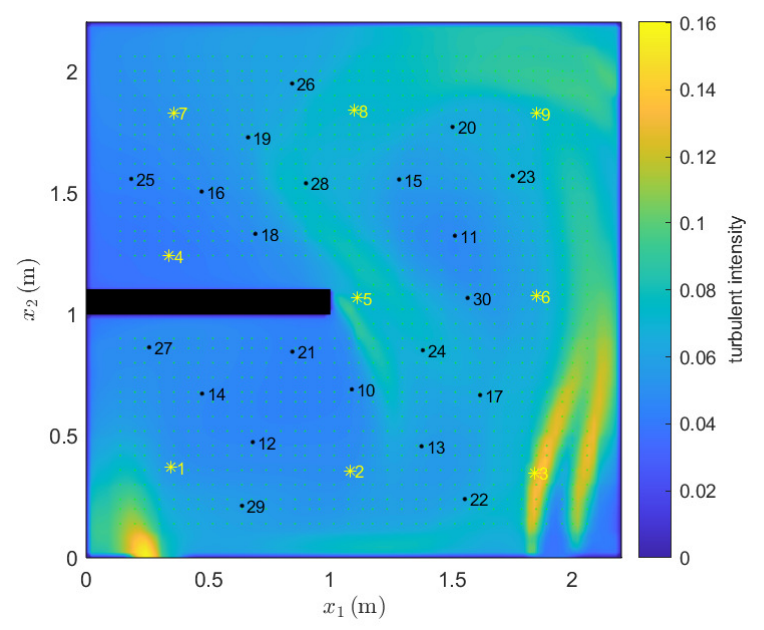}
	\captionsetup{justification=centering}
	\caption{mutual information metric \eqref{eq:addediMI}}
	\end{subfigure}
	\caption{Path of the mobile sensor for the two optimality metrics overlaid on turbulent intensity field from model $7$. The green dots show the candidate measurement locations and the yellow stars show the exploration measurement locations. Finally, the black dots show the sequence of waypoints.}\label{fig:planning2}
\end{figure}
It is evident that the entropy metric signifies regions with higher variations in the velocity field whereas the MI metric is more sensitive to correlation; see also equations \eqref{eq:addediH} and \eqref{eq:addediMI}.
Conditioning on the measurements collected using these two metrics up to step $k=30$, the prediction errors for the same $\hhatm = 100$ test locations, are $e_{30} = 0.040$\,m/s for the entropy metric and $e_{30, \text{MI}} = 0.043$\,m/s for the MI metric, respectively. Thus, in this particular case, the entropy metric outperforms the MI metric by $7\%$. Note that with only $30$ measurements, the error in the posterior fields is considerably smaller than the prior value $e_0 = 0.092$\,m/s. Nevertheless, reaching a steady estimated field requires more measurements as we demonstrated in Section \ref{sec:exp}.

\end{document}